\documentclass[11pt]{article}

\usepackage[]{graphicx,float,latexsym,times}
\usepackage{amsfonts,amstext,amsmath,amssymb,amsthm}
\usepackage{caption2}
\usepackage{multirow}
\usepackage{subfigure}

\long\def\symbolfootnote[#1]#2{\begingroup%
\def\thefootnote{\fnsymbol{footnote}}\footnote[#1]{#2}\endgroup}

\usepackage{titlesec} 

\titleformat{\section}{\large\bfseries}{\thesection.}{.5em}{}
\titlespacing*{\section}{0pt}{*3}{*2}
\titleformat{\subsection}{\normalfont\bfseries}{\thesubsection.}{.5em}{}
\titlespacing*{\subsection} {0pt}{*3}{*2}
\titleformat{\subsubsection}{\normalfont\bfseries}{\thesubsubsection.}{.5em}{}
\titlespacing*{\subsubsection} {0pt}{*3}{*2}

\theoremstyle{plain} 
\newtheorem{theorem}{Theorem}[section]
\newtheorem{lemma}{Lemma}[section]

\theoremstyle{definition} 

\newtheorem{remark}{Remark}[section]

\numberwithin{equation}{section} 

\usepackage[authoryear]{natbib}

\begin{document}

\title{\textbf{\Large Scan $B$-Statistic for Kernel Change-Point Detection}}

\date{}

\maketitle

\author{
\begin{center}
\vskip -1cm

\textbf{\large Shuang Li$^a$, Yao Xie$^a$, Hanjun Dai$^b$, Le Song$^b$}

$^a$School of Industrial and Systems Engineering (ISyE), Georgia Institute of Technology, 
Atlanta, Georgia, USA, 
$^b$School of Computer Science and Engineering (CSE), Georgia Institute of Technology, 
Atlanta, Georgia, USA.

\end{center}
}

\symbolfootnote[0]{\normalsize Address correspondence to Yao Xie,
School of Industrial and Systems Engineering, Georgia Institute of Technology, Atlanta, Georgia, 30332, USA; E-mail: yao.c.xie@gmail.com}



{\small \noindent\textbf{Abstract:} Detecting the emergence of an abrupt change-point is a classic problem in statistics and machine learning. Kernel-based nonparametric statistics have been used for this task, which enjoys fewer assumptions on the distributions than the parametric approach and can handle high-dimensional data. 
In this paper, we focus on the scenario when the amount of background data is large, and propose a computationally efficient kernel-based statistics for change-point detection, which are inspired by the recently developed $B$-statistics. A novel theoretical result of the paper is the characterization of the tail probability of these statistics using the change-of-measure technique, which focuses on characterizing the tail of the detection statistics rather than obtaining its asymptotic distribution under the null distribution. Such approximations are crucial to controlling the false alarm rate, which corresponds to the average-run-length in online change-point detection. Our approximations are shown to be highly accurate. Thus, they provide a convenient way to find detection thresholds for online cases without the need to resort to the more expensive simulations. We show that our methods perform well on both synthetic data and real data.}
\\ \\
{\small \noindent\textbf{Keywords:} Change-point detection; Kernel-based statistics; Online algorithm;  False-alarm control. }
\\ \\
{\small \noindent\textbf{Subject Classifications:} Primary 62L10; Secondary 62G10, 62G32.}

\section{INTRODUCTION}

Given a sequence of samples, $x_1, x_2, \ldots, x_t$, from a domain $\mathcal{X}$, we are interested in detecting a possible change-point $\tau$, such that before the change samples $x_i$ are {\it i.i.d.}~with a null distribution $P$, and after the change samples $x_i$ are {\it i.i.d.}~with a distribution $Q$. Here, we consider two scenarios: the time horizon $t$ is fixed, $t=T_0$, which we call the offline or fixed-sample change-point detection, or the time horizon $t$ is not fixed, meaning that one can keep getting new samples, which we call the online or sequential change-point detection. In the offline setting, our goal is to detect the existence of a change. In the online setting, our goal is to detect the emergence of a change as soon as possible after it occurs. Here, we restrict our attention to detecting one change-point.  One such instance is seismic event detection as studied by \citet{P-S-wave-picking2014}, where one would like to either detect the presence of a weak event in retrospect to better understand the geophysical structure or detect the event as quickly as possible for online monitoring. 

Ideally, the detection algorithm should be free of distributional assumptions to be robust when applied to real data. To achieve this goal, various kernel-based nonparametric statistics have been proposed in the statistics and machine learning literature, see, e.g., \citet{Zaid2008,zaid-sparse-alternative2014,LiangKME2014,Kifer2004,density-ratio2013,online-kernel-changepoint2005}, which typically work well with multi-dimensional real data since they are distributional free. Kernel approaches are distribution free and more robust as they provide consistent results over larger classes of data distributions; albeit they can be less powerful in settings where a clear distributional assumption can be made.  However, most kernel based statistics cost $\mathcal{O}(n^2)$ to compute over $n$ samples. In the online change-point detection setting, the number of samples grows with time and hence we cannot directly use the naive approach. Recently,  \cite{B-test2013} developed the so-called \textit{$B$-test statistic} to reduce the computational complexity. The $B$-test statistic samples $N$ pairs of blocks of size $B$ from the two-sample data, compute the unbiased estimates of the kernel-based statistic between each pair and then take an average. The computational complexity of the $B$-test statistic reduces to $\mathcal{O}(nB^2)$  instead of $\mathcal{O}(n^2)$.

In this paper, we present two scan statistics related to $B$-test statistics customized for offline and online change-point detection, which we name as \textit{scan $B$-statistics}. The proposed statistics are based on kernel maximum mean discrepancy (MMD) in \citet{gretton2012kernel, kernel-based-testing-2013}. They are inspired by the $B$-test statistic but differ in various ways to tailor to the need of change-point detection. Typically, there is a small number of post-change samples (for instance, seismic events are relatively rare, and in online change-point detection, one would like to detect the change quickly). But there is a large amount of reference data. So when constructing the detection statistic, we reuse the post-change samples for the test block and construct multiple and disjoint reference blocks. This leads to a non-negligible dependence between the MMD statistics being averaged over. Hence, we cannot use the existing approach based on the central limit theorem to analyze them. Moreover, the scanning nature of the proposed statistic also introduces non-negligible dependence. We construct the reference and test blocks in a structured way so that analytical expressions for false alarm can be obtained.  
 
Our main theoretical contribution includes accurate theoretical approximations to the false-alarm rate of scan $B$-statistics. Controlling false alarms is a key challenge in change-point detection. Specifically, this means to quantify the significance level for offline change-point detection, and the average run length (ARL) for online change-point detection. Here, we cannot directly rely on the null property of the $B$-test statistic established in the existing work, because the scan statistics take the maximum of multiple statistics computed over overlapping data blocks that causes strong correlations. Hence, one cannot use the central limit theorem or even the martingale central limit theorem. Instead, we adopt a recently developed change-of-measure technique by \citet{YakirBook2013} for scan statistics, which are capable of dealing with the more challenging situation here.

Our contribution also includes: (1) obtaining a closed-form variance estimator, which allows easy calculation of the scan $B$-statistics; (2) further improving the accuracy of our approximations by taking into account the skewness of the kernel-based statistics. The accuracy of our approximations is validated by numerical examples. Finally, we demonstrate the good performance of our method using real-data, including speech and human activity data.

\subsection{Related Work}

Classic parametric approaches for change-point detection can be found in \cite{SiegmundBook1985,Tartakovsky14}. There has been an array of nonparametric change-point detection methods. Notable non-parametric schemes for change-point detection include \cite{GordonPollak94,Picard85}, which are designed for scalar observations and not suitable for vector observations. \cite{Brodsky13} provide a comprehensive introduction to the methodologies and applications of nonparametric change-point detection. \cite{Bibinger17} construct a nonparametric minimax-optimal test to discriminate continuous paths with volatility jumps and prove weak convergence of the test statistic to an extreme value distribution.
In the online setting, \cite{Kifer2004} present a meta-algorithm which compares data in some ``reference window'' to the data in the current window, using empirical distance measures that are not kernel-based; \cite{online-kernel-changepoint2005} detect abrupt changes by comparing two sets of descriptors extracted online from the signal at each time instant: the immediate past set and the immediate future set, and then use a soft margin single-class support vector machine  to build a dissimilarity measure in the feature space between those sets without estimating densities as an intermediate step,  which is asymptotically equivalent to the Fisher ratio  in the Gaussian case; \cite{density-ratio2013} present a density-ratio estimation method to detect change-points, fitting the density-ratio using a non-parametric Gaussian kernel model, whose parameters are updated online via stochastic gradient descent approach. Another important branch of nonparametric change-point detection method is based on Kolmogorov-Smirnov test, in \citet{Massey51,Lilliefors67}, which has been used in  \citet{Wang14}. The generalization of Kolmogorov-Smirnov test from the univariate setting to the multi-dimensional setting is given by \citet{Fasano87}, which, however, is less convenient to use than the kernel-based statistic test.

Seminal works by \cite{CsorgHorv88} study kernel based $U$-statistic for change-point detection. They show that the statistic indexed by the assumed change-point location parameter $\tau$, after proper standardization and rescaling of time and magnitude, converges in distribution to a Gaussian process under the null, and converges to a deterministic path in probability under the alternative distribution when the number of samples goes to infinity. These results are useful for bounding the detection statistics under the null with high-probability (hence, controlling the false detection), and for studying the consistency of tests. 
\cite{CsorgHorv97} and \cite{Serfling2009} contain comprehensive discussions on asymptotic theory of nonparametric statistics including $U$-statistics. Our scan $B$-statistic can also be viewed as a form of $U$-statistic using an appropriate definition of the kernel. The main differences between these classic works from our proposed scan $B$-statistic are: (1) our statistic uses $B$-test block decomposition and averaging to make the test statistic more computationally efficient; (2) our statistic is more challenging to analyze due to the block structure and correlation introduced by scan 
statistics; (3) our analytical approach is different: \cite{CsorgHorv88} leverage invariance principle to establish convergence of the entire sample path; we focus on characterizing the tail probability of the statistic under the null and use the change-of-measure technique to achieve good approximation accuracy. 

Other existing works that also focus on establishing asymptotic distribution of the detection statistic under the null for controlling the false alarm rate include the following: 
\citet{Zaid2008} present a maximum kernel Fisher discriminant ratio statistic and study its asymptotic null distribution;
\citet{Dehling15} investigate the two-sample test $U$-statistic for dependent data. Our approach is different from above in that we focus on directly approximating the tail of the detection statistic under the null, rather than trying to obtain its asymptotic distribution. Moreover, traditional analyses are usually done for offline change-point detection, while our analytical framework based on change-of-measure can be applied to both offline and online change-point detection. 

Change-point detection problems are related to the classical statistical two-sample test. However, they are usually more challenging than the two-sample test because the change-point location $\tau$ is unknown. Hence, when forming the detection statistic, one has to ``take the maximum'' of the detection statistics. The statistics being maxed over are usually highly correlated since they are computed using overlapping data. 

Our techniques for approximating false alarm rates differ from large-deviation techniques in \cite{dembo09}, which establish exponential rate by which the probability converges to zero. In certain scenarios, the first-order approximation obtained from large-deviation techniques may not be sufficient for choosing threshold. Our method provides more refined approximations to include polynomial terms and constants. 

Finally, there are also works taking different approaches rather than hypothesis test for change-point detection. For instance, \cite{HarchaouiCap07} develop a kernel-based multiple-change-point detection approach, where the optimal location to segment the data is obtained by dynamic programming; \cite{ArlotCelisseHarchaoui12} estimates multiple change-points by developing a kernelized linear model, and they provide a non-asymptotic oracle inequality for the estimation error. In the offline setting, 
\cite{LiangKME2014} study a problem when there are $s$ anomalous sequences out of $n$ sequences to be detected, and the test statistic is constructed using MMD; \cite{Matteson14} propose a nonparametric approach based on $U$-statistics and adopt the hierarchical clustering, which is capable of consistently estimating an unknown number of multiple change-point locations; \cite{Zou14} propose a nonparametric maximum likelihood approach, with the number of change-points determined from the Bayesian information criterion (BIC) and the locations of the change-points estimated via dynamic programming.

Our notations are standard. Let $I_k$ denote the identity matrix of size $k$-by-$k$. Let $\mathbb E[{\cal A}; {\cal B}] = \mathbb E [{\cal A} \textbf{1}_{\cal{B}}]$ denote the expectation conditioned on event $\cal{B}$, where $\textbf{1}_{\cal{B}}$ represents the indicator function that takes value 1 when the event $\cal{B}$ happens and takes value 0, otherwise. Let ${\rm Var}(\cdot)$ and ${\rm Cov}(\cdot)$ denote the variance and the covariance. Let ${\bf 0}$ and ${\bf e}$ denote vectors of all zeros and all ones, respectively. Let $[\Sigma]_{ij}$ denote the $ij$-th element of a matrix $\Sigma$. In Section \ref{sec:theory}, $\mathbb{E}_B$, ${\rm Var}_B$, and ${\rm Cov}_B$ denote the values computed under the new probability measure $\mathbb{P}_B$ after the change-of-measure, where $B$ is the block size. Similarly,  in Section \ref{sec:online}, $\mathbb{E}_t$, ${\rm Var}_t$, and ${\rm Cov}_t$ denote the values obtained under the new probability measure $\mathbb{P}_t$ after the change-of-measure, where $t$ is the time index.  

\section{BACKGROUND} 

We first briefly review the reproducing kernel Hilbert space (RKHS) and the \textit{maximum mean discrepancy} (MMD).  A RKHS $\mathcal F$ on $\mathcal X$ with a kernel $k(x,x')$ is a Hilbert space of
functions $f(\cdot):\mathcal X \mapsto \mathbb R$ equipped with inner product $\langle\cdot, \cdot\rangle_{\mathcal F}$. Its element $k(x,\cdot)$ satisfies the reproducing property:
$\langle f(\cdot), k(x, \cdot)\rangle_{\mathcal F} = f(x)$, and consequently, $\langle k(x,\cdot), k(x', \cdot)\rangle_{\mathcal F} = k(x,x')$,
meaning that we can view the evaluation of a function $f$ at any point $x\in\mathcal X$ as an inner product. Commonly used RKHS kernel functions include the Gaussian radial basis function (RBF) $k(x,x') = \exp(-\|x-x'\|^2/2\sigma^2)$, where $\sigma>0$ is the kernel bandwidth, and polynomial kernel $k(x,x')=(\langle x, x'\rangle + a)^d$, where $a>0$ and $d \in \mathbb N$ (see \citet{SchSmo02}). RKHS kernels can also be defined for sequences, graph and other structured object (see \citet{SchTsuVer04}). In this paper, if not otherwise stated, we will assume that Gaussian RBF kernel is used. 

Assume there are two sets $X$ and $Y$, each with $n$ samples taking value on a general domain $\mathcal{X}$, where $X=\{x_1, x_2, \dots, x_n\}$ are {\it i.i.d.} ~with a distribution $P$, and $Y=\{y_1, y_2, \dots, y_n\}$ are {\it i.i.d.} ~with a distribution $Q$. 
The MMD is defined as~\citep{gretton2012kernel} 
\[
\mbox{MMD} [\mathcal{F}, P, Q] := \sup_{f\in \mathcal{F}} \left\{ \mathbb{E}_{X\sim P} [ f(X) ]-\mathbb{E}_{Y\sim Q} [f(Y)]\right\}.
\]
An unbiased estimator of $\mbox{MMD}^2$ can be obtained using $U$-statistic~\citep{gretton2012kernel}
\begin{equation}
\mbox{MMD}_u^2[\mathcal F, X, Y] =\frac{1}{n(n-1)} \sum_{i \neq j}^n h(x_i, x_j, y_i, y_j),
\label{MMD_u}
\end{equation}
where $h(\cdot)$ is the kernel for $U$-statistic and it can be defined using an RKHS kernel as
\begin{equation}
h(x_i, x_j, y_i, y_j)=k(x_i,x_j)+k(y_i,y_j)-k(x_i,y_j)-k (x_j,y_i). \label{kernel-to-kernel}
\end{equation}
Intuitively, the empirical test statistic $\mbox{MMD}_u^2$ is expected to be small (close to zero) if $P=Q$, and large if $P$ and $Q$ are ``far'' apart. The complexity for evaluating $\mbox{MMD}_u^2$ is $\mathcal{O}(n^2)$, since we have to form the so-called Gram matrix for the data, which is of size $n$-by-$n$. Under the null hypothesis, $P=Q$, the $U$-statistic is degenerate and has the same distribution as an infinite sum of Chi-square variables.  

To improve computational efficiency, an alternative approach to eatimate $\mbox{MMD}^2$, called the $B$-test, is presented by~\citep{B-test2013}. The key idea is to partition the $n$ samples from $P$ and $Q$ into $N$ non-overlapping blocks, $X_1,\ldots,X_N$ and $Y_1,\ldots,Y_N$, each of size $B$. Then one computes $\mbox{MMD}_u^2[\mathcal F,X_i,Y_i]$ for each pair of blocks and takes an average:
\[
\mbox{MMD}_B^2[\mathcal F, X,Y]=\frac{1}{N}\sum_{i=1}^N \mbox{MMD}_u^2[\mathcal F,X_i,Y_i]. 
\]
Since $B$ is constant and $N$ is on the order of $\mathcal{O}(n)$, the computational complexity of $\mbox{MMD}_B^2[\mathcal F,X,Y]$ is $\mathcal{O}(nB^2)$, which is significantly lower than the $\mathcal{O} (n^2)$ complexity of $\mbox{MMD}_u^2[\mathcal F,X,Y]$. Furthermore, by averaging $\mbox{MMD}_u^2[\mathcal F,X_i,Y_i]$ over blocks, when blocks are independent, the $B$-test statistic is asymptotically normal under the null using central limit theorem. This property allows a simple threshold to be derived for the B-test.  

\section{SCAN $B$-STATISTICS}

Now we present our change-point detection procedure based on \textit{scan $B$-statistic}. Consider a sequence of data $\{\ldots, x_{-2}, x_{-1}, x_0, x_1, \dots, x_t\}$, each taking value on a general domain $\mathcal X$. Let $\{\ldots, x_{-2}, x_{-1}, x_0 \}$ denote the reference data that we know to follow a given pre-change distribution. Assume there is a large amount of reference data. 

In offline change-point detection, the number of samples is fixed, and our goal is to detect the {\it existence} of a change-point $\tau$, such that before the change-point, the samples are {\it i.i.d.}~with a distribution $P$, and after the change-point, the samples are {\it i.i.d.}~with a different distribution $Q$. The location $\tau$ where the change-point occurs is unknown. In other words, we are concerned with testing the null hypothesis
\begin{align*}
H_0: {x}_i \sim P, \quad i = 1, \dots, t,
\end{align*}
against the single change-point alternative
\begin{align*}
H_1: \exists 1 \leq \tau < t \quad {x}_i \sim \begin{cases}
Q, & i > \tau \\
P, & \rm{otherwise}.
\end{cases}
\end{align*}
Note that we are interested in the case of a sustained change: before the change, all samples follow one distribution, and after the change, all samples follow another distribution and never switch back. In online change-point detection, the number of samples is not fixed, and the goal is to detect the {\it emergence} of a change-point as quickly as possible. In various change-point detection settings, the number of post-change samples is small, but the number of reference samples is large. Therefore, when constructing MMD statistics over blocks, we will use a common post-change block and multiple disjoint pre-change reference blocks.

\begin{figure}[t!]
        \begin{center}
        \small
        \begin{tabular}{cc}
                \includegraphics[width=.48\textwidth]{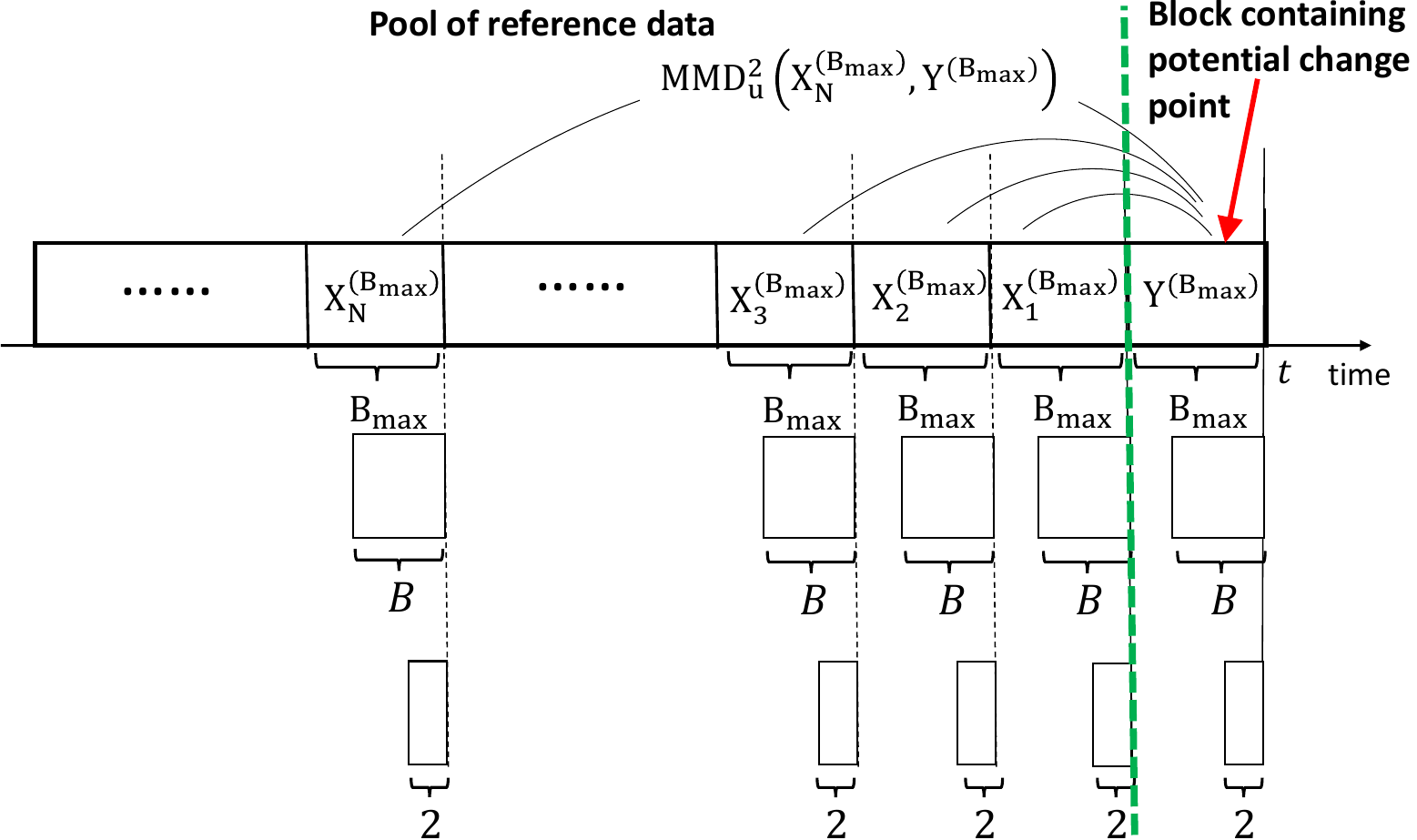} & 
                \includegraphics[width=.48\textwidth]{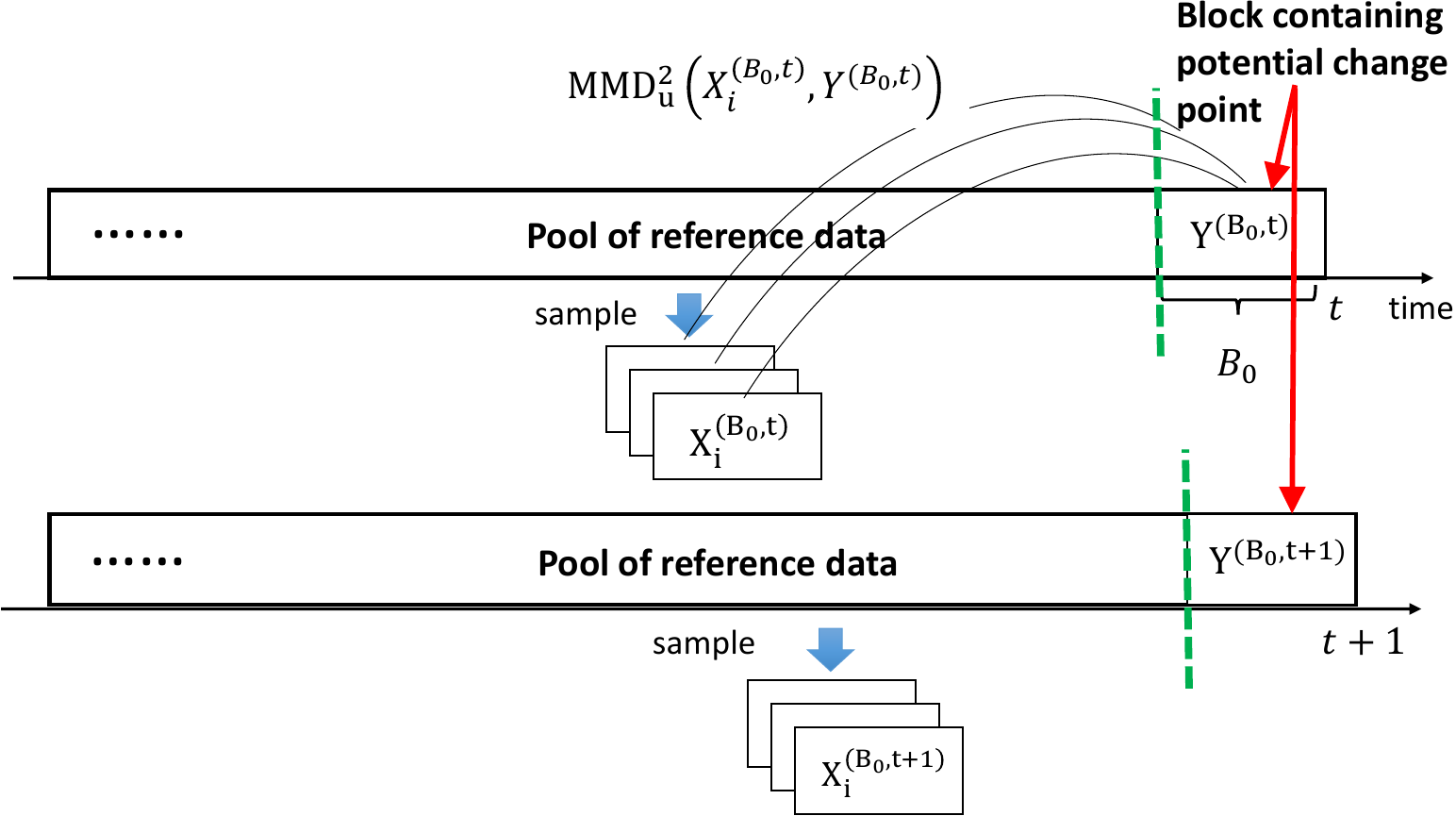}\\
                (a): offline & (b): sequential 
        \end{tabular}
        \end{center}
        \caption{Illustration of (a) offline change-point detection: data are initially split into blocks of size $B_{\rm max}$; we select data from each block to form smaller sub-blocks of various size $B$, $2\leq B\leq B_{\rm max}$; (b) online change-point detection: the most recent $B_0$ samples constitute the test block, which is constantly updated with new data; the reference blocks of the same size $B_0$ are sampled from the reference pool of data. }
        \label{model}                               
\end{figure}

\subsection{Offline Change-Point Detection}

For each possible change location $\tau$, the post-change block consists of the most recent samples indexed from $\tau$ to $t$. Since we do not know the change-point location, we scan all possible change-point locations $\tau$. This corresponds to considering a range of post-change block sizes $B$ ranging from two (i.e., the most recent two samples are post-change samples) to $B_{\rm max}$. Here, we exclude $B = 1$ because the corresponding MMD is unable to compute.

The detection statistic is constructed as follows, also illustrated in Figure \ref{model}(a). Data are split into $N$ reference blocks and one test block, each block is size  of $B_{\rm max}$. Then we select data from each block to form smaller sub-blocks of various size $B$, $2\leq B\leq B_{\rm max}$. The reference blocks are denoted as $X_i^{(B)}$, $i = 1, \ldots, N$, and the test block as $Y^{(B)}$. We compute $\mbox{MMD}_u^2$ for each reference sub-block with respect to the {\it common} post-change block, and take an average: 
\begin{equation}
\begin{split}
Z_B &= \frac{1}{N} \sum_{i=1}^N \mbox{MMD}_u^2 (X_i^{(B)},Y^{(B)}). 
      \end{split}
      \label{def_MMD_stats} 
\end{equation}

Since the estimator $\mbox{MMD}_u^2$ is unbiased, under the null hypothesis $P = Q$, $\mathbb{E} [Z_B]=0$. Let ${\rm Var} [Z_B]$ denote the variance of $Z_B$ under the null. The variance of $Z_B$ depends on the block size $B$ and the number of blocks $N$. To have a fair comparison, we normalize each $Z_B$ by their standard deviation
 \[
 Z_B' = Z_B/(\mbox{Var}[Z_B])^{1/2}, \]
and take the maximum over all $B$ to form the {\it offline scan $B$-statistic}. The variance ${\rm Var} [ Z_B ]$ is given in Lemma \ref{thm:variance}. The closed-form expression facilitates the estimation of the variance of the statistic.  A change-point is detected whenever the offline scan $B$-statistic exceeds a pre-specified threshold $b$:
\begin{equation}
\max_{2\leq B\leq B_{\rm max}} Z_B' > b. \quad \{\mbox{offline change-point detection}\} \label{scan-stats}
\end{equation}

\subsection{Online Change-Point Detection} 

In the online setting, new samples sequentially and we constantly test whether the incoming samples come from a different distribution. To reduce computational burden, in the online setting, we fix the block-size and adopt a {\it sliding window} approach. The resulted sliding window procedure can be viewed as a type of Shewhart chart by \citet{Shewhart39}. 

The detection statistic is constructed as follows, also illustrated in Figure \ref{model}(b). At each time $t$, we treat the most recent $B_0$ samples as the post-change block. In online change-point detection, we want to detect the change as quickly as possible. Hence, typically we will not wait till collecting many post-change samples. On the other hand, there is a large amount of reference data. To utilize data efficiently, we utilize a common test block consisting of the most recent samples to form the statistic with $N$ different reference blocks.
The reference blocks are formed by taking $NB_0$ samples without replacement from the reference pool.  
We compute $\mbox{MMD}_u^2$ between each reference block with respect to the common post-change block, and take an average:
\begin{equation}
Z_{B_0, t} = \frac{1}{N} \sum_{i=1}^N \mbox{MMD}_u^2 (X_{i}^{(B_0, t)},Y^{(B_0, t)}), \label{B_stats}
\end{equation}
where $B_0$ is the fixed block-size, $X_{i}^{(B_0, t)}$ is the $i$-th reference block at time $t$, and $Y^{(B_0, t)}$ is the the post-change block at time $t$. When there are new samples, we append them to the post-change block and purge the oldest samples. We show later that this construction allows for an explicit characterization of the false-alarm rate.
We divide each statistic by its standard deviation to form the {\it online scan $B$-statistic}: \[Z_{B_0,t}' = Z_{B_0, t}/(\mbox{Var}[Z_{B_0, t}])^{1/2}.\] The calculation of ${\rm Var}[Z_{B_0, t}]$ can also be achieved using Lemma \ref{thm:variance}.  The online change-point detection procedure is a stopping time: an alarm is raised whenever the detection statistic exceeds a pre-specified threshold $b>0$:
\begin{equation}
T = \inf \{t: Z'_{B_0, t} > b\}.
\quad \{\mbox{online change-point detection}\}
\label{stopping-time}
\end{equation}

The online scan $B$-statistic can be computed efficiently. 
Note that the variance of the $Z_{B_0, t}$ only depends on the block size $B_0$ but is independent of $t$. Hence, it can be pre-computed. Moreover, there is a simple way to compute the online $B$-statistic recursively, as specified in Appendix \ref{recursive}.

 \subsection{Analytical Expression for ${\rm Var}[Z_B]$}  
We obtain an analytical expression for ${\rm Var}[Z_B]$, which is useful when forming the detection statistic in (\ref{scan-stats}) and (\ref{stopping-time}). 
 \begin{lemma}[Variance of $Z_B$ under the null]\label{thm:variance}
Given block size $B\geq 2$ and the number of blocks $N$, under the null hypothesis, 
\begin{equation}
\mbox{\rm Var} [Z_B ]={\binom{B}{2}}^{-1}\left(   \frac{1}{N} \mathbb{E} [h^2(x, x',y,y')]  +\frac{N-1}{N} {\rm Cov} \left[ h(x,x',y,y'),h(x'',x''', y, y')   \right] \right),
\label{var_expr1}
\end{equation}
where $x,\, x',\,x'',\,x''',\, y,$ and $y'$ are {\it i.i.d.}~random variables with the null distribution $P$. 
\end{lemma}
The lemma is proved by making a connection between $\mbox{MMD}_u^2$ and $U$-statistic in \citet{Serfling2009} and utilizing the properties of $U$-statistic. A detailed proof is provided in Appendix \ref{sec:var}.

\subsection{Examples of Detection Statistics}\label{sec:eg}

Below, we present a few examples to demonstrate that the $B$-statistics is quite robust in various settings with different distributions.

\vspace{.1in}
\noindent {\bf Gaussian to Gaussian mixture.} 
In Figure \ref {offlinestat}(a), $P = \mathcal{N}(0, I_2)$, $Q$ is a mixture Gaussians: $0.3 \mathcal{N}(0,I_
2)+0.7 \mathcal{N}(0, 0.1I_2)$, and $\tau = 250$. The online procedure stops at time 270 meaning the change is detected with a small delay of 20 unit time.  

\vspace{.1in}
\noindent {\bf Sequence of graphs.} In Figure \ref {offlinestat}(b), we consider detecting the emergence of a community inside a network, which modeled using a stochastic block model, as considered by \citet{marangoni2015sequential}. Assume that before the change, each sample is a realization of an Erd\H os-R\'enyi random graph, with the probability of forming an edge $p_0=0.1$ uniformly across the graph. After the change, a ``community'' emerges, which is a subset of nodes, where the edges are formed in between these nodes with much higher probability $p_1=0.3$. The post-change distribution models a community where the members of the community interact more often. Our online procedure stops at time 102, meaning the change is detected with a small delay of 2 unit times.

\begin{figure}[h!]
        \begin{center}
        \begin{tabular}{ccc}
                     \includegraphics[width=4cm,height=4cm]{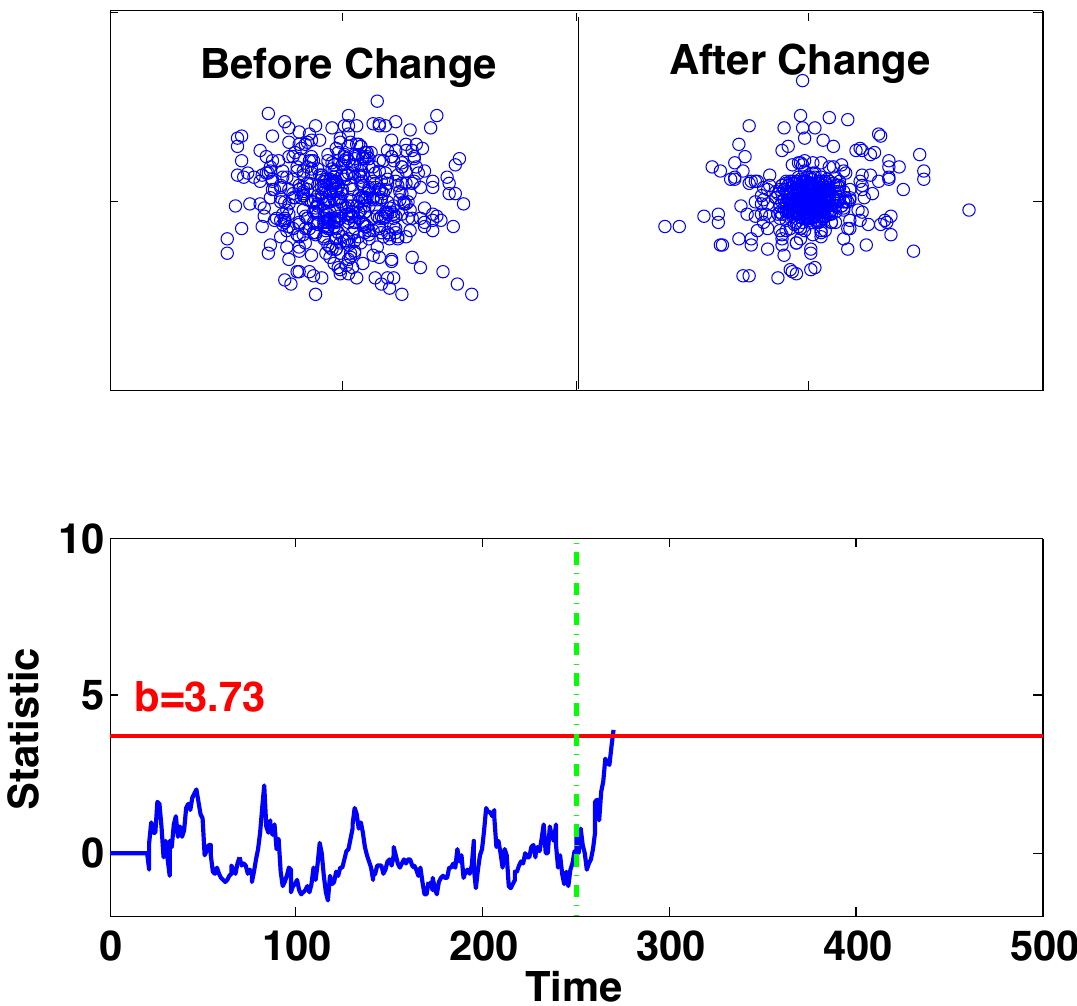}
                     &
\includegraphics[width=4cm,height=4cm]{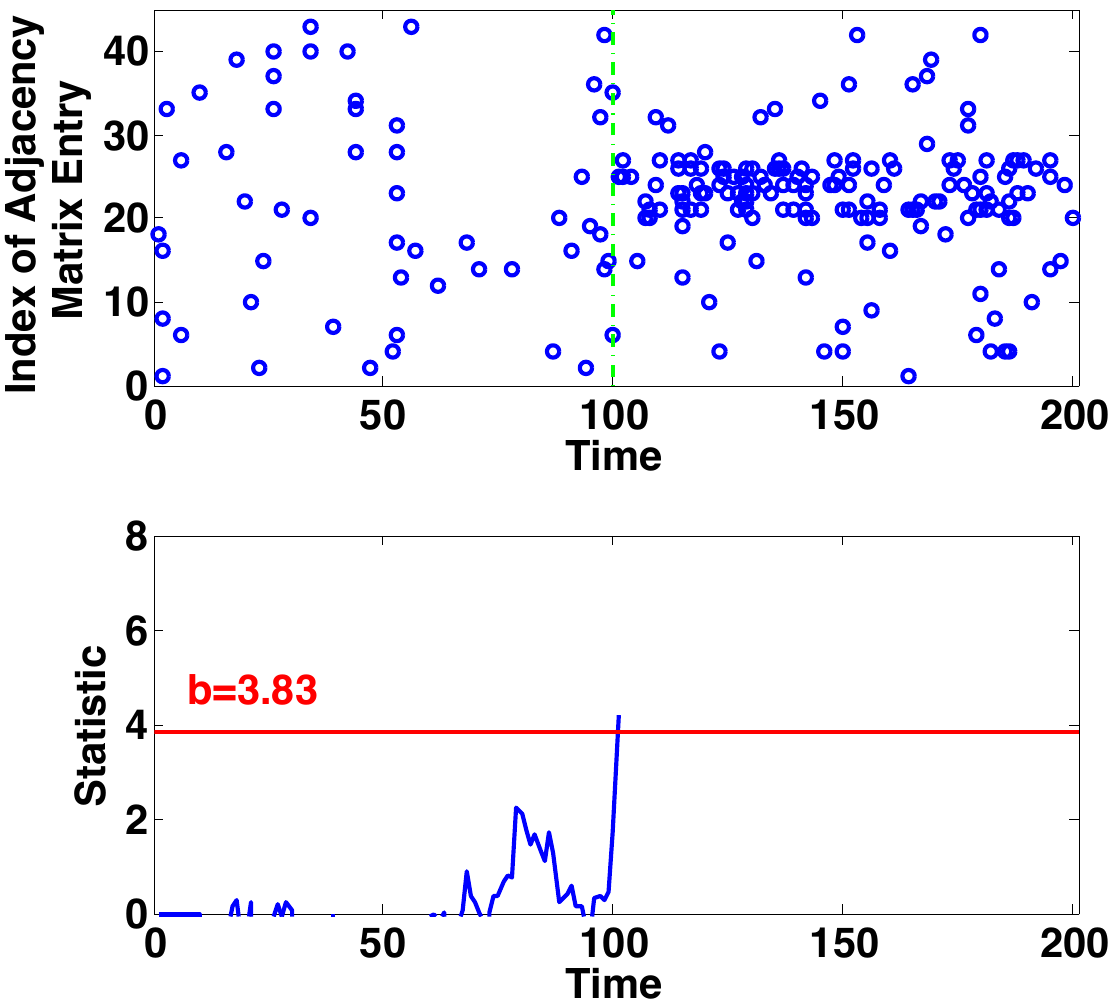}
                      & \includegraphics[width=4cm,height=4cm]{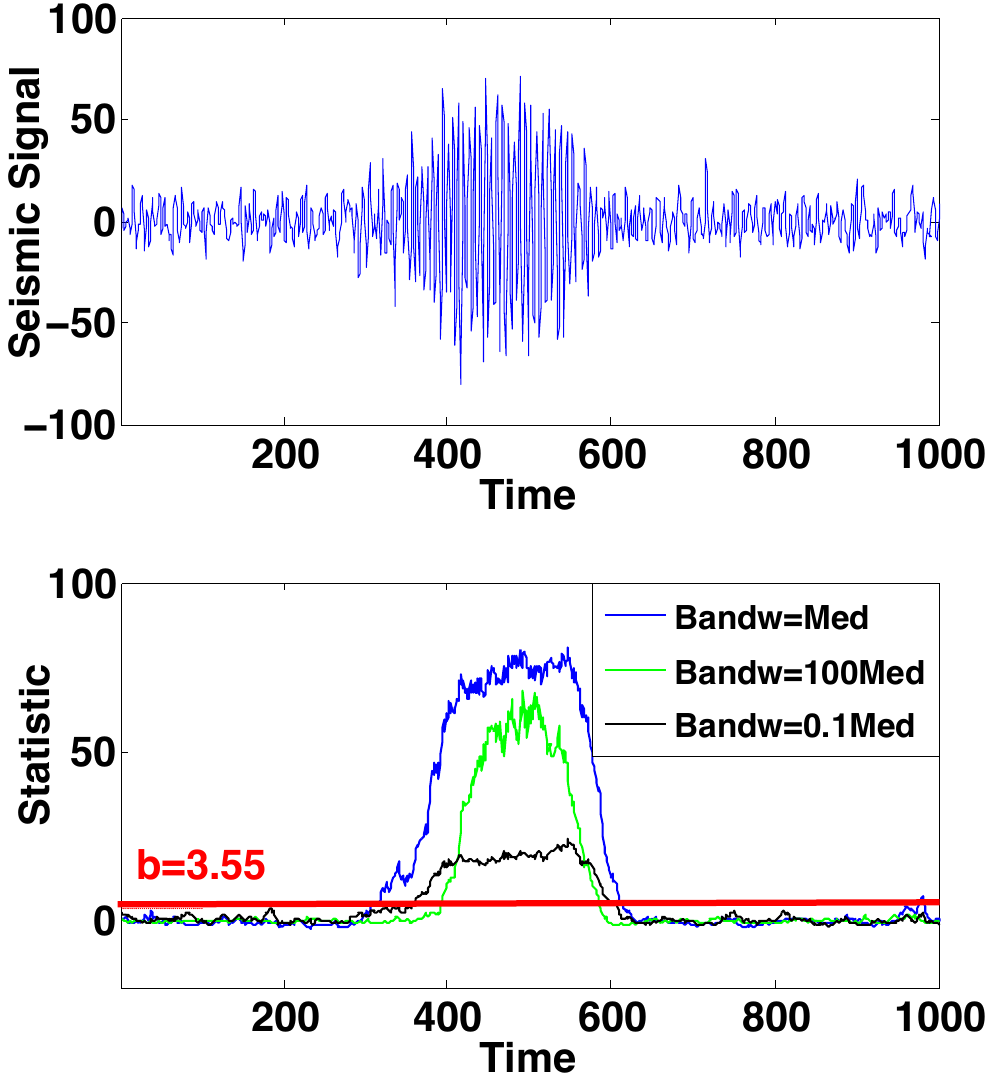}
                     \\
                     (a): Gaussian to GMM, $\tau = 250$  & (b) Graphs, $\tau=100$ & (c): Real seismic signal
                     \end{tabular}
                     \end{center}
                \caption{Examples of scan $B$-statistics with $N = 5$. All thresholds (red lines) are theoretical values obtained from Theorem \ref{thm:tail_fixedsample} (offline) and Theorem \ref{thm:ARL} (online). 
                }
                \label{offlinestat}
 \end{figure}

\vspace{.1in}
\noindent {\bf Real seismic signal; effect of kernel bandwidth.} In Figure \ref {offlinestat}(c), we consider a segment of real seismic signal that contains a change-point. Using the seismic signal, we illustrate the effect of different kernel bandwidth. For Gaussian RBF kernel $k(Y, Y')=\mbox{exp} \left(-\|Y-Y'\|^2/2 \sigma^2  \right)$,  the kernel bandwidth $\sigma >0$ is typically chosen using a ``median trick'' in \citet{scholkopf2001learning,ramdas2015decreasing}, where $\sigma$ is set to be the median of the pairwise distances between data points. 

\section{THEORETICAL APPROXIMATIONS}
\subsection{Theoretical Approximation for Significance Level of Offline Scan $B$-Statistic}\label{sec:theory}

In the offline setting, the choice of the threshold $b$ involves a tradeoff between two standard performance metrics: (1) significance level (SL), which is the probability that the statistic exceeds the threshold $b$ when the null hypothesis is true (i.e., when there is no change); and (2) power, which is the probability of the statistic exceeds the threshold when the alternative hypothesis is true.

We present an accurate approximation to the SL of the offline scan $B$-statistic, assuming the detection threshold $b$ tends to infinity and the number of blocks $N$ is fixed. The following theorem is our main result. 
 \begin{theorem}[SL of offline scan $B$-statistic]\label{thm:tail_fixedsample}
When $b\rightarrow \infty$,  and $B_{\rm max}\rightarrow \infty$, with $b/(B_{\rm max})^{1/2}$ held as a fixed positive constant, 
the significance level of the offline $B$-statistic defined in (\ref{scan-stats}) is given by
\begin{equation}
\mathbb{P} \left \{  \max_{2\leq B \leq B_{\rm max}}Z_B' > b    \right\}
=  b e^ {-\frac{1}{2}  b^2} \cdot \sum_{B=2} ^{B_{\rm max} }  
\frac{(2B-1)}{2\sqrt{2 \pi} B(B-1)}  \nu \left(  b \sqrt{\frac{2B-1}{B(B-1)}} \right)  \cdot \left[1 + o(1) \right],
\label{tail_offline}
\end{equation}
where the special function 
\begin{align}\label{eq:fucnu}
\nu (\mu) \approx \frac{(2/\mu) (\Phi(\mu/2)-0.5) }{(\mu/2)\Phi(\mu/2)+ \phi (\mu/2)},
\end{align}
$\phi(x)$ and $\Phi(x)$ are the probability density function and the cumulative distribution function of the standard normal distribution, respectively.
 \end{theorem}

Although the approximation (\ref{tail_offline}) is derived in the asymptotic regime and under the assumption that the collection of random variables $\{Z_B'\}_{B = 2,\dots, B_{\rm max}}$ form a Gaussian random field, we can show numerically that (\ref{tail_offline}) is quite accurate in the non-asymptotic regime. Consider synthetic data that are {\it i.i.d.}~normal $P=\mathcal{N}(0, I_{20})$. We set $B_{\rm max}$ to be 50, 100, 150, and in each case, $N=5$. We compare the thresholds obtained by (\ref{tail_offline}) and by simulation, for a prescribed SL $\alpha$. To obtain threshold by simulation, we generate Monte Carlo trials for offline $B$-statistics and find the $(1-\alpha)$-quantile as the estimated threshold. Table \ref{Tablestat_simu} shows that for various choices of $B_{\rm max}$, the thresholds predicted by Theorem \ref{thm:tail_fixedsample} match quite well with those obtained by simulation. The accuracy can be further improved for smaller $\alpha$ values by skewness correction as shown in Section \ref{sec:skewness}.

\begin{table}[H]
\begin{center}
\caption{Thresholds for the offline scan $B$-statistics using {\bf synthetic data}, obtained by simulation, theory (Theorem \ref{thm:tail_fixedsample}), and theory with Skewness Correction (Section \ref{sec:skewness}), respectively, for various SL values $\alpha$. } 
\begin{tabular}{|c|ccc|ccc|ccc|}
   \hline
    \multirow{2}{*}{$\alpha$} &
      \multicolumn{3}{c}{$B_{\rm max} = 50$} &
      \multicolumn{3}{|c}{$B_{\rm max}= 100$} &
      \multicolumn{3}{|c|}{$B_{\rm max} = 150$} \\
      & {$b$ (sim)} & {$b$ (theory)} &{$b$ (SC)} & {$b$ (sim)} & {$b$ (theory)} &{$b$ (SC)} & {$b$ (sim)}  &{$b$ (theory)}&{$b$ (SC)}  \\
      \hline
  0.10 & 2.41&{\bf2.38}&2.57 & 2.43 &{\bf2.50}&2.76 & 2.53 &{\bf2.56}&2.89\\ \hline
  0.05 & 2.77 &{\bf2.67}&2.97 & 2.76 &{\bf2.78} &3.17 &2.97 &{\bf2.83} &3.22\\ \hline
  0.01 & 3.54 &3.23 &{\bf3.64} &3.47 & {\bf3.32} &3.82 & 3.64 & 3.37&{\bf3.89} \\ \hline
  \end{tabular}
  \vspace{-5 mm}
  \label{Tablestat_simu}
  \end{center}
\end{table}


The complete proof of Theorem \ref{thm:tail_fixedsample} can be found in Appendix \ref{proof_main_theorem}, which leverages the change-of-measure technique. In a nutshell, we aim to find the probability of a rare event: under null the distribution, the boundary exceeding event $\left \{  \max_{2 \leq B \leq B_{\rm max}} Z_B'   > b    \right\}$ for a large threshold $b$ is rare (so that false alarm remains low). Since quantifying such a small probability is hard under the null distribution, we consider an alternative probability measure under which this boundary exceeding event happens with much higher probability. Under the new measure, one can use the local central limit theorem to a obtain an analytical expression for the probability. In the end, the original small probability will be related to the probability under the alternative measure using the Mill's ratio in \citet{YakirBook2013}.

The proof assumes the collection of random variables $\{Z_B'\}_{B = 2,\dots, B_{\rm max}}$ form a Gaussian random field (as an approximation). This means the finite-dimensional joint distributions of the collection of random variables are all Gaussian, and they are completely specified by the mean and the covariance functions, which we characterize below (this is useful for establishing Theorem \ref{thm:tail_fixedsample}). These results will be used when we quantify the tail probability of the scan $B$-statistics. 
Under the null distribution, the expectation $\mathbb{E}[Z_B']$ is zero due to the unbiased property of the MMD estimator. The covariance under the null distribution is given by the following lemma: 
\begin{lemma}[Covariance structure of $Z_B'$ in the offline setting]\label{lemma:cov_stats}
Under the null distribution, the covariance of $\{Z_B'\}_{B = 2,\dots, B_{\rm max}}$ is given by
 \begin{equation}
 r_{u, v}=  {\rm Cov} \left(  Z_u', Z_v'\right)=\sqrt{{\binom{u}{2}}{ \binom{v}{2}  }}\bigg/\binom{u \vee v }{2}, \quad 2\leq u, v \leq B_{\rm max},
 \label{r_def}
 \end{equation}
where $u \vee v = \max\{u, v\}$.
\end{lemma}
The proof can be found in Appendix \ref{App:covoffline}.

\subsection{Theoretical Approximation for ARL of Online Scan $B$-Statistic}\label{sec:online}

In the online setting, two commonly used performance metrics are (see, e.g., \citep{XieSiegmund2013}): (1) the average run length (ARL), which is the expected time before incorrectly announcing a change of distribution when none has occurred; (2) the expected detection delay (EDD), which is
the expected time to fire an alarm when 
a change occurs immediately at $\tau = 0$.  The EDD considers the worst case and provides
an upper bound on the expected delay to detect a change-point when the change occurs later in the sequence of observations. 

We present an accurate approximation to the ARL of online scan $B$-statistics. 
The approximation is quite useful in setting the threshold. As a result, given a target ARL, one can determine the corresponding threshold value $b$ from the analytical approximation, avoiding the more expensive numerical simulations. 
Our main result is the following theorem. 
\begin{theorem}[ARL in online scan $B$-statistic]\label{thm:ARL}
Let $B_0\geq 2$. When $b \to \infty$, the ARL of the stopping time $T$ defined in (\ref{stopping-time}) is given by 
\begin{equation}
\mathbb{E}[T] =  \frac{e^ {  b^2/2 }}{b} \cdot \left \{
\frac{ (2B_0-1)}{ \sqrt{2 \pi}  B_0(B_0-1)} \cdot \nu \left(  b \sqrt{\frac{2(2B_0-1)}{B_0(B_0-1)}} \right) \right \} ^{-1}\cdot \left[ 1 + o(1) \right].  
\label{ARL}
\end{equation}
\end{theorem}
The complete proof of Theorem \ref{thm:ARL} is given in Appendix \ref{theorem_ARL}.

We verify the accuracy of the approximation numerically, by comparing the thresholds obtained by Theorem \ref{thm:ARL} with those obtained from Monte Carlo simulation. Consider several cases of null distributions: standard normal
$\mathcal{N}(0, 1)$, exponential distribution with mean 1, Erd\H os-R\'enyi random graph with ten nodes and probability of 0.2 of forming random edges, as well as Laplace distribution with zero mean and unit variance. The simulation results are obtained from  5000 direct Monte Carlo trials. As shown in Figure \ref{fig:ARL},  the thresholds predicted by Theorem \ref{thm:ARL} are quite accurate.  
Figure \ref{fig:ARL} also demonstrated that theory is quite accurate for various block sizes (especially 
for larger $B_0$). However, we also note that theory tends to underestimate the thresholds. This is especially pronounced for small $B_0$, e.g., $B_0=50$. The accuracy of the theoretical results can be improved by skewness correction, shown by black lines in Figure \ref{fig:ARL}, which are discussed later in Section \ref{sec:skewness}.

Theorem \ref{thm:ARL} shows that ${\rm ARL}$ is $\mathcal{O}(e^{b^2})$ and, hence, $b$ is $\mathcal{O}((\log {\rm ARL})^{1/2})$. Note that EDD is typically on the order of $b/\Delta$ due to Wald's identity \citep{SiegmundBook1985}, where $\Delta$ is the Kullback-Leibler (KL) divergence between the null and the alternative distributions (a constant). Hence, given the desired ARL (typically on the order of 5000 or 10000), the error in the estimated threshold will only be translated linearly to EDD. This is a blessing since it means typically a reasonably accurate $b$ will cause little performance loss in EDD. Similarly, Theorem \ref{thm:tail_fixedsample} shows that ${\rm SL}$ is $\mathcal{O}(e^{-b^2})$ and a similar argument can be made for the offline case.

\begin{figure}[h!]
        \begin{center}
        \begin{tabular}{cc}
                 \includegraphics[
                 width=7cm,height=6cm
                 ]{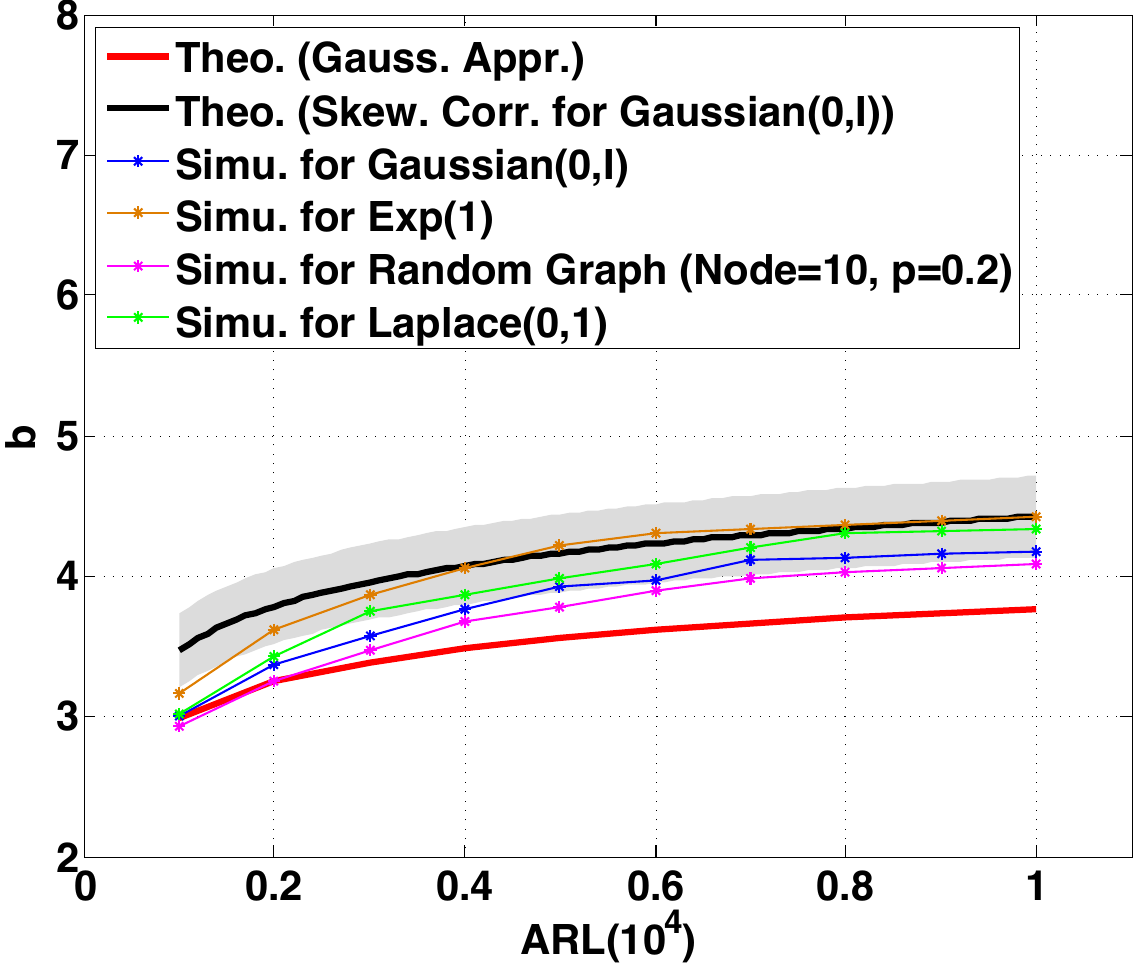}  &
                  \includegraphics[
                  width=7cm,height=6cm
                  ]{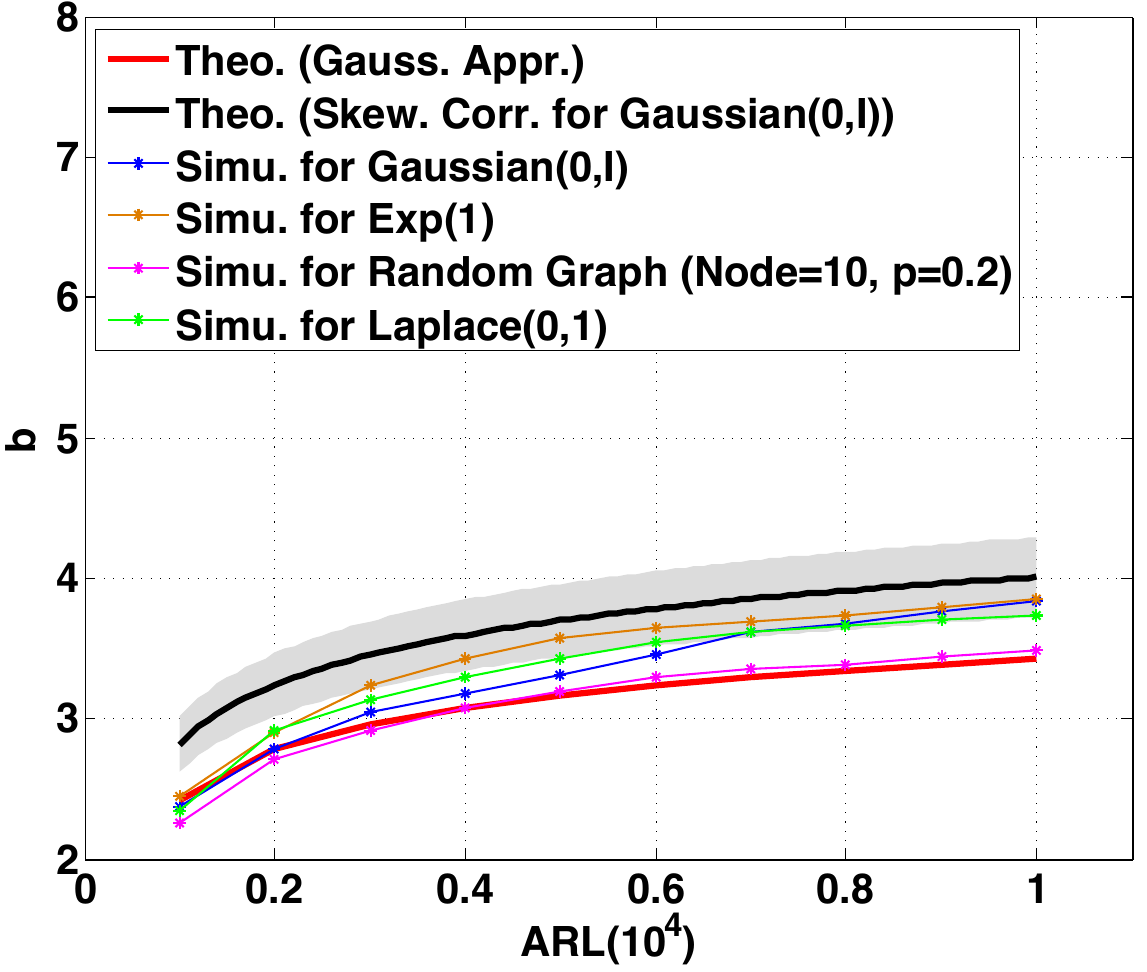}\\
                 (a): $B_0=50$ & (b): $B_0=200$ \\
                       \end{tabular}
                \caption{For a range of target ARL values, thresholds determined from simulation,  from Theorem \ref{thm:ARL}, and from theory with the skewness correction (Section \ref{sec:skewness}) under various null distributions are compared. Shaded areas represent standard deviations for skewness-corrected thresholds.}
                                \label{fig:ARL}
                \end{center}
 \end{figure}


\section{DETECTION POWER STUDY}

In this section, we study the detection power and the expected detection delay of the offline and online scan $B$-statistics, respectively, and compare them with classic methods. 

\subsection{Offline Change-Point Detection: Comparison with Parametric Statistics} 

We compare the offline scan $B$-statistic with two commonly used parametric test statistics: the Hotelling's $T^2$ and the generalized likelihood ratio (GLR) statistics. 
Assume samples $\{x_1, x_2, \dots, x_n\}$. 

\vspace{.1in}
\noindent {\bf Hotelling's $T^2$ statistic}. For a hypothetical change-point location $\tau$, we can define the Hotelling's $T^2$ statistic for samples in two segments $[1, \tau]$ and $[\tau+1, t]$ as
\[
T^2(\tau)=\frac{\tau(n-\tau)}{n}(\bar{x}_\tau -\bar{x}_\tau^* )^T\widehat{ \Sigma}^{-1} (\bar{x}_\tau -\bar{x}_\tau^*),
\]
where, $\bar{x}_\tau=\sum_{i=1}^\tau x_i /\tau$, $\bar{x}_\tau^*=\sum_{i=\tau+1}^n x_i/(n-\tau)$ and the pooled covariance estimator 
\[\widehat{ \Sigma}=(n-2)^{-1} \left( \sum_{i=1}^\tau(x_i-\bar{x}_i) (x_i-\bar{x}_i)^T + \right. \left. \sum_{i=\tau+1}^n(x_i-\bar{x}_i^*) (x_i-\bar{x}_i^*)^T\right).\]
The Hotelling's $T^2$ test detects a change whenever $\max_{1\leq \tau \leq n} \max T^2(\tau)$ exceeds a threshold.

\vspace{.1in}
\noindent The {\bf generalized likelihood ratio (GLR)} statistic can be derived by assuming the null and the alternative distributions are two multivariate normal distributions, and both the mean and the covariance matrix are all unknown. For a hypothetical change-point location $\tau$, the GLR statistic is given by
\[\ell (\tau)=n \mbox{log} |\widehat{\Sigma}_n|-\tau \mbox{log}|\widehat{\Sigma}_\tau|-(n-\tau)\mbox{log}|\widehat{\Sigma}_\tau^*|,\]
where $ \widehat{\Sigma}_\tau=\tau^{-1}\left( \sum_{i=1}^\tau (x_i-\bar{x}_i) (x_i-\bar{x}_i)^T   \right),$ and 
$ \widehat{\Sigma}_\tau^*=(n-\tau)^{-1}  \sum_{i=\tau+1}^n(x_i-\bar{x}_i^*) (x_i-\bar{x}_i^*)^T.$ The GLR statistic detects a change whenever $\max_{1\leq \tau \leq n} \ell(\tau)$ exceeds a threshold. 

For our examples, we set $n = B_{\rm max}=200$ for the Hotelling's $T^2$ and the scan $B$-statistics, respectively. Let the change-point occurs at $\tau=100$, and choose the significance level $\alpha = 0.05$. The thresholds for the offline scan $B$-statistic are obtained from Theorem \ref{thm:tail_fixedsample}, and those for the other two methods the thresholds are obtained from simulations.  Consider the following cases:

\textit{Case 1} (mean shift): observe a sequence of observations in $\mathbb{R}^{20}$, whose distribution shifts from $ \mathcal N ({\bf 0}, I_{20})$ to $\mathcal N (0.1{\bf e}, I_{20})$; 

\textit{Case 2} (mean shift with larger magnitude): observe a sequence of observations in $\mathbb{R}^{20}$, whose distribution shifts from $ \mathcal N ({\bf 0}, I_{20})$ to $\mathcal N(0.2{\bf e}, I_{20})$; 

\textit{Case 3} (mean and local covariance change): observe a sequence of observations in $\mathbb{R}^{20}$, whose distribution shifts from $ \mathcal N ({\bf e},I_{20})$ to $\mathcal N (0.2{\bf e}, \Sigma),$ where $[\Sigma]_{11} = 2$ and $[\Sigma]_{ii} = 1,\, i=2,\dots, 20$; 

\textit{Case 4} (Gaussian to Laplace): observe a sequence of one-dimensional observations, whose distribution shifts from $\mathcal N(0,1)$ to Laplace distribution with zero mean and unit variance. Note that the mean and the variance remain the same after the change. 

We estimate the power for each case using 100 Monte Carlo trials. Table \ref{offlinepower} shows that the scan $B$-statistic achieves higher power than the Hotelling's $T^2$ statistic as well as the GLR statistic in all cases. The GLR statistic performs poorly, since when $\tau$ is small or closer to the end point, it estimates the pre-change and post-change sample covariance matrix using a very limited number of samples.

\begin{table}  
\begin{center}
\caption{Comparison of detection power for offline change-point detection. Thresholds for all methods are calibrated so that the significance level is $\alpha = 0.05$.}
\vspace{.1in}
\small{
  \begin{tabular}{|c|c|c|c|c|c|}
  \hline
                  & Case 1  & Case 2 & Case 3 & Case 4 \\ \hline
$B$-statistic  & \bf 0.71 &\bf 1.00   & \bf 1.00 & \bf 0.44        \\ \hline
Hotelling's $T^2$  & 0.18 & 0.88  & 0.87  &  0.03   \\  \hline
GLR  &  0.03 & 0.05 & 0.12 & 0.04  \\ \hline
  \end{tabular}
  }
  \vspace{-5 mm}
 \label{offlinepower}
\end{center}
\end{table}

\subsection{Online Change-Point Detection: Comparison with Hotelling's $T^2$ Statistic}

Now consider the online scan $B$-statistic with a fixed block-size $B_0=20$. We compare the online scan $B$-statistic with a Shewhart chart based on Hotelling's $T^2$ statistic\footnote{Here we made no comparison of the online scan $B$-statistic with the GLR statistic, since in our experiments, Hotelling's $T^2$ consistently outperforms GLR when the dimension is high.}. At each time $t$, we form a Hotelling's $T^2$ statistic using the immediately past $B_0$ samples in $[t-B_0 + 1, t]$, 
\[T^2(t)=B_0(\bar{x}_t-\hat{\mu})^T \widehat{\Sigma}_0^{-1}(\bar{x}_t-\hat{\mu}_0),\] where $\bar{x}_t = (\sum_{i=t-B_0+1}^t x_i)/B_0$, and $\hat{\mu}_0 $ and $\widehat{\Sigma}_0$ are estimated from reference data. The procedure detects a change-point whenever $T^2(t)$ exceeds a threshold for the first time. 
The threshold for online scan $B$-statistic is obtained from Theorem \ref{thm:ARL}, and from simulations for the Hotelling's $T^2$ statistic.  
To simulate EDD, let the change occur at the first point of the testing data. 
Consider the following cases: 

\textit{Case 1} (mean shift): distribution shifts from $\mathcal N ({\bf 0}, I_{20})$ to $\mathcal N (0.3  {\bf 1}, I_{20})$; 

\textit{Case 2} (covariance change): distribution shifts from $ \mathcal N ({\bf 0}, I_{20})$ to $\mathcal N ({\bf 0}, \Sigma),$ where $[\Sigma]_{ii}=2$, $i=1,2,\dots, 5$ and $[\Sigma]_{ii}=1,\, i=6,\dots, 20$; 

\textit{Case 3} (covariance change): distribution shifts from $ \mathcal N ({\bf 0}, I_{20})$ 
to $\mathcal N ({\bf 0}, 2 I_{20})$; 

\textit{Case 4} (Gaussian to Gaussian mixture): distribution shifts from $ \mathcal N ({\bf 0}, I_{20})$ to mixture Gaussian $0.3\mathcal N ({\bf 0}, I_{20})+0.7\mathcal N ({\bf 0}, 0.1 I_{20})$; 

\textit{Case 5} (Gaussian to Laplace)\footnote{For these difficult situations, we report the EDD comparisons based on the selected 500 sequences where $B$-statistics successfully detect the changes, which are defined as crossing the threshold within 50 steps from the time that the change occurs. Hotelling's $T^2$ fails to detect the changes for all sequences.}: distribution shifts from $\mathcal N (0,1)$ to Laplace distribution with zero mean and unit variance.

We evaluate the EDD for each case using 500 Monte Carlo trials. The results are summarized in Table \ref{EDD}. Note that in detecting changes in either Gaussian mean or covariance, the online scan $B$-statistic performs competitively with Hotelling's $T^2$, which is tailored to the Gaussian distribution. In the more challenging scenarios such as Case 4 and Case 5, the Hotelling's $T^2$ fails to detect the change-point whereas the online scan $B$-statistic can detect the change fairly quickly. 
\begin{table}  
\begin{center}
\caption{Comparison of EDD in online change-point detection. The parameter is $B_0=20$ and thresholds for all methods are calibrated so that $\mbox{ARL}= 5000.$ Dashed line means that the procedure fails to detect the change, i.e., EDD is longer than 50.}
\vspace{.1in}
\small{
  \begin{tabular}{|c|c|c|c|c|c|}
  \hline
          & Case 1  & Case 2  & Case 3   & Case 4& Case 5  \\ \hline
$B$-statistic      & 4.20       &\bf{9.10}&     \bf{1.00}    & \bf{23.38}  &  \bf{23.03}  \\\hline
Hotelling's $T^2$  & \bf{2.47}  & 25.46  & 1.27  & $-$  &  $-$   \\ \hline
  \end{tabular}
  }
  \vspace{-5 mm}
 \label{EDD}
\end{center}
\end{table}

\section{SKEWNESS CORRECTION}\label{sec:skewness}

We have shown that approximations to the significance level and ARL, assuming that random variables $\{Z_B'\}_{B = 2,3,\dots}$ form a Gaussian random field, are reasonably accurate. However, $Z_B'$ does not converge to normal distribution even when $B$ is large (see Appendix \ref{app:scewness}) and it has a non-vanishing skewness, as illustrated by the following numerical example. 
Form 10000 instances of $Z_B$ computed using samples from $\mathcal{N}(0, I_{20})$.  Figures \ref{statisticdist}(a)-(b) show the empirical distributions of $Z_B$ when $N =5$, and $B=2$ or $B=200$, respectively. Also plotted are the Gaussian probability density functions with mean equal to the sample mean, and the variance predicted by Lemma \ref{thm:variance}. Note that the empirical distributions of $Z_B$ match with Gaussian distributions to a certain extent but the skewness becomes larger for larger $B$. Figures \ref{statisticdist}(c)-(d) show the corresponding Q-Q plots.

\begin{figure}[h]
        \begin{center}
        \begin{tabular}{cc}
                \includegraphics[width=6.2cm,height=3.8cm
                ]{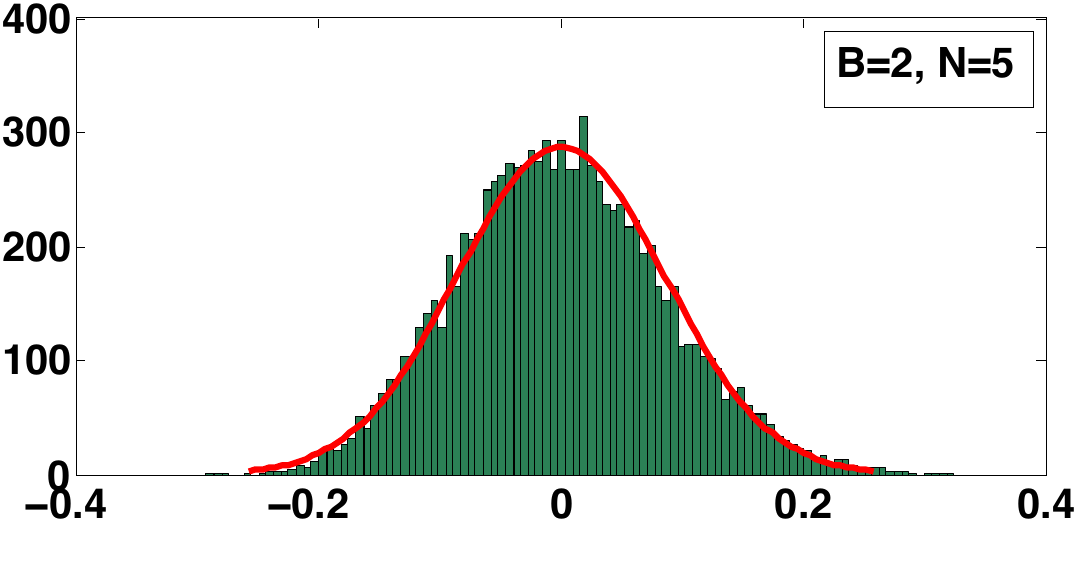} &
                 \includegraphics[width=6.2cm,height=3.8cm
                 ]{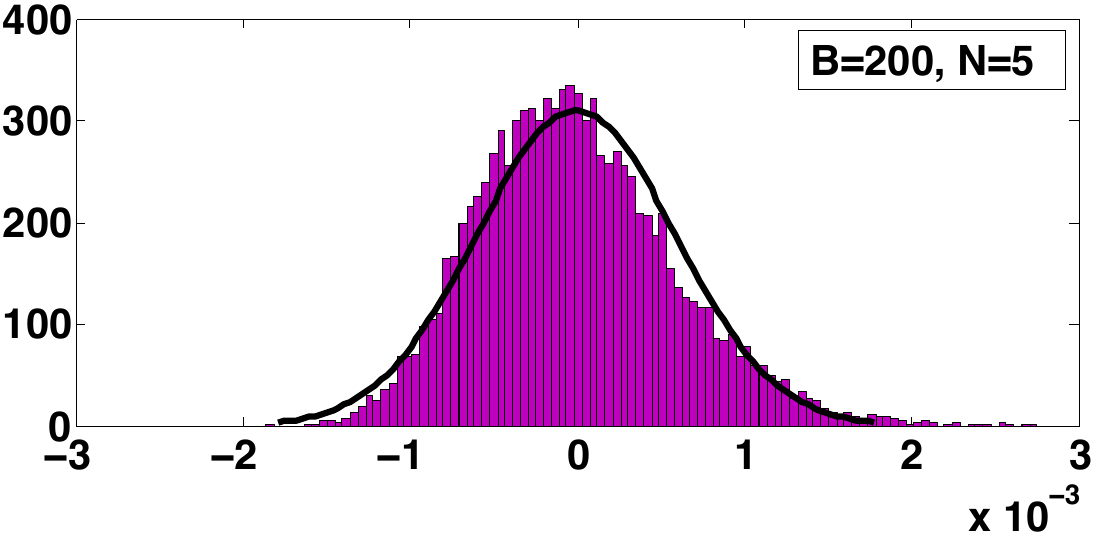} \\  (a): $B = 2$, $N = 5$, empirical distribution &
                 (b): $B = 200$, $N = 5$, empirical distribution \\
    \vspace{.05in}            \\
    
    \includegraphics[width=6.5cm,height=4cm
                ]{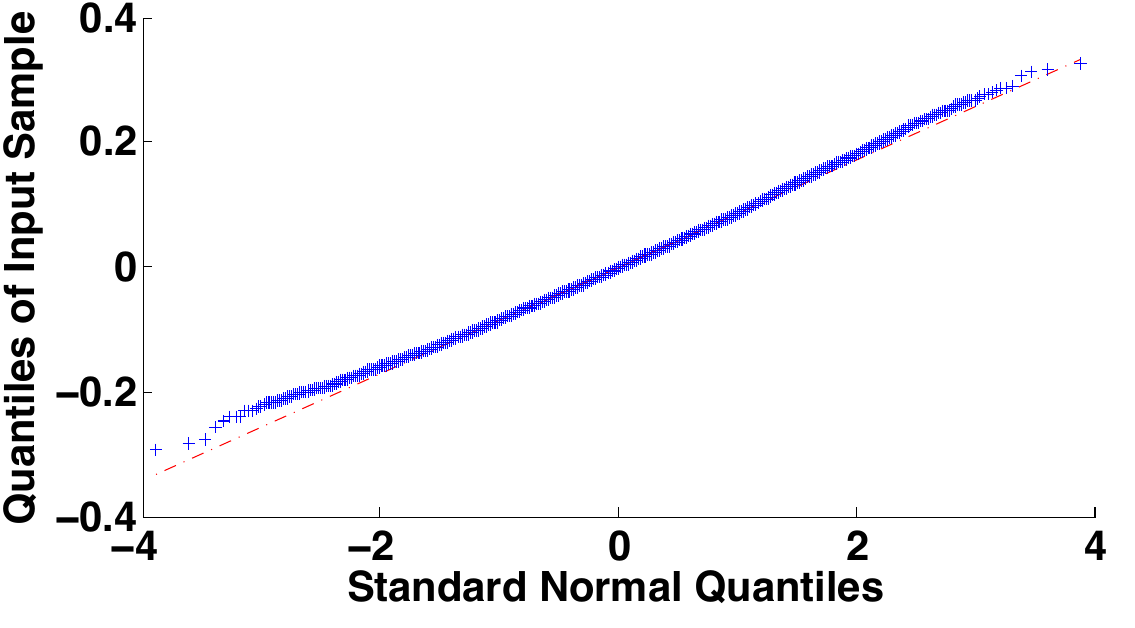} & 
                \includegraphics[width=6.5cm,height=4cm
                  ]{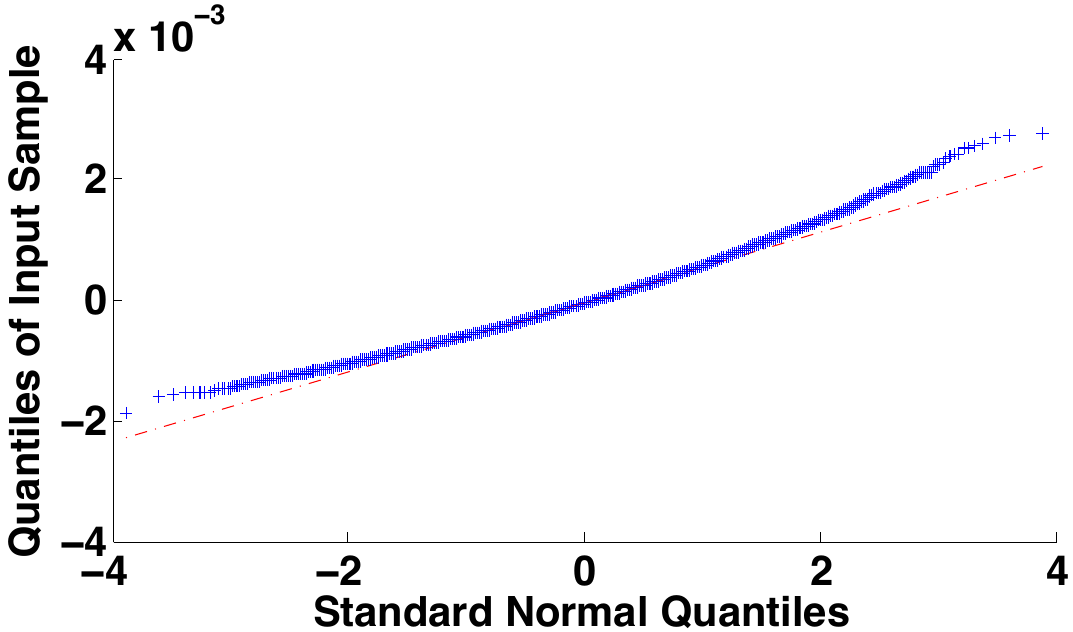} \\
                  (c): $B = 2$, $N = 5$, Q-Q plot &
                 (d): $B = 200$, $N = 5$, Q-Q plot  
                
          \end{tabular}
                \caption{Empirical distributions of $Z_B$ when $B=2$ and $B=200$, respectively. Note that although Gaussian distribution seems to be a reasonable fit to the statistic, the skewness becomes larger for larger values of $B$.}
                                \label{statisticdist}
                \end{center}
 \end{figure}

To incorporate the skewness of $Z_B$, one can improve the accuracy of the approximations for significance level in Theorem \ref{thm:tail_fixedsample} and for ARL in Theorem \ref{thm:ARL}. Note that the log moment generating function $\psi(\theta)$ defined in (\ref{phi}) corresponds to the cumulant generating function \citep{McCullagh4699} and it has an expansion for $\theta$  close to zero:
\[
\psi(\theta) = \kappa_1 \theta + \frac{\kappa_2}{2}\theta^2 + \frac{\kappa_3}{3!}\theta^3 + o(\theta^3). 
\]
Since $\mathbb{E}[Z_B'] = 0$, the cumulants take values
 $\kappa_1 = \mathbb{E}[Z_B'] = 0$, $\kappa_2 = {\rm Var}[Z_B'] = 1$, $\kappa_3 = \mathbb{E}[(Z_B')^3]- 3\mathbb{E}[(Z_B')^2]\mathbb{E}[Z_B'] + 2(\mathbb{E}[Z_B'])^3 = \mathbb{E}[(Z_B')^3]$. 
Recall that when deriving  approximations using change-of-measurement, we choose parameter $\theta$ such that $\dot{\psi}(\theta) = b$. If $Z_B'$ is a standard normal,  $\psi(\theta) = \theta^2/2$, and hence $\theta = b$. Now with skewness correction, we approximate $\psi(\theta)$ as $\theta^2/2 + \kappa_3\theta^3/6$ when solving for $\theta$. Hence, we solve for \[\dot{\psi}(\theta) \approx \theta + \mathbb{E}[(Z_B')^3]\theta^2/2 = b,\] 
and denote the solution to be $\theta_B$ (note that this time the solution depends on $B$). 
Moreover, with skewness correction, we will change the leading exponent term in (\ref{tail_offline}) and (\ref{tail-online}) from $e^ {-  b^2/2}$ to be $e^{\psi(\theta_B')-\theta_B' b}$. 

From numerical experiments, we find that the skewness correction is especially useful when the significance level is small (e.g., $\alpha = 0.01$) for the offline case, when block size $B_0$ is small (see Table \ref{Tablestat_simu} and Fig. \ref{fig:ARL}), and can be important for real data where the data are noisy and the null distribution is more difficult to characterize. 

For example, we consider real speech data from the CENSREC-1-C dataset (more details in Section \ref{sec:real_data}). Here, the null distribution $P$ corresponds to the unknown distribution of the background signal, and we are interested in detecting the onset of speech signals. This case is more challenging because the true distribution can be arbitrary. In the dataset, there are 3000 reference samples. We bootstrap these reference samples to generate 10000 re-samples to estimate the tail of the detection statistic. Table \ref{Tablestat_boot} demonstrates that the thresholds predicted by the expensive bootstrapping, by Theorem \ref{thm:tail_fixedsample}), and by theory with skewness correction, respectively, for various SL values $\alpha$. Note that in this case, the accuracy improves significantly by skewness correction. 

\begin{table}[h!]
\begin{center}
\caption{Thresholds for the offline scan $B$-statistics using {\bf real speech data}, obtained by simulation, theory (Theorem \ref{thm:tail_fixedsample}), and theory with skewness correction, respectively, for various significance levels $\alpha$. 
}
\begin{tabular}{|c|ccc|ccc|ccc|}
   \hline
    \multirow{2}{*}{$\alpha$} &
      \multicolumn{3}{c}{$B_{\rm max} = 50$} &
      \multicolumn{3}{|c}{$B_{\rm max}= 100$} &
      \multicolumn{3}{|c|}{$B_{\rm max} = 150$} \\
      &{$b$ (boot)}&{$b$ (theory)} &{$b$ (SC)} &{$b$ (boot)}&{$b$ (theory)} &{$b$ (SC)}&{$b$ (boot)} &{$b$ (theory)}&{$b$ (SC)}  \\
      \hline
  0.10 & 2.96 &2.38&{\bf3.23} &3.16&2.50 &\textbf{3.59}  &3.21&\textbf{2.56}&3.94\\ \hline
  0.05 & 3.62&2.67&{\bf 3.68}&3.82&2.78 &\textbf{4.06} &3.86& 2.83&\textbf{4.43} \\ \hline
  0.01 & 4.85&3.23&{\bf 4.61} &5.20&3.32 & \textbf{5.03}  &5.42&3.37& \textbf{5.45} \\ \hline
  \end{tabular}
 \vspace{-5 mm}
  \label{Tablestat_boot}
  \end{center}
\end{table}

The remaining task is to estimate the skewness of scan $B$-statistic. 
Since $Z_B$ is zero-mean, the skewness of $Z_B'$ is related to the variance and third moment of $Z_B$ via
\[
\kappa_3=\mathbb{E}[(Z_B')^3] = {\rm Var}[Z_B]^{-3/2}\mathbb{E}[Z_B^3]. 
\]
We already know how to estimate the variance of $Z_B$ from Lemma \ref{thm:variance}. 
The following lemma shows the third-order moment $\mathbb{E}[Z_B^3]$ in terms of the moments of the kernel $h$ defined in (\ref{kernel-to-kernel}): 
\begin{lemma}[Third-order moment of $Z_B$]\label{thm: skewness}
\begin{equation}
\begin{split}
\mathbb{E}[Z_B^3] &= \frac{8(B-2)}{B^2(B-1)^2} \left\{  \frac{1}{N^2} \mathbb{E} \left[ h(x,x', y,y') h(x',x'', y', y'') 
h(x'',x, y'', y) \right]  \right. \\
&  \qquad \left.  +\frac{3(N-1)}{N^2} \mathbb{E} \left[ h(x,x', y,y') h(x',x'', y', y'') 
h(x''',x'''', y'', y) \right] \right. \\
& \qquad  \left.       +\frac{(N-1)(N-2)}{N^2}    \mathbb{E} \left[ h(x,x', y,y') h(x'',x''', y', y'') 
h(x'''',x''''', y'', y) \right]         \right\}  \\
& ~~~~+ \frac{4}{B^2(B-1)^2} \left\{     \frac{1}{N^2}\mathbb{E} \left[ h(x,x',y,y')^3\right] \right. \\
& \qquad \left. +\frac{3(N-1)}{N^2}\mathbb{E} \left[ h(x,x',y,y')^2h(x'',x''',y,y')\right]  \right. \\
& \qquad \left. +   \frac{(N-1)(N-2)}{N^2}    \mathbb{E} \left[ h(x,x',y,y')h(x'',x''',y,y')h(x'''',x''''',y,y')\right]    \right\} .
\end{split}
\label{skewness_eqn}
\end{equation}
\end{lemma}
The proof can be found in Appendix \ref{proof:skewness}. Lemma \ref{thm: skewness} enables us to estimate the skewness efficiently, by reducing it to evaluating simpler terms in (\ref{skewness_eqn}) that only requires estimating the statistic of the kernel function $h(\cdot, \cdot, \cdot, \cdot)$ with tuples of samples. 

Finally, although $Z_B'$ does not converge to Gaussian, the difference between its moment generating functions and that of the standard normal distribution can be bounded, as we show below. 
By applying an argument on Page 220 of \citet{YakirBook2013}, we obtain that
\[
\left|\mathbb{E}[e^{\theta Z_B'}] - (1 + \frac{\theta^2}{2}) \right| \leq \min\{\frac{|\theta|^3}{6} \mathbb{E}[|Z'_B|^3], \theta^2 \mathbb{E}[|Z'_B|^2]\}.
\]
If considering the skewness $\kappa_3$ of $Z_B'$, we have a better estimation
$$\left|\mathbb{E}[e^{\theta Z_B'}] - (1+ \frac{\theta^2}{2} + \frac{\theta^3 \kappa_3}{6}) \right|  \leq \min \{\frac{ \theta^4  }{24} \mathbb{E}[|Z'_B|^4], \frac{1}{3}|\theta|^3 \mathbb{E}[|Z'_B|^3]\}. $$

\section{REAL DATA} \label{sec:real_data}

We test the performance of the scan $B$-statistics for change-point detection on real data. Our datasets include: (1) CENSREC-1-C: a real-world speech data set in the Speech Resource Consortium (SRC) corpora provided by National Institute of Informatics (NII)\footnote{Available from http://research.nii.ac.jp/src/en/CENSREC-1-C.html}; (2) Human Activity Sensing Consortium (HASC) challenge 2011 data\footnote{Available from http://hasc.jp/hc2011}. We compare our proposed scan $B$-statistics with a baseline algorithm, the relative density-ratio (RDR) estimate \citep{density-ratio2013}. One limitation of the RDR algorithm, however, is that it is not suitable for high-dimensional data because estimating density ratio in the high-dimensional setting is an ill-posed problem.  
To achieve reasonable performance for the RDR algorithm, we adjust the bandwidth and the regularization parameter at each time step and, hence, the RDR algorithm is computationally more expensive than using the scan $B$-statistics. We adopt the standard Area Under Curve (AUC) as in \citet{density-ratio2013} for our performance metric. The larger the AUC, the better. 

Our scan $B$-statistics demonstrate competitive performance compared with the baseline RDR algorithm on the real data. Here we only report the main results and leave the details in Appendix \ref{app:real_data}. For speech data, our goal is to online detect the emergence of a speech signal from the background. The backgrounds are taken from real acoustic signals, such as noise recorded in highway, airport and subway stations. The overall AUC for the scan $B$-statistic is {\bf 0.8014} and for the baseline algorithm is {\bf 0.7578}. For human activity detection data, our goal is to detect a transition from one activity to another as quickly as possible. Each instance consists of six possible human activity signals collected by portable three-axis accelerometers. The overall AUC for the scan $B$-statistic is {\bf 0.8871} and for the baseline algorithm is {\bf 0.7161}. 

\section{DISCUSSIONS}

There are a few possible directions to extend our work. (1) Thus far, we have assumed that data are {\it i.i.d.}~from a null distribution $P$ and when the change happens, data are {\it i.i.d.}~from an alternative distribution $Q$. Under these assumptions, we have developed the offline and online change-point detection algorithms based on the two-sample nonparametric test statistic MMD. One may relax the temporal independence assumption and extend scan $B$-statistics for dependent data by incorporating ideas from \citet{chwialkowski2014kernel}. (2) We have demonstrated how the number of blocks and block size affect the performance of scan $B$-statistics. One can also explore how kernel bandwidth, as well as the dimensionality of the data, would affect the performance. An empirical observation is that the performance of MMD statistic degrades with the increasing dimensions of data. Some recent results for the kernel-based test can be found in \citet{ramdas2015decreasing}. 
We may adopt the idea of \citet{ramdas2015decreasing} to extend our scan $B$-statistics for detecting a change in high dimensions. (3) For an exceedingly high dimensional data set with large Gram matrix, one can perform random subsampling to reduce complexity similar to \citet{XieLiaSon15}. 

\section*{ACKNOWLEDGEMENTS}

This research was supported in part by NSF CMMI-1538746, NSF CCF-1442635, NSF CAREER CCF-1650913, DMS-1830210, a grant from Atlanta Police Foundation, gift donation from Adobe Research to Yao Xie; NSF/NIH BIGDATA
1R01GM108341, ONR N00014-15-1-2340, NSF IIS-1218749, NSF IIS-1639792, NSF CAREER IIS-1350983, grant from Intel and NVIDIA to Le Song.

\newpage


\clearpage
\appendix

\section{RECURSIVE IMPLEMENTATION OF ONLINE SCAN STATISTIC}\label{recursive}
The online scan $B$-statistic can be computed recursively via a simple update scheme. By its construction, when time elapses from $t$ to $(t+1)$, a new sample is added into the post-change block, and the oldest sample is moved to the reference pool. Each reference block is updated similarly by adding one sample randomly drawn from the pool of reference data, and the oldest sample is purged. Hence, only a limited number of entries in the Gram matrix due to the new sample will be updated. The update scheme is illustrated in Fig. \ref{kerelmatrix} and explained in more details therein. 
%
Similarly, the offline scan $B$-statistic can also be computed recursively by utilizing the fact that $Z_B$ for $B \in \{2, \dots, B_{\rm max}\}$ shares many common terms. 

\begin{figure}[h!]
        \begin{center}
        \begin{tabular}{cc}
                \includegraphics[width = .60\textwidth
                ]{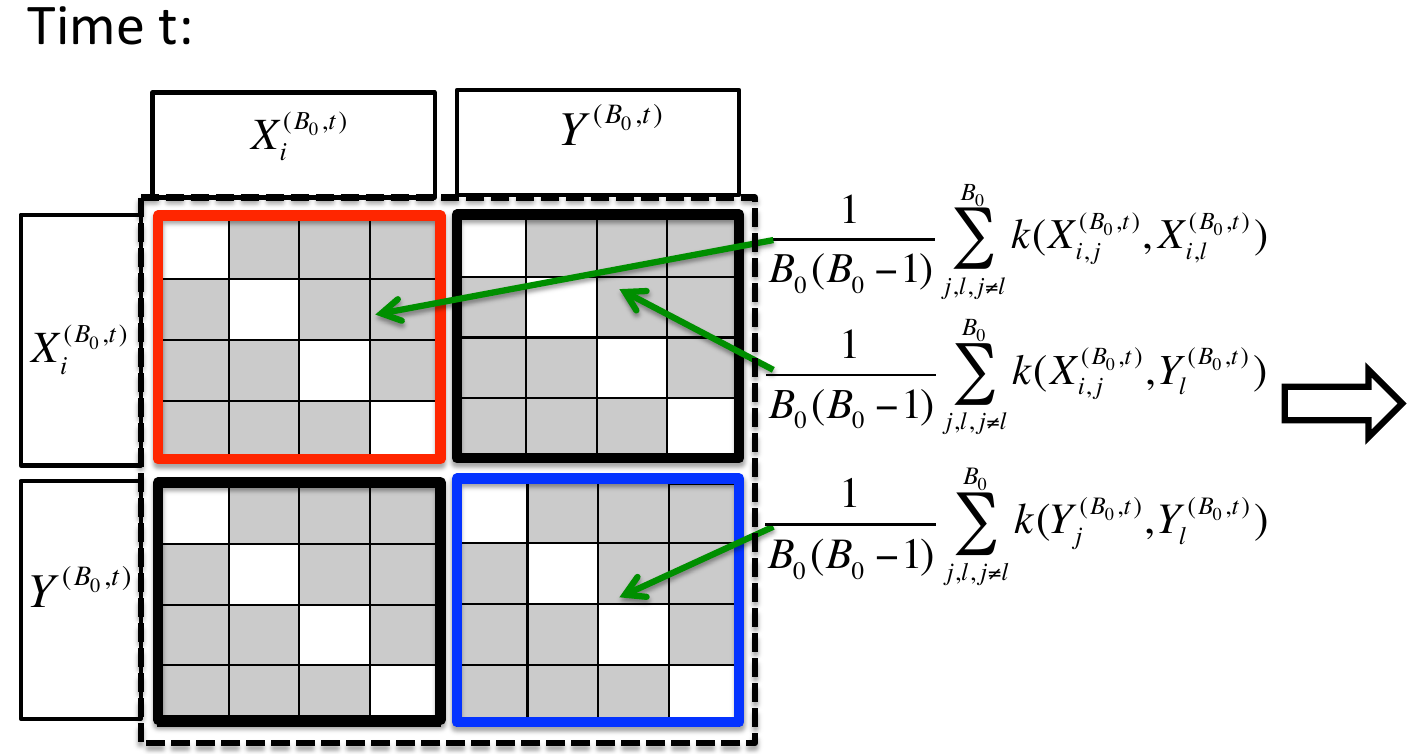} 
                 &
                    \includegraphics[width = .32\textwidth]{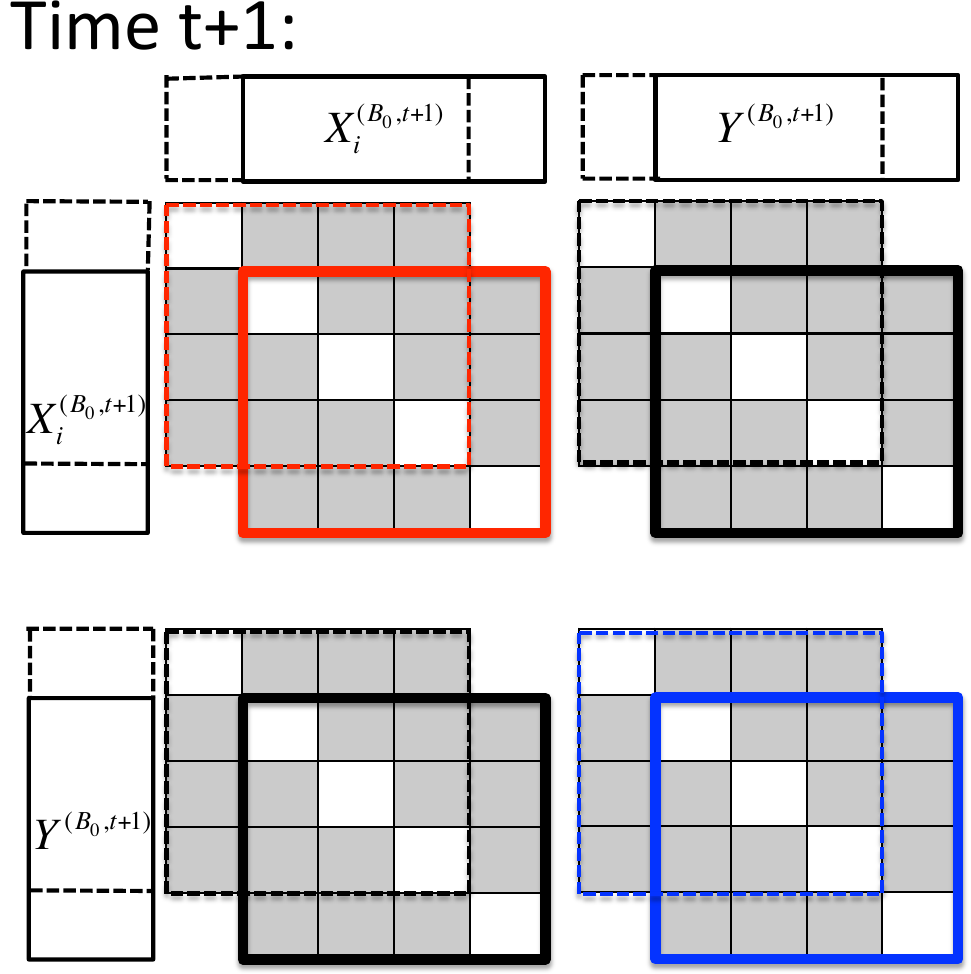}
                   \end{tabular}
\vspace{.1in}
                 \end{center}
                \caption{Recursive update scheme to compute the online scan $B$-statistics. The online $B$-statistic is formed with $N$ background blocks and one testing block and, hence, we keep track of $N$ Gram matrices. For illustration purposes, we partition the Gram matrix into four windows (in red, black and blue, as shown on the left panel). At time $t$, to obtain $\mbox{MMD}^2(X_i^{(B_0,t)}, Y^{(B_0, t)})$, we compute the shaded elements and take an average within each window. The diagonal entries in each window are removed to obtain an unbiased estimate. At time $t+1$, we update $X_i^{(B_0, t)}$ and $Y^{(B_0,t)}$ with the new data point and purge the oldest data point, and update the Gram matrix by moving the colored window as shown on the right panel. We compute the elements within the new windows, and take an average. Note that 
               we only need to compute the right-most column and the bottom row. 
}
               \label{kerelmatrix}
 \end{figure}

\section{VARIANCE AND COVARIANCE CALCULATION}\label{sec:var}

Below, $X_{i,j}^{(B)}$, where $i = 1, \dots, N$, and $j = 2, \dots, B_{\rm max}$, denotes the $j$-th sample in the $i$-th block $X_i^{(B)}$, and $Y_{j}^{(B)}$ denotes the $j$-th sample in $Y^{(B)}$. The superscript $B$ denotes the block size. 
We start with proving Lemma \ref{Lemma:var_b_test} and Lemma \ref{Lemma:cov}, which are useful in proving Lemma \ref{thm:variance}.
\begin{lemma}[Variance of MMD, under the null.]\label{Lemma:var_b_test}

Under the null hypothesis, 
\begin{equation}
{\rm Var} \left[ { \rm MMD}^2(X_i^{(B)}, Y^{(B)})  \right]={\binom{B}{2}}^{-1}\mathbb{E} [h^2(x,x',y,y')], \quad i = 1, \ldots, N.
\label{var_expr}
\end{equation}
 \end{lemma}
 \begin{proof}
For notational simplicity, below we drop the superscript $B$, which denotes the block size. Furthermore, we use $x$, $x'$, $y$ and $y'$ to denote generic samples, i.e., $X_{i,l} \stackrel{d}{=} x$, $X_{i,j} \stackrel{d}{=} x'$, $Y_l \stackrel{d}{=} y$, $Y_j \stackrel{d}{=} y'$ and they are mutually independent of each other. Here the notation $\stackrel{d}{=}$ means two random variables have the same distribution. Below, we follow the same convention. For any $i=1,2,\dots, n$, by definition of U-statistic, we have
\begin{equation}
\begin{split}
 &{\rm Var} \left[ {  \rm MMD }^2(X_i, Y)  \right]
 = {\rm Var} \left[    {\binom{B}{2} }^{-1}  \sum_{l < j}  h(X_{i,l}, X_{i,j}, Y_l, Y_j  )  \right]   \\
 &=   {\binom{B}{2} }^{-2}  \left[  \binom{B}{2} \binom{2}{1} \binom{B-2}{2-1}    {\rm Var} \left[    \mathbb{E}_{x,y} [h(x,x',y,y')]  \right] \right. \\
& \left. \quad +    \binom{B}{2} \binom{2}{2} \binom{B-2}{2-2}  {\rm Var}  \left[ h(x,x',y,y')   \right]
    \right].  
\end{split} \label{var_int_expr} 
\end{equation}
Under null distribution, $\mathbb{E}_{x, y} [h(x,x',y,y')]=0$. Thus,
${\rm Var} \left[    \mathbb{E}_{x_iy} [h(x,x',y,y')]  \right]=0, $
and
\[ {\rm Var}  \left[ h(x,x',y,y')   \right]=\mathbb{E} [h^2(x,x',y,y')]-  \mathbb{E} [h(x, x',y,y')]^2 
        =\mathbb{E} [h^2(x,x',y,y')].\]
Substitute these results into (\ref{var_int_expr}), and we obtain the desired result (\ref{var_expr}).

 \end{proof}
 
 \begin{lemma}[Covariance of MMD, under the null, different block index.]\label{Lemma:cov}
 For $s\neq 0$, under null hypothesis
 \begin{equation*}
 \begin{split}
&{ \rm Cov } \left[ {  \rm MMD }^2(X_i^{(B)}, Y^{(B)}), { \rm  MMD }^2 (X_{i+s}^{(B)}, Y^{(B)})    \right]\\
=
&{\binom{B}{2}}^{-1} {\rm Cov} \left[  h(x,x',y,y'),h(x'',x''', y, y')   \right].
\end{split}
\end{equation*}
 \end{lemma}
 \begin{proof}
For $i=1,2,\dots, N$, and $s=(1-i), (2-i), \dots, (N-i),   s \neq 0$,
\begin{equation*}
\begin{split}
   & {\rm Cov } \left[  { \rm MMD }^2(X_i, Y), { \rm  MMD}^2 (X_{i+s}, Y)    \right] \\
= &~ {\rm Cov}   \left[ {\binom{B}{2} }^{-1}  \sum_{l < j}  h(X_{i,l}, X_{i,j}, Y_l, Y_j),   {\binom{B}{2} }^{-1}  \sum_{p < q}  h(X_{i+s,p}, X_{i+s,q}, Y_p, Y_q  )           \right]  \\
=&  ~{\binom{B}{2} }^{-2}      \binom{B}{2} \binom{2}{1} \binom{B-2}{2-1} { \rm Cov} \left[  h(x,x',y,y'),h(x'',x''', y, y'')   \right]\\
&  +          {\binom{B}{2} }^{-2}                     \binom{B}{2} \binom{2}{2} \binom{B-2}{2-2} { \rm Cov} \left[  h(x,x',y,y'),h(x'',x''', y, y')   \right].
\end{split}
\end{equation*}
Under null distribution,
\begin{align*}
& { \rm Cov} \left[  h(x,x',y,y'),h(x'',x''', y, y'')   \right] \\
 =& \int h(x,x',y,y') h(x'',x''', y, y'') d \mathbb P (x,x',x'',x''',y,y',y'' )\\
 =& \int \left( \underbrace{ \int  h(x,x',y,y') d \mathbb P(x',y') }_{=0} \right) d \mathbb P(x) 
 \cdot \int  \left( \underbrace{ \int h(x'',x''', y, y'') d\mathbb{P} (x'', y'')}_{=0} \right) d \mathbb P(x''')
 = 0.
\end{align*}
Above, with a slight abuse of notation, we use $d \mathbb P(\cdot)$ to denote the probability measure of appropriate arguments. Finally, we have the desired results as shown in Lemma \ref{Lemma:cov}.


 \end{proof}
 
 \subsection{Variance of Scan $B$-Statistics.}
 \begin{proof}[Proof for Lemma \ref{thm:variance}] 
 Using results in Lemma \ref{Lemma:var_b_test} and Lemma \ref{Lemma:cov}, we have
 \begin{align*}
{\rm Var}[Z_B] &={\rm Var} \left[ \frac{1}{N} \sum_{i=1}^N {\rm MMD}^2(X_i, Y) \right] \\
  &=\frac{1}{N^2}\left[ N {\rm Var}[ { \rm MMD}^2(X_i, Y)] +\sum_{i \neq j} { \rm Cov} \left[  {\rm MMD}^2(X_i, Y; B), {\rm MMD}^2(X_j, Y)   \right]  \right] \\
 & = {\binom{B}{2}}^{-1}\left[   \frac{1}{N} \mathbb{E} [h^2(x, x',y,y')]  +\frac{N-1}{N} {\rm Cov} \left[  h(x,x',y,y'),h(x'',x''', y, y')   \right]  \right].
  \end{align*}
\end{proof}
Next, we introduce Lemma \ref{lemma:cov2} and Lemma \ref{lemma:cov_mmd}, which are useful in proving Lemma \ref{lemma:cov_stats}. 

\begin{lemma}[Covariance of MMD, different block sizes, same block index.]
 For blocks with the same index $i$ but with distinct block sizes, under the null hypothesis we have
 \begin{align}
 { \rm Cov} \left[ { \rm MMD}^2(X_i^{(B)}, Y^{(B)}), { \rm MMD}^2(X^{(B+v)}_i, Y^{(B+v)}) \right]={ \binom {B \vee (B+v)} {2} }^{-1} \mathbb{E} [h^2(x, x',y,y')].
 \end{align}
 \label{lemma:cov2}
 \end{lemma}
 \begin{proof}
 Note that
 \begin{align*}
   & { \rm Cov} \left[  { \rm MMD}^2(X_i^{(B)}, Y^{(B)}), {\rm MMD}^2(X_i^{(B+v)}, Y^{(B+v)}) \right]\\
   =~&  {\rm Cov}   \left[    {\binom{B}{2} }^{-1}  \sum_{l < j}^{B}  h(X_{i,l}, X_{i,j}, Y_l, Y_j  ),   {\binom{B+v}{2} }^{-1}  \sum_{p < q}^{B+v } h(X_{i,p}, X_{i,q}, Y_p, Y_q  )           \right]  \\
   =~&  {\binom{B}{2} }^{-1}   {\binom{B+v}{2} }^{-1}     {\rm Cov}   \left[  \sum_{l < j}^{B}  h(X_{i,l}, X_{i,j}, Y_l, Y_j  ),     \sum_{p < q}^{B+v}  h(X_{i,p}, X_{i,q}, Y_p, Y_q  )           \right]  \\
    =~&  {\binom{B}{2} }^{-1}   {\binom{B+v}{2} }^{-1}    \binom{B \wedge (B+v)}{2}   {\rm Var} [h(x,x',y,y')]  \\
   =~ & {\binom{B \vee (B+v)}{2} }^{-1}  \mathbb{E} [h^2(x,x',y,y')],
 \end{align*}
 where the second last equality is due to a similar argument as before to drop block indices as they are {\it i.i.d.}~under the null.
 \end{proof}

 \begin{lemma}[Covariance of MMD, different block sizes, different block indices.]\label{lemma:cov_mmd}
Under the null we have
\begin{align*}
{\rm Cov} \left[   {\rm MMD}^2(X_i^{(B)}, Y^{(B)}), {\rm MMD}^2 (X_{i+s}^{(B+v)}, Y^{(B+v)})    \right] 
= & {\binom{B \vee (B+v)}{2} }^{-1}  \cdot \\ \nonumber 
   & \quad \quad {\rm Cov} \left[  h(x,x',y,y'),h(x'',x''', y, y')   \right].
\end{align*}
 \end{lemma}
 \begin{proof}
 Note that
 \begin{align*}
   &  {\rm Cov} \left[   {\rm MMD}^2(X_i^{(B)}, Y^{(B)}), {\rm MMD}^2 (X_{i+s}^{(B+v)}, Y^{(B+v)})    \right]  \\
  =~&  {\rm Cov}   \left[   {\binom{B}{2} }^{-1}  \sum_{l < j}^{B}  h(X_{i,l}^{(B)}, X_{i,j}^{(B)}, Y_l^{(B)}, Y_j^{(B)}),   {\binom{B+v}{2} }^{-1}  \sum_{p < q}^{B+v}  h(X_{i+s,p}^{(B+v)}, X_{i+s,q}^{(B+v)}, Y_p^{(B+v)}, Y_q^{(B+v)})           \right] \\
   =~&  {\binom{B}{2} }^{-1}   {\binom{B+v}{2} }^{-1}     {\rm Cov}   \left[  \sum_{l < j}^{B}  h(X_{i,l}^{(B)}, X_{i,j}^{(B)}, Y_l^{(B)}, Y_j^{(B)}),     \sum_{p < q}^{B+v}  h(X_{i+s,p}^{(B+v)}, X_{i+s,q}^{(B+v)}, Y_p^{(B+v)}, Y_q^{(B+v)})           \right]  \\
      =~&  {\binom{B}{2} }^{-1}   {\binom{B+v}{2} }^{-1}   \binom{B \wedge (B+v)}{2}    {\rm Cov} \left[  h(x,x',y,y'),h(x'',x''', y, y')   \right]    \\
   =~&  {\binom{B \vee (B+v)}{2} }^{-1}  {\rm Cov} \left[  h(x,x',y,y'),h(x'',x''', y, y')   \right],
  \end{align*}
   where the second last equality is due to a similar argument as before to drop block indices as they are {\it i.i.d.}~under the null.

 \end{proof}

 \subsection{Covariance of Offline Scan $B$-Statistics.}\label{App:covoffline}
 \begin{proof}[Proof of Lemma \ref{lemma:cov_stats}]
For the offline case, we have that the correlation 
 \begin{align*}
r_{B, B+v}:=\frac{1}{\sqrt{ {\rm Var}[Z_B]}} \frac{1}{\sqrt{ {\rm Var}[Z_{B+v}]}} {\rm Cov} \left[  Z_B, Z_{B+v}  \right],
 \end{align*}
 where \begin{align*}
 {\rm Cov} \left(  Z_B, Z_{B+v}  \right) & = {\rm Cov} \left[  \frac{1}{N} \sum_{i=1}^N {\rm MMD}^2(X_i^{(B)}, Y^{(B)}), \frac{1}{N} \sum_{j=1}^n {\rm MMD}^2( X_j^{(B+v)}, Y^{(B+v)})   \right] \\
&=\frac{1}{N}  {\rm Cov} \left[  \mbox{MMD}^2(X_i^{(B)}, Y^{(B)}), \mbox{MMD}^2(X_i^{(B+v)}, Y^{(B+v)})      \right]   \\
& \,\,\,\,  +  \frac{1}{N^2}\sum_{i \neq j}  {\rm Cov} \left[   \mbox{MMD}^2(X_i^{(B)}, Y^{(B)}), \mbox{MMD}^2(X_j^{(B+v)}, Y^{(B+v)})    \right]   .
\end{align*}
Using results from  Lemma \ref{lemma:cov2} and Lemma \ref{lemma:cov_mmd}, we have:
\begin{align*}
 {\rm Cov} \left( Z_B, Z_{B+v}  \right) 
&={\binom{B \vee (B+v)}{2} }^{-1} \left[ \frac{1}{N} \mathbb{E} [h^2(x,x',y,y')]\right.\\
 &\quad\left.+\frac{N-1}{N}   {\rm Cov} \left[  h(x,x',y,y'),h(x'',x''', y, y')   \right] \right].
\end{align*}
Finally, plugging in the expressions for ${\rm Var} [Z_B]$ and ${\rm Var}[Z_{B+v}]$, we have (\ref{r_def}) for the offline case.

\subsection{Covariance of Online Scan $B$-Statistics} \label{App:covonline}
Similarly, for the online case we need to analyze
$
\rho_{t, t+s} :=\mbox{Cov} \left(Z_{B_0,t}', Z_{B_0,t+s}' \right).   
$
We adopt the same strategy as the above for a fixed block size $B_0$ to obtain
\begin{align}
& \mbox{Cov} \left( \mbox{MMD}^2(X_i^{(B_0, t)}, Y^{(B_0, t)}), \mbox{MMD}^2 (X_i^{(B_0, t+s)}, Y^{(B_0, t+s)}) \right)  \nonumber\\
 =~&  {\rm Cov}   \left[    {\binom{B_0}{2} }^{-1}  \sum_{l < j}^{B_0}  h(X^{(t)}_{i,l}, X^{(t)}_{i,j}, Y^{(t)}_l, Y^{(t)}_j  ),   {\binom{B_0}{2} }^{-1}  \sum_{p < q}^{B_0 } h(X^{(t+s)}_{i,p}, X^{(t+s)}_{i,q}, Y^{(t+s)}_p, Y^{(t+s)}_q  )           \right]  \nonumber \\
   =~& {\binom{B_0}{2} }^{-2}   \binom{(B_0-s) \vee 0}{2}  \mbox{\rm Var} [h(x,x',y,y')].   \label{covonline1}  
\end{align}
Figure \ref{demoonline1} (a) demonstrates how $\mbox{MMD}^2(X_i^{(B_0, t)}, Y^{(B_0, t)})$ and $\mbox{MMD}^2 (X_i^{(B_0, t+s)}, Y^{(B_0, t+s)})$ are constructed. The shaded areas represent the overlapping data.

\begin{figure}[H]
\vspace{-3mm}
        \begin{center}
        \begin{tabular}{cc}
                \includegraphics[width=.38\textwidth]{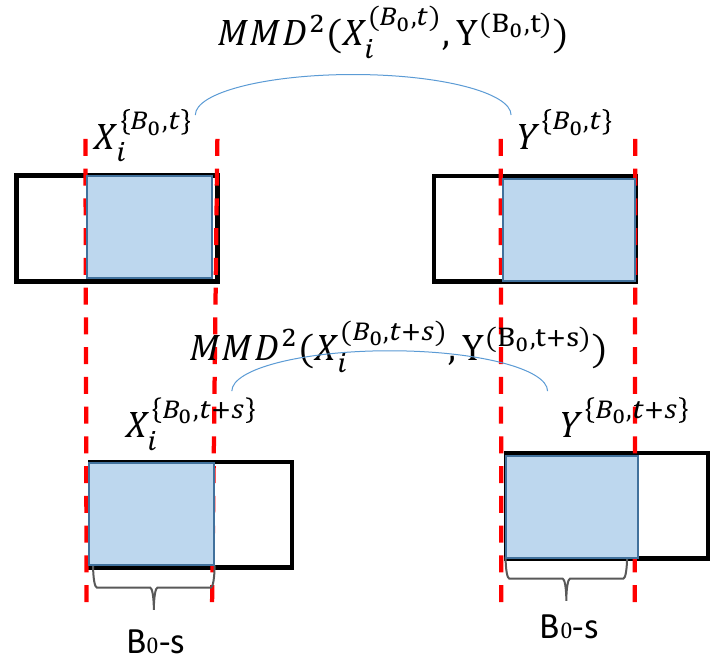} &
                \includegraphics[width=.55\textwidth]{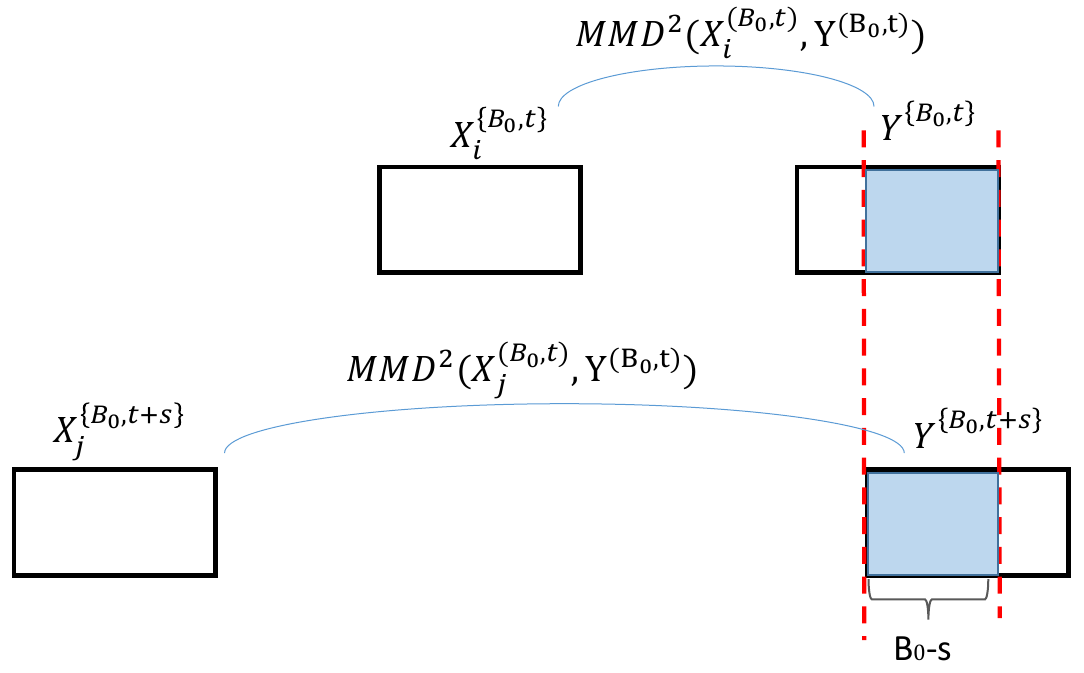}\\
(a) & (b)
\end{tabular}
        \end{center}
        \caption{\small (a): Illustration of how $\mbox{MMD}^2(X_i^{(B_0, t)}, Y^{(B_0, t)})$ and $\mbox{MMD}^2 (X_i^{(B_0, t+s)}, Y^{(B_0, t+s)})$ are constructed in the online change-point detection, where the shaded areas represent the overlapping data. (b): Illustration of how $\mbox{MMD}^2(X_i^{(B_0, t)}, Y^{(B_0, t)})$ and $\mbox{MMD}^2 (X_j^{(B_0, t+s)}, Y^{(B_0, t+s)})$, $j \neq i$ are constructed in the online change-point detection, where the shaded areas represent the overlapping data.}
        \label{demoonline1}                                
\end{figure}
Similarly, we have
\begin{align}
& \mbox{Cov} \left( \mbox{MMD}^2(X_i^{(B_0, t)}, Y^{(B_0, t)}), \mbox{MMD}^2 (X_j^{(B_0, t+s)}, Y^{(B_0, t+s)}) \right) \nonumber  \\
=~&  {\rm Cov}   \left[    {\binom{B_0}{2} }^{-1}  \sum_{l < k}^{B_0}  h(X^{(t)}_{i,l}, X^{(t)}_{i,k}, Y^{(t)}_l, Y^{(t)}_k  ),   {\binom{B_0}{2} }^{-1}  \sum_{p < q}^{B_0 } h(X^{(t+s)}_{j,p}, X^{(t+s)}_{j,q}, Y^{(t+s)}_p, Y^{(t+s)}_q  )           \right]  \nonumber \\
=~&  {\binom{B_0}{2} }^{-2} \binom{(B_0-s) \vee 0}{2}    \mbox{Cov} (h(x,x',y,y'), h(x'',x''',y,y')), \label{covonline2}
\end{align}
Figure \ref{demoonline1} (b) demonstrates how $\mbox{MMD}^2(X_i^{(B_0, t)}, Y^{(B_0, t)})$ and $\mbox{MMD}^2 (X_j^{(B_0, t+s)}, Y^{(B_0, t+s)})$, $j \neq i$ are constructed. The shaded areas represent the overlapping data.
Thus, 
\begin{equation*}
\begin{split}
&\mbox{Cov} \left( Z_{B_0, t}, Z_{B_0, k+s}  \right) \\
= ~& \mbox{Cov} \left(   \frac{1}{N} \sum_{i=1}^N \mbox{MMD}^2 (X_i^{(B_0, t)}, Y^{(B_0, t)}) , \frac{1}{N} \sum_{j=1}^N \mbox{MMD}^2 (X_j^{(B_0, t+s)}, Y^{(B_0, t+s)})   \right) \\
= ~&   {\binom{B_0}{2} }^{-2} \binom{(B_0-s) \vee 0}{2}   \Big[  \frac{1}{N} \mbox{\rm Var}(h(x,x',y,y'))\\
&~~+\frac{N-1}{N} \mbox{Cov} (h(x,x',y,y'),h(x'',x''',y,y'))  \Big] .
\end{split}
\end{equation*}
Finally, plugging in the expressions for ${\rm Var} [Z_{B_0, t}]$ and ${\rm Var}[Z_{B_0, t+s}]$, we have (\ref{eq:covonline}) for the online case.
\end{proof}

\section{PROOF OF THEOREM \ref{thm:tail_fixedsample}}
\label{proof_main_theorem}

Below, we present the main steps in proving Theorem \ref{thm:tail_fixedsample}, including (1) exponential tilting; (2) change-of-measure by the likelihood identity; (3) establish properties of the local field and the global term; and (4) perform asymptotic approximation using the localization theorem (Theorem 5.1 in \cite{Siegmund10} and Sec. 3.4 in \cite{YakirBook2013}) by showing that the ``global'' log likelihood and the ``local process'' are asymptotically independent. Finally, we collect terms together to obtain the result. 

\subsection{Step One: Exponential Tilting}

We first introduce exponential tilting, which creates a family of distributions that is related to the original distribution of $Z_B'$. Let the log moment generating function of $Z_B'$ be 
\begin{equation}
\psi(\theta) = \log \mathbb{E}[e^{\theta Z_B'}]. \label{phi}
\end{equation} Define a family of new measures
\begin{equation}
d \mathbb{P}_B= \exp\left\{ \theta Z_B' - \psi(\theta) \right\} d \mathbb{P}, \label{new_meas}
\end{equation}
where $\mathbb{P}$ represents the original probability measure of $Z_B'$ under the null distribution $P$, $\mathbb{P}_B$ is the new measure after the transformation, and $\theta$ parameterizes the family of the new measures.  Note that the new measures take the form of exponential family, with $\theta$ being the parameter. 

Recall that, under the null distribution, $Z_B'$ has zero mean and unit variance. Given the assumption that $Z_B'$ is a standard Gaussian random variable, the corresponding log moment generating function is given by $\psi(\theta) = \theta^2/2$. One has the freedom to select the value of $\theta$ to determine the new measure. We will set $\theta$  such that the mean under the tilted measure is equal to a given threshold $b$. This means that the new measure peaks at the threshold $b$, which enables us to use the local central limit theorem later on. This can be done by choosing $\theta$ such that $\dot{\psi}(\theta) = b$, and therefore $\theta = b$. Note that the solution $\theta$ does not depend on $B$. Hence, we can set the mean under the transformed measure to $b$, by uniformly choosing $\theta = b$ for any $B$. Given such a choice, the transformed measure is given by
$
d\mathbb{P}_B =  \exp\left\{ b Z_B'(x) - b^2/2 \right\} d\mathbb{P}. 
$
We also define, for each $B$, the log-likelihood ratio $\log (d\mathbb{P}_B/d \mathbb{P})$ of the form 
 \begin{equation}\label{likelihood_ratio}
 \ell_B =  b Z_B' - b^2/2. 
 \end{equation}
This way, we have associated the detection statistic $Z_B'$ with a likelihood ratio, even if $Z_B'$ itself does not come out of a likelihood ratio.

The following lemma shows that $Z_B'$ under the new measure has the same unit variance and its mean has been shifted to $b$. This key fact will lead to the desired exponential tail. 
 \begin{lemma}[Mean and variance under tilted measure]\label{lemma_mean_var_transformed}
 Define $\mathbb{E}_B$ and $\mbox{Var}_B$ as the expectation and variance under the transformed measures
 \begin{align}
 \mathbb{E}_B[U] & = \mathbb{E}[Ue^{\ell_B}], \label{def:exp_change} \\
  \mbox{\rm Var}_B[U] & = \mathbb{E}[U^2e^{\ell_B}] -  \mathbb{E}_B^2[U]. 
 \end{align}
 We have
$
 \mathbb{E}_B[Z_B'] = b$, and ${\rm Var}_B [Z_B'] = 1.$
 \end{lemma}
 \begin{proof}
First, $ \mathbb{E}_B[Z_B'] = \dot{\psi}(b)= b$ by construction. To show ${\rm Var}_B [Z_B']=1$, note that $\log \mathbb{E}[e^{bZ_B'}] = b^2/2$. Taking the derivative  of $\psi(\theta)$ with respect to $b$ twice gives $\mathbb{E}[(Z_B')^2 e^{bZ_B'}] = e^{b^2/2} + b^2 e^{b^2/2}$. Hence, $\mathbb{E}_B[(Z_B')^2] = \mathbb{E}[(Z_B')^2 e^{\theta Z_B'-\psi(b)}] = 1+ b^2$, and ${\rm Var}_B [Z_B'] = \mathbb{E}_B[(Z_B')^2] - b^2 = 1.$
\end{proof}

The following lemma shows that $Z_B'$ under the new measure has the same unit variance with the mean  shifted to $b$. This key fact will lead to the desired exponential tail. 
 \begin{lemma}[Mean and variance under tilted measure]\label{lemma_mean_var_transformed}
 Define $\mathbb{E}_B$ and $\mbox{Var}_B$ as the expectation and variance under the transformed measures
 \begin{align}
 \mathbb{E}_B[U] & = \mathbb{E}[Ue^{\ell_B}], \label{def:exp_change} \\
  \mbox{\rm Var}_B[U] & = \mathbb{E}[U^2e^{\ell_B}] -  \mathbb{E}_B^2[U]. 
 \end{align}
 We have
$
 \mathbb{E}_B[Z_B'] = b$, and ${\rm Var}_B [Z_B'] = 1.$
 \end{lemma}
 \begin{proof}
First, $ \mathbb{E}_B[Z_B'] = \dot{\psi}(b)= b$ by construction. To show ${\rm Var}_B [Z_B']=1$, note that $\log \mathbb{E}[e^{bZ_B'}] = b^2/2$. Taking the derivative  of $\psi(\theta)$ with respect to $b$ twice gives $\mathbb{E}[(Z_B')^2 e^{bZ_B'}] = e^{b^2/2} + b^2 e^{b^2/2}$. Hence, $\mathbb{E}_B[(Z_B')^2] = \mathbb{E}[(Z_B')^2 e^{\theta Z_B'-\psi(b)}] = 1+ b^2$, and ${\rm Var}_B [Z_B'] = \mathbb{E}_B[(Z_B')^2] - b^2 = 1.$
\end{proof}

\subsection{Step Two: Change-of-Measure}

Now we are ready to analyze the tail probability $
\mathbb{P} \left \{  \max_{2 \leq B \leq B_{\rm max}} Z_B'   > b    \right\}
$. The basic idea is to convert the original problem of finding the small probability that the maximum of a random field exceeds a large threshold to another problem: finding an alternative measure under which the event happens with a much higher probability. 

Here, the alternative measure will be a mixture of simple exponential tilted measures. Define the maximum and the sum for likelihood ratio differences relative to a particular parameter value $B$: 
\begin{equation}\label{def:MS}
M_B = \max_{s \in \{2,\dots, B_{\rm max}\}} e^{\ell_s - \ell_B}, \qquad
S_B =\sum_{s \in \{2,\dots, B_{\rm max}\}} e^{\ell_s - \ell_B}.
\end{equation}
Also define a re-centered likelihood ratio, which we call the {\it global term}
\[\tilde{\ell}_B =  b(Z_B' - b).\]
With the definitions above and the log likelihood ratios $\ell_B$ in (\ref{likelihood_ratio}), we have the following
\begin{equation}
\begin{split}
&  \mathbb{P}\left\{ \max_{2\leq B \leq B_{\rm max}}  Z_B' >b       \right\}  
= \mathbb{E} \left[ 1;\max_{2\leq B \leq B_{\rm max}}  Z_B' >b   \right]
= \mathbb{E} \Bigg[ \underbrace{\frac{\sum_{B=2}^{B_{\rm max}}  e^{\ell_B}   }{\sum_{s=2}^{B_{\rm max} } e^{\ell_s}}}_{=1}   ; \max_{2\leq u \leq B_{\rm max}} Z_u'  >b    \Bigg]\\
=& \sum_{B=2} ^{B_{\rm max}}  \mathbb{E}  \left[     \frac{e^{\ell_B}   }{\sum_s   e^{\ell_s}}   ; \max_{2\leq u \leq B_{\rm max}}Z_u' >b       \right]
\stackrel{(\ref{def:exp_change})}{=} \sum_{B=2}^{B_{\rm max}} \mathbb{E}_B \left[    \frac{1  }{\sum_s   e^{\ell_s}}   ; \max_{2\leq u \leq B_{\rm max}} Z_u' >b       \right] \\
= &~ e^{-b^2/2} \sum_{B=2}^{B_{\rm max}}     
     \mathbb{E}_B \left[ \frac{M_B}{S_B}    e^{- (\tilde{\ell}_B + \log M_B )}; \tilde{\ell}_B + \log M_B \geq 0       \right]   
\end{split}
\label{measure_transform}
\end{equation}
where an intermediate step is done by changing the measure to $\mathbb P_B$, and the last equality can be verified by simple algebra. Recall our notation $\mathbb E_B[\mathcal A; \mathcal B] = \mathbb E_B [\mathcal A \textbf{1}\{\mathcal B\}]$ for a random quantity $\mathcal A$ and event $\mathcal B$; $\textbf{1}$ denotes an indicator function. 

In a nutshell, the last equation in (\ref{measure_transform}) converts the tail probability to a product of two terms: a deterministic term $e^{-b^2/2}$ associated with the large deviation rate, and a sum of conditional expectations under the transformed measures. A close examination of the conditional expectations of the form $\mathbb{E}_B[\cdots; [\cdots] \geq 0]$ reveals that it involves a product of the ratio $M_B/S_B$, and an exponential function that depends on $\tilde{\ell}_B$, which plays the role of weight. Under the new measure $\mathbb P_B$, $\tilde{\ell}_B$ has zero mean and variance equal to $b^2$ (shown below in Lemma \ref{lm: global_field}) and it dominates the other term $\log M_B$ and, hence, the probability of exceeding zero will happen with much higher probability. Next, we characterize the limiting ratio and the other factors precisely, by the localization theorem.

\subsection{Step Three: Establish Properties of Local and Global Terms}\label{sec:property}

In (\ref{measure_transform}), our target probability has been decomposed into terms that only depend on (i) the {\it local field} $\{\ell_s -\ell_B\}$, $2\leq s \leq B_{\rm max}$, which are the differences between the log-likelihood ratio with parameter $B$ and with other parameter values $s$, $2\leq s \leq B_{\rm max}$, and (ii) the \textit{global term} $\tilde{\ell}_B$,  which is the centered and scaled likelihood ratio with parameter $B$. We need to first establish some useful properties of the local field and the global term under the tilted measure. We will eventually show that the local field and the global term are asymptotically independent.

The following property for the global term can be derived from Lemma \ref{lemma_mean_var_transformed}. The result shows that under the tilted measure, the global term $\tilde{\ell}_B$ has zero mean for any $B$, with variance diverging with $b$.  
\begin{lemma}[Global term for offline scan $B$-statistic]\label{lm: global_field}
The mean and variance of the global term $\tilde{\ell}_B = b(Z_B'-b)$, for $2\leq B \leq B_{\rm max}$, are given by
\begin{align} \label{eq:globalterm}
\mathbb{E}_B[ \tilde{\ell}_B] = 0, \quad {\rm Var}_B [\tilde{\ell}_B] = b^2.
\end{align}
\end{lemma}
Assuming $Z_B'$ is approximately normal, the local field $\ell_{s}-\ell_B$ (or equivalently $b(Z_{s}'-Z_B')$) and the global term $\tilde{\ell}_B$ (or equivalently $b(Z_B'-b)$) are also approximately normally distributed. 
\begin{lemma}[Local field for offline scan $B$-statistic]\label{lm: local_field}
The mean and variance of the local field $\{\ell_{s}-\ell_B\}$, for $|s-B| = 0, 1, 2, \ldots$, are given by
\[
\mathbb{E}_B [\ell_{s} - \ell_B] =  -b^2(1-r_{ s, B}), \quad  {\rm Var}_B [\ell_{s} - \ell_B] = 2b^2(1-r_{s,B}),\]
with $r_{s,B} $ defined in (\ref{r_def}). For any $s_1$ and $s_2$, the covariance between two local field terms is given by
\[{\rm Cov}_B\left(\ell_{s_1} - \ell_B , \ell_{s_2} - \ell_B  \right) =  b^2\left( 1+ r_{s_1, s_2} - r_{s_1, B} -r_{s_2, B}\right). 
\]
\end{lemma}

\begin{proof}
Note that $\ell_s - \ell_B = b(Z_s'-Z_B')$, $\mathbb{E}_B[Z_B']=b$, ${\rm Var}_B [Z_B']=1$. Moreover, due to the normal assumption of $Z_B'$, we have the following decomposition $\mathbb E_B [\ell_{s} - \ell_B] = \mathbb E_B[b(Z_s' - Z_B')] = \mathbb E_B[b(r_{s, B}Z_B' + (1-r_{s, B}^2)^{1/2}W - Z_B')] = -b^2(1-r_{s, B})$, where $W$ is a zero-mean random variable and independent of $Z_B'$, representing residual of regression. The variance and covariance can be found using similar decompositions. 
\end{proof}

\begin{remark}[Consequence of Lemma \ref{lm: local_field}]
\label{remark}
From the expression of the covariance in (\ref{r_def}), we have that for $s-B>0$, 
\[
r_{s, B}=  \left[1+(s-B)/B  \right]^{-1/2}  \left[1+ (s-B)/(B-1) \right)]^{-1/2},    
\]
and for  $s-B<0$,
\[
r_{s, B}=\left[1+(s-B)/B\right]^{1/2} \left[1+(s-B)/(B-1)\right]^{1/2}. 
\]
Consequently,
\begin{enumerate}
\item When $|s-B|\to \infty$, $r_{s, B}\to 0$. Therefore, when $|s-B| \to \infty$, $\mathbb{E}_B [\ell_{s} - \ell_B] $ converges to $-b^2$ and ${\rm Var}_B [\ell_{s} - \ell_B]$ converges $2b^2$.

\item When $|s-B|$ is small, assume $s = B+j$, $j = 0, \pm 1, \pm 2, \ldots$.
Perform the Taylor expansion of $r_{B+j, B}$ around 0, we have that
\begin{equation} \label{eq:appror}
r_{B+j,B} = 1-\frac{1}{2} \frac{2B-1}{B(B-1)} |j| +o(|j|).
\end{equation}

Define 
\begin{equation}
\mu = b\{(2B-1)/[B(B-1)]\}^{1/2}.\label{mudef}
\end{equation} Note that $\mu$ depends on the threshold as well as $B$, the block size parameter. Using (\ref{eq:appror}), we have
\begin{align*}
\lim_{|j|\to 0}\mathbb{E}_B [\ell_{B+j} - \ell_B] &=  -\frac{\mu^2}{2}|j|, \nonumber \\
  \lim_{|j|\to 0 }{\rm Var}_B [\ell_{B+j} - \ell_B] & = \mu^2|j| ,  \nonumber \\
\lim_{|j_1|\to 0, |j_2|\to 0 }{\rm Cov}_B\left(\ell_{B+j_1} - \ell_B , \ell_{B+j_2} - \ell_B  \right) & = 
\mu^2 (|j_1| \wedge |j_2|).
 \end{align*}
Therefore, when $|j|$ is small (i.e., in the neighborhood of zero), we can approximate the local field using a two-sided Gaussian random walk with drift $\mu^2/2$ and the variance of the increment being $\mu^2$:
\begin{equation}
\ell_{B+j}-\ell_B \stackrel{d}{=} \mu \sum_{i=1}^{|j|} \vartheta_{i} - \mu^2 j/2,\quad j = \pm 1, \pm2, \ldots
\end{equation}
where $\vartheta_{i}$ are {\it i.i.d.} ~standard normal random variables. 
\end{enumerate}
\end{remark}

\subsection{Step Four: Approximation Using Localization Theorem}

The remaining work is to compute the conditional expectations $\mathbb{E}_B[\cdots; (\cdots) \geq 0]$ for each $B$ in (\ref{measure_transform}). 
In the following, we drop the subscript $B$ in $\mathbb E_B$ for simplicity, and the approximation results hold for each $B$. 
We assume $b \to \infty$, $B_{\rm max}\to \infty$, and $b^2/B_{\rm max}$ is held to a fixed positive constant. Introduce an abstract index $\kappa$ and let $\kappa= b^2$; this choice is because $\kappa^{1/2}$ is the multiplicative factor that balances the rate of convergence of the global term under the transformed measure. Typically, $\kappa$ is equal to the variance of the global term $\tilde{\ell}_B = b(Z_B'-b)$, which is $b^2$ as shown in Lemma \ref{lm: global_field}; $\kappa$ is also associated with the drift and the variance of the incremental of the local field $\{\ell_{s}-\ell_B\}$ for $|s-B| = 0, 1, 2, \ldots$, as shown in Lemma \ref{lm: local_field}.

Using a powerful localization theorem (see Theorem 3.1 in \citep{Siegmund10} or Theorem 5.2 in \citep{YakirBook2013}), we can obtain the limit for each term in the summand of (\ref{measure_transform}), rewritten as (by changing the index to $\kappa$)
\begin{equation}\label{def:conditionalexp} 
\mathbb{E}\left[ \frac{M_{\kappa}}{S_{\kappa}} e^{- (   \tilde{\ell}_{\kappa} + \log M_{\kappa} ) } ; \tilde{\ell}_{\kappa} + \log M_{\kappa} \geq 0       \right],
\end{equation}
when $\kappa \to \infty$. Basically, the localization theorem states that (\ref{def:conditionalexp}) scaled by $\kappa^{\frac{1}{2}}$ converges under mild conditions when $\kappa \to \infty$.

The statement of the theorem involves a local $\sigma$-algebra denoted as $\widehat{\mathcal{F}}_{\kappa}$:\begin{equation}
\widehat{\mathcal{F}}_{\kappa}= \sigma\{ \ell_s-\ell_B : |s-B | \leq g(\kappa) \}, \label{Fk}
\end{equation}
where a function $g(\kappa)$ specifies the size of the local region. The choice of $g(\kappa)$ is critical and it guarantees subsequent convergence. Following the analysis of scan statistics in \citep{YakirBook2013}, we choose $g(\kappa) = c b^{-2}$ for some large constant $c$. This local $\sigma$-field is asymptotically independent of $\tilde{\ell}_{\kappa}$, and it carries all information needed to construct the local field.

Define $\widehat{M}_{\kappa}$ and $\widehat{S}_{\kappa}$ as the maximization and summation restricted to a smaller subset of parameter values $\{s: |s-B | \leq g(\kappa) \}$, and they are measurable with respect to $\widehat{\mathcal{F}}_{\kappa}$. Note that $\widehat{M}_{\kappa}$ and $\widehat{S}_{\kappa}$ serve as approximations to $M_{\kappa}$ and $S_{\kappa}$. In the limit, the local random field converges to a Gaussian random field, and the ratio $\mathbb E[\widehat{M}_\kappa/\widehat{S}_\kappa]$ converges to a limit that can be determined with the parameters of the Gaussian random field. 
 
The localization theorem (Theorem 5.1 in \cite{Siegmund10} and Sec. 3.4 in \cite{YakirBook2013}) consists of the five conditions as follows. 
\begin{theorem}[Localization theorem]\label{the:aymp}
Given $\epsilon > 0$, if for all large $\kappa$, all following conditions hold
\begin{enumerate}
\item[I.] Both $0 < M_{\kappa} \leq S_{\kappa} < \infty$ and $0 < \widehat{M}_{\kappa} \leq \widehat{S}_{\kappa} < \infty$ hold in probability one.
\item[II.] 
Denote $A^c = \{| \log M_{\kappa} - \log \widehat{M}_{\kappa}| > \epsilon \} \cup \{|\widehat{S}_{\kappa}/S_{\kappa} -1 | > \epsilon\} $. For some $0 < \delta$ that does not depend on $\epsilon$:
\begin{align*}\label{eq:newcondition}
\max_{|x| \leq 3g(\kappa)} \mathbb{P} \left[ A^c \cap \{ \tilde{\ell}_{\kappa} + \log \widehat{M}_{\kappa} \in x + (0, \delta]\} \cap \{ |\hat{m}| \leq g(\kappa)\} \right] \leq \epsilon \kappa^{-1/2},
\end{align*} 
where $\hat{m}_{\kappa}=\min\{ \log \widehat{M}_{\kappa}, g(\kappa)\}-\log (1-\epsilon)$.
\item[III.] $\mathbb{E}[ \widehat{M}_{\kappa} / \widehat{S}_{\kappa}]$ converges to a finite and positive limit denoted by $\mathbb{E}[M/S]$.
\item[IV.] There exist $\mu_{\kappa} \in \mathbb{R}$ and $\sigma_{\kappa} \in \mathbb{R}^+$ such that for every $0 < \epsilon'$, $\delta$, for any event $E \in \widehat{\mathcal F}_\kappa$ and for all large enough $\kappa$
\[
\sup_{|x|\leq \epsilon \kappa^{1/2}} \left|\kappa^{1/2}\mathbb{P}(\tilde{\ell}_\kappa \in x + (0, \delta], E) - \frac{\delta}{\sigma}\phi\left(\frac{\mu}{\sigma}\right)\mathbb{P}(E)\right|\leq \epsilon' .
\]
\item[V.] $\mathbb{P}(|\log M_\kappa|>\epsilon \kappa^{1/2})$, $\mathbb{P}(|\log \widehat{M}_\kappa| > \epsilon \kappa^{1/2})$ and $\mathbb{P}(\log M_\kappa - \log \widehat{M}_\kappa < -\epsilon)$ are all $o(\kappa^{-1/2})$. 
\end{enumerate}
Then
\begin{align}
\lim_{\kappa \to \infty} \kappa^{1/2} \mathbb{E}\left[ \frac{M_{\kappa}}{S_{\kappa}}   e^{- \left[\tilde{\ell}_{\kappa} + \log M_{\kappa} \right]} ; \tilde{\ell}_{\kappa} + \log M_{\kappa} \geq 0       \right]  = \sigma^{-1} \phi\left(\frac{\mu}{\sigma}\right) \mathbb{E}[M/S],
\end{align}
where $\phi(\cdot)$ is the density of the standard normal distribution.
\end{theorem}

Intuitively, the localization theorem says the following. To find the desired limit of (\ref{def:conditionalexp}) as $\kappa \to \infty$, one first approximates $M_{\kappa}$ and $S_{\kappa}$ by their localized versions, which are obtained by restricting the maximization and summation in a neighborhood of parameter values. Then one can show that the localized ratio $M_{\kappa}/S_{\kappa}$ is asymptotically independent of the global term $\tilde{\ell}_{\kappa}$ as $\kappa \to \infty$. The asymptotic analysis is then performed on the local field and the global term separately. The expected value of the localized ratio $\mathbb{E}[M_{\kappa}/S_{\kappa}]$ converges to a constant independent of $\kappa$, and the limiting conditional distribution of $\tilde{\ell}_{\kappa}$ can be found using the local central limit theorem. Thus, one can calculate the remaining conditional expectation involving $\tilde{\ell}_{\kappa}$. 

\vspace{0.1in}
\noindent\textbf{Checking conditions.}
Let us now verify the validity of the conditions in our setting. First, {\it Condition I} is met since for Gaussian random variables, $M_\kappa >0$, $S_\kappa >0$ with probability 1, and the maximization of a collection of non-negative numbers is smaller or equal to the summation. Similar arguments hold for their counterparts $\widehat{M}_\kappa >0$ and $\widehat{S}_\kappa >0$ when the maximization and summation are over a smaller set.

{\it Condition II} describes that the localized versions $\widehat{M}_{\kappa}$ and $\widehat{S}_{\kappa}$ are good approximations of $M_{\kappa}$ and $S_{\kappa}$ when $\kappa$ is sufficiently large, for properly defined $\widehat{\mathcal F}_\kappa$. 
In Section 3.4.4 of \cite{YakirBook2013}, the corresponding Condition II has been rigorously checked, assuming a local region defined in the same form of our local region and assuming Gaussian random field. Thus, checking Condition II for our case will follow the same steps, using the properties established in Section \ref{sec:property}. We omit the details here.

{\it Condition III} is checked by applying the distributional approximations to the localized version of $M_{\kappa}/S_{\kappa}$. We can show that the expectation of the ratio $\mathbb{E}[ \widehat{M}_{\kappa} / \widehat{S}_{\kappa}]$ converges to a finite and positive limit denoted by $\mathbb{E}[M/S]$,  which does not depend on $\kappa$. Since the increment $\ell_{B+j}-\ell_B$ has negative mean as shown in Lemma \ref{lm: local_field}, the values of $M_{\kappa}$ and $S_{\kappa}$ will be determined by values $j$ close to 0, so is the ratio $M_{\kappa}/S_{\kappa}$. This implies, a relatively small local region centered on $B$ is sufficient. 

From Remark \ref{remark}, the local field when the index is close to the shifted measure parameter $B$ can be approximated as a two-sided Gaussian random walk with drift $-\mu^2/2$ and variance $\mu^2$ (with $\mu$ defined in (\ref{mudef})), which is denoted as
 $W(\mu^2 j) $ below. Therefore, we have that with high probability, 
\begin{align*}
\mathbb{E}[ \widehat{M}_{\kappa} / \widehat{S}_{\kappa}] = \mathbb{E}\left[ \frac{\max_{|j| \leq cb^{-2}} e^{W(\mu^2 j)}}{\sum_{|j| \leq c b^{-2}} e^{W(\mu^2 j)}} \right] .
\end{align*}
When $c \to \infty$, it approaches to a limit known as the {\it Mill's ratio}
\begin{align*}
\mathbb{E}[ M / S] = \mathbb{E}\left[ \frac{\max_{|j| } e^{W(\mu^2 j)}}{\sum_{|j| } e^{W(\mu^2 j)}} \right], 
\end{align*}
with maximization and summation extending to the entire collection of negative and positive integers. The Mill's ratio is related to the Laplace transform of the overshoot of the maxima of Gaussian random field over a threshold $b$, and an expression has been obtained based on nonlinear renewal theory (see, \citep{SiegmundBook1985} and Chapter 2.2 of the book \citep{YakirBook2013}): $\mathbb{E}[ M / S]=\exp (  -2 \sum_{j=1}^{\infty} \Phi (-j^{1/2} \mu/2)).$ An easier numerical evaluation is given by
$
 \mathbb{E}[ M / S] \approx (\mu^2/2) \nu(\mu)
$
for a special function $\nu(\mu)$ defined in (\ref{eq:fucnu}).

{\it Condition IV} can be checked via a local multivariate central limit theorem that is local in one component and non-local in others (Theorem 5.3 in \cite{YakirBook2013}). The theorem says the following: assuming $\xi_i$ are independent, identically distributed random vector of dimension $d+1$. Assume the mean of each vector is zero, and variance of the first component converges to a finite $\sigma$, the covariance matrix of the last $d$ components converges a finite matrix $\Sigma$, and the correlation between these components and the first one converges to zero (hence, the overall covariance matrix is block-diagonal). Define $S_\gamma = \sum_{i=1}^\gamma \xi_{i, 1}$ and a $d$ dimensional vector with element $h_{\gamma,j} = \gamma^{-1/2} \xi_{i, j}$, for $1\leq j \leq d$. Then under mild conditions, 
\begin{equation}
\lim_{\gamma\rightarrow \infty} \gamma^{1/2} \mathbb{P}(S_\gamma \in [l, u], h_\gamma \in \mathcal A) = \frac{l-u}{(2\pi)^{1/2}\sigma} \mathbb P(h \in \mathcal A) 
\label{local_CLT}
\end{equation}
for any interval $[l, u]$ and an arbitrary set $\mathcal A$. 

Our setting matches exactly to the above distribution when we set the global term as the first component and the local field as the remaining components. Using the properties in Section \ref{sec:property}, we have shown the finite mean and variance (covariance) of the global and local field terms. We only need to show the global term, and the local fields are independent of each other asymptotically.  It suffices to prove that the conditional covariance of  $\{\ell_{B+j}-\ell_B\}$ given $\tilde{\ell}_B$ converges to the unconditional covariance, and the conditional means converges to the unconditional one. 
With a slight abuse of notation, $r_1 = r_{B+j_1, B}$ and $r_2 = r_{B+j_2, B}$ and using the linear regression decomposition, when conditioning on $Z_B'$ (which is proportional to $\tilde{\ell}_{B}$), the two local field terms are independent of each other:
\begin{align*}
&{\rm Cov}(b(Z_{B+j_1}'-Z_B'), b(Z_{B+j_2}'-Z_B') | Z_B')\\
&= {\rm Cov}(b( r_1 Z_B'+ (1-r_1^2)^{1/2} W_1-Z_B'), b(r_2 Z_B'+ (1-r_2^2)^{1/2} W_2-Z_B') | Z_B') = 0.
\end{align*}
where $W_1$ and $W_2$ are two mutually independent zero-mean random variables that represent the regression residuals (they are also independent of $Z_B'$).

On the other hand, using the same decomposition, we can show that without conditioning, the covariance is given by
\[
{\rm Cov}(b(Z_{B+j_1}'-Z_B'), b(Z_{B+j_2}'-Z_B'))
= b^2(1-r_1)(1-r_2).
\]
Hence, when $b \to \infty$, due to the property of local field in equation (\ref{eq:appror}), for $|j_1|\leq cb^{-2}$, $|j_2|\leq cb^{-2}$, the unconditioned covariance converges to zero given (\ref{eq:appror}), which is equal to the conditioned covariance. Similarly, we can show that the conditional means of $\{Z_{B+j}' - Z_B'\}$ conditioning on $Z_B'$ converges to the unconditional ones.  

Now we invoke the local central limit theorem. Since the density of the global term $\tilde{\ell}_B$ is approximately normal, we can calculate a desired form of the probability. From (\ref{eq:globalterm}), the variance of the global term increases with $b$. The density of $\tilde{\ell}_B$ can be uniformly approximated by $1/(2\pi b^2)^{1/2}$ within a small region around the origin $|x|\leq 3(4/+1+\epsilon)\log b$ \citep{YakirBook2013}. Such an approximation also holds for $ \tilde{\ell}_B-x$ given any value $x$ that is not too large. Furthermore, notice that $\log \hat{M}_{\kappa}$ is very close to 0 and therefore is negligible; this is because $e^{\ell_s-\ell_B}$ should attain its maximal value when $|s-B|$ close to 0 as analyzed before. 
Let $\mu_{\kappa} = \mathbb{E}_B[\tilde{\ell}_{\kappa} /b ] =0$ and $\sigma_{\kappa}^2  = {\rm Var}_B[ \tilde{\ell}_{\kappa} /b]=1$. When $\kappa = b^2 \to \infty$, using local central limit theorem  (\ref{local_CLT}), we have that
\begin{align}
\kappa^{1/2} \mathbb{P} \left( \tilde{\ell}_{\kappa} \in x - \log \hat{M}_{\kappa} + (0, \delta] \right) \to \frac{\delta}{\sigma_{\kappa}}\phi\left(\frac{\mu_{\kappa}}{\sigma_{\kappa}} \right).
\end{align}

{\it Condition V} is checked as follows. Note that the terms inside the $M_{\kappa}$ are likelihood ratios with unit expectation since $\mathbb{E}_B[ \exp(\ell_B)] = 1$. Thus, $\exp(\ell_s - \ell_B)$ is a martingale and by a standard martingale inequality, $\mathbb P(\log M_\kappa > \epsilon \kappa^{1/2}) \leq \exp(-\epsilon \kappa^{1/2})$. Then using a similar argument as in \cite{Siegmund10}, one can show the other two inequalities, since $\widehat{M}_\kappa$ is an approximation to $M_\kappa$.

Finally, since all conditions are met, we can now apply the localization theorem for $b \to \infty$ and put things together to obtain
\begin{equation}\label{eq:asyapprox}
\mathbb{E}_B \left[ \frac{M_B}{S_B}    e^{-[\tilde{\ell}_B+\log M_B]}   ;   \tilde{\ell}_B + \log M_B \geq 0       \right]    = \frac{\mu^2}{2} \nu(\mu) \frac{1}{\sqrt{2\pi b^2}} (1+o(1)).
\end{equation} 
Substitute (\ref{eq:asyapprox}) back to the likelihood ratio identity (\ref{measure_transform}), and we arrive at the approximation in Theorem \ref{thm:tail_fixedsample}.

\section{PROOF OF THEOREM \ref{thm:ARL}}
\label{theorem_ARL}

The method for approximating the ARL is related to that used to analyze the offline scan $B$-statistic. In addition, we need the following lemma.
\begin{lemma}[Asymptotic null distribution of $T$] \label{lem:stoppingtime}
Under the null, when $b\rightarrow\infty$, the stopping time $T$ defined in (\ref{stopping-time}) is uniformly integrable and asymptotically exponentially distributed, i.e.,  
\[
|\mathbb{P}\{T\geq m\} - \exp(-\lambda_0 m)|\rightarrow 0,\] in the range where $m\lambda_0$ is bounded away from 0.
\end{lemma} 

\begin{proof}
The proof is based on adapting arguments in \cite{GLR-changepoint1995, SiegmundYakirImage,Yakir09}. The main idea is to show that  the number of boundary cross events for detection statistic over disjoint intervals converges to Poisson random variable in the total variation norm, resulted from the Poisson limit theorem (Theorem 1 in \cite{Arratia89}) for dependent samples. 
First, we show that the stopping time $T$ is asymptotically exponentially distributed. The analysis of the distribution of the stopping time is based on Poisson approximation. 
Define an indicator of the event $\textbf{1}_j$ such that the event $\textbf{1}\{ \max_{(j-1)m \leq t \leq jm} Z'_{B_0, t} > b\}$. Consider the time interval $[0, x]$. Note that the stopping time is not activated in the interval $[0, x]$, if and only if, all the relevant indicators are zero. For simplicity, we assume $x$ is divisible by $m$. Define the random variable $\widehat{W} = \sum_{j=1}^{x/m} \textbf{1}_j$. Hence, 
$
\{\widehat{W} = 0\}
= \{T_b > x\}. 
$
Thus, to characterize the tail probability of the stopping time $\mathbb{P}\{T_b >  x\}$, we show that the sum of the indicator functions converge to a Poisson distribution. 
\end{proof}

Using Lemma \ref{lem:stoppingtime}, we know for large $m$, $\mathbb{P}\{T\leq m\}$ is approximately $1-\exp(-\lambda_0 m) \approx \lambda_0 m$, and $\mathbb{E}\{T\}$ is equal to $ \lambda_0^{-1}$ asymptotically when $b \to \infty$. So the remaining question is to find the probability and the corresponding $\lambda_0$.
Consider
$
\mathbb{P}\{T \leq m\}  = \mathbb{P}\{\max_{2\leq t\leq m} Z_{B_0,t}' > b \}.  
$
Suppose $m>B_0$ and $\log b \ll m \ll b^{-1} e^{\frac{1}{2}b^2}$.
We will adopt a similar strategy to approximate this probability using the change-of-measure technique. 

Note that the covariance structures for online and offline scan $B$-statistics are different, so there will be different drift parameters when we invoke the localization theorem.
Using exponential tilting, we introduce a likelihood ratio   
\[\zeta_t =  b Z_{B_0, t}' - b^2/2.\] 
Again using the change-of-measure by likelihood ratio identity, we obtain
\begin{equation}
\begin{split}
 \mathbb{P}\left\{ \max_{2 \leq t\leq m} Z_{B_0,t}' > b   \right\}  
=  e^{-b^2/2} \sum_{t=2}^{m}     
     \mathbb{E}_t \left[ \frac{M_t'}{S_t'}    e^{- \left[\tilde{\zeta}_t + \log M_t \right]}; \tilde{\zeta}_t + \log M_t' \geq 0       \right],
\end{split}
\label{measure_transform_online}
\end{equation}
where 
\begin{equation*}
M_t' = \max_{2 \leq s \leq m} e^{\zeta_s - \zeta_t}, \quad 
S_t' =\sum_{2\leq s \leq m} e^{\zeta_s - \zeta_t}, \quad  \mbox{and} \quad
\tilde{\zeta}_t =  b(Z_{B_0, t}' - b).
\end{equation*}
Hence, one can again apply the localization theorem to find the approximation when $b\to \infty$, and the only differences are in the definition and characterization of global and local field terms.

\begin{lemma}[Local field of online scan $B$-statistic]\label{lm: local_field_online}
The mean, variance, and covariance of the local field $\{\zeta_{s}-\zeta_t\}$ are given by
\begin{align*}
&\mathbb{E}_t [\zeta_{s} - \zeta_t] =  -b^2(1-\rho_{s, t}), \quad  {\rm Var}_t [\zeta_{s} - \zeta_t] = 2b^2(1-\rho_{s,t}), \nonumber \\
&{\rm Cov}_t\left(\zeta_{s_1} - \zeta_t , \zeta_{s_2} - \zeta_t  \right) =  b^2\left( 1+ \rho_{s_1, s_2} - \rho_{s_1, t} -\rho_{s_2, t}\right),
 \end{align*}
where
\begin{equation} \label{eq:covonline}
\rho_{s, t} = {\rm Cov}(Z_{B_0, s}', Z_{B_0, t}' ) =\frac{   \binom{(B_0-|t-s|) \vee 0}{2} } { \binom{B_0}{2} }  .
\end{equation}
\end{lemma}
The proof can be found in Appendix \ref{App:covonline}.
Note that when $|t-s|$ is close to 0, $\mathbb{E}_t [\zeta_{s} - \zeta_t] $ is close to 0. With an increasing $|t-s|$, $\mathbb{E}_t [\zeta_{s} - \zeta_t] $ decreases until $|t-s|>B_0$ (when there are no overlapping test data in the sliding block), then $ \mathbb{E}_t[ \zeta_s-\zeta_t] $ becomes $-b^2$.  The values of $M_{\kappa}$ and $S_{\kappa}$ as in localization theorem will be determined by the values of $|j|$ close to 0. 

Now, again, we will use an argument based on Taylor expansion to find the drift term of the local field. When $|s-t|$ is close to 0, we can approximate $\{\zeta_{s}-\zeta_t\}$ as a two-sided random walk. 
Using Taylor expansion, we have 
\begin{equation}
\rho_{t+j, t} = 1-\frac{2B_0-1}{B_0(B_0-1)} |j| + o(|j|).
\end{equation} Let $\lambda = b[2(2B_0-1)]/[B_0(B_0-1)]^{1/2}$.  Hence, we can show that the mean, variance, and covariance of the local field are approximately
\begin{align*}
\lim_{|j|\to 0}\mathbb{E}_t [\zeta_{t+j} - \zeta_t] &=  -\frac{\lambda^2}{2}|j|, \nonumber \\
  \lim_{|j|\to 0 }{\rm Var}_t [\zeta_{t+j} - \zeta_t] &= \lambda^2|j| ,  \nonumber \\
\lim_{|j_1|\to 0, |j_2|\to 0 }{\rm Cov}_t\left(\zeta_{t+j_1} - \zeta_t , \zeta_{t+j_2} - \zeta_t  \right) & = \lambda^2 (|j_1| \wedge |j_2|).
 \end{align*}
As a result, by invoking the localization theorem through a similar set of steps, we obtain
\begin{equation}
\mathbb{P}  \left\{ T \leq m \right\}  
 =  m  \cdot  \frac{be^ {-\frac{1}{2}  b^2 }}{\sqrt{2 \pi}}   
\frac{ (2B_0-1)}{B_0(B_0-1)} \cdot \nu \left(  b \sqrt{\frac{2(2B_0-1)}{B_0(B_0-1)}} \right) (1+o(1)),
\label{tail-online}
\end{equation}
Matching this to above, we know $\lambda_0$ is the factor that multiplies $m$ and this leads to the desired result.

For online scan $B$-statistics, the standard Poisson limit cannot be directly applied, since the events $\{\textbf{1}_j \}$, $j = 1, \dots, x/m$, are not independent, and we need the generalized Poisson limit theorem  \citep{Arratia89}, which allows for dependence between the variables. The setup for the theorem is as follows. Let $I$ be an arbitrary index set, and for $\alpha \in I$, let $X_\alpha$ be a Bernoulli random variable with $p_\alpha = \mathbb P(X_\alpha = 1) > 0$. Let $W = \sum_{\alpha \in I} X_\alpha$. For each $\alpha \in I$, suppose we choose $B_\alpha \subset I$ with $\alpha \in B_\alpha$. Think of $B_\alpha$ as a ``neighborhood of dependence'' for each $\alpha$, such that $X_\alpha$ is independent or nearly independent of all of the $X_\beta$ for $\beta \notin B_\alpha$. Define $p_1 = \sum_{\alpha \in I} \sum_{\beta \in B_\alpha} p_\alpha p_\beta$, $p_2 = \sum_{\alpha_I}\sum_{\alpha\neq \beta \in B_\alpha} \mathbb{E}(X_\alpha X_\beta)$, $p_3 = \sum_{\alpha \in I} \mathbb{E}|\mathbb E(X_\alpha - p_\alpha | \sigma(X_\beta: \beta \in I - B_\alpha))|$, where $\sigma(\cdot)$ represents the $\sigma$-field generated by the corresponding random field. Loosely speaking, $p_1$ measures the neighborhood size, $p_2$ measures the expected number of neighbors of a given occurrence and $p_3$ measures the dependence between an event and the number of occurrences outside its neighborhood. Then, we have the following theorem.
\begin{theorem}[Poisson approximation, Theorem 1 in \citep{Arratia89}]
Let $W$ be the number of occurrences of dependent events, and let $Z$ be a Poisson random variable with $\mathbb E Z = \mathbb E W = \lambda > 0$. Then the total variation distance between the distributions of $W$ and $Z$ is bounded by
\[
\sup_{\|h\|=1} |\mathbb{E} h(W) - \mathbb E h(Z)|
\leq p_1 + p_2 + p_3.
\]
where $h: \mathbb Z^+ \rightarrow \mathbb R$, $\|h\| = \sup_{k\geq 0}|h(k)|$. 

\end{theorem} 
The theorem is a consequence of the powerful Chen-Stein method.

Invoking the above theorem in our online scan $B$-statistics setting, we can bound the total variation distance between the random variable, defined as the number of boundary cross events for the statistic over disjoint intervals, and a Poisson random variable with the same rate. In our setting, let $I = \{1, 2, \dots, x/m\}$ and $\mathcal{N}(j)=\{j-1, j, j+1\}$ where $j =2, \dots (x/m-1)$ (with obvious modifications for $j = 1$ and $j = x/m$). Then we can specify: 
\begin{align}
&p_1 = \sum_{j \in I} \sum_{i \in \mathcal N (j) \setminus \{j\}}\mathbb{P}\{\textbf{1}_j=1 \} \mathbb{P}\{\textbf{1}_i=1 \} = 2(x/m-2)\mathbb{P}\{ \textbf{1}_1=1\}^2 +2 \mathbb{P}\{ \textbf{1}_1=1\} \label{p1} ,\\
& p_2 = \sum_{j \in I} \sum_{i \in \mathcal N (j) \setminus \{j\}} \mathbb{P}\{ \textbf{1}_j=1,\textbf{1}_i=1\} = 2(x/m-1)\mathbb{P}\{\textbf{1}_1=1,\textbf{1}_2=1  \} \label{p2} ,\\
& p_3 = \sum_{j \in I}\mathbb{E}\left\{ |\mathbb{E}\{\textbf{1}_j | \sigma\{\textbf{1}_i: i \not\in \mathcal N (j) \} \} -\mathbb{E}\{\textbf{1}_j \}|\right\} \label{p3}.
\end{align}

We will show that $p_1$, $p_2$, and $p_3$ converge to 0 as $b \to \infty$. For $p_1$, the last summand in (\ref{p1}) is associated with the two edge elements. It follows that $p_1$ is asymptotically to $(2C+2) \mathbb{P}\{ \textbf{1}_1=1\}$, which will converge to zero as $b\to \infty$ since $\mathbb{P}\{ \textbf{1}_1=1\}$ converges to zero when $m$ is sub-exponential, i.e., $\log b \ll m \ll b^{-1} e^{\frac{1}{2}b^2}$.
Next, let us examine $p_2$ in (\ref{p2}). Redefine parameter sub-region
\begin{align*}
S_1 = [0, m-B_0/2], \quad
S_2 = [m-B_0/2,m+B_0/2], \quad
S_3 = [m+B_0/2,2m],
\end{align*}
and denote $Y_i$, $i = 1, 2, 3$ as $\{Y_i=1\}=\{\max_{t \in S_i} Z'_{B_0,t} >b\}$, which are the indicator functions of crossings of the threshold in the approximate sub-regions. Notice that the indicator functions $Y_1$ and $Y_3$ are independent of each other and they share the same distribution. We use the fact that unless  the crossing occurs in a shared sub-region, it must simultaneously occur in two disjoint sub-regions in order to have double crossing. As a consequence, we obtain the upper bound $ \textbf{1}_1\cdot  \textbf{1}_2 \leq Y_2+ Y_1 \cdot Y_3$, and  
\begin{align*}
\mathbb{P}\{  \textbf{1}_1=1,  \textbf{1}_2=1\} \leq \mathbb{P}\{Y_2=1\}+\mathbb{P}\{Y_1=1\}^2 \leq \mathbb{P}\{Y_2=1\} + \mathbb{P}\{\textbf{1}_1=1\}^2.
\end{align*}
The probability $ \mathbb{P}\{Y_2=1\}$ is proportional to $B_0 \cdot b e^{-\frac{1}{2}b^2}$. Consequently, $p_2$ is asymptotically bounded by $2 C (B_0/m +\mathbb{P}\{\textbf{1}_1=1\} )$. Hence, $p_2$ converges to zero if $\log b \ll m \ll b^{-1} e^{\frac{1}{2}b^2}$ whenever $b \to \infty$. For $p_3$ in (\ref{p3}), $\textbf{1}_j$ and $\textbf{1}_i$ are computed over non-overlapping observations and are therefore independent. Thus, the term $p_3$ vanishes.

Next prove that the collection of stopping times $\{T_b\}$ indexed by $b$ is {\it uniformly integrable}. Again consider the sequence of indicators $\{\textbf{1}_j\}$, $j = 2k$ and $k=1, 2, \dots$. Define the random variable $\tau$ that identifies the index of the first indicator in the sequence that obtains the value one:
$
\tau = \inf\{k: \textbf{1}_{2k}=1\}.
$
Note that $\tau$ has a geometric distribution. 
Moreover, since $T_b \leq 2m\tau$ we obtain that
\[
\mathbb{P}\{ T_b > x\} \leq \mathbb{P}\{\tau > x/ (2m)\} = (1-\mathbb{P}(\textbf{1}_2=1))^{\lfloor x/(2m)\rfloor}.
\]
The conclusion then follows from that $1/m\cdot \mathbb{P}(\textbf{1}_2=1)$ converges to 0.

\section{SKEWNESS CALCULATION}\label{proof:skewness}
In the following, Lemma \ref{lemma: mmdiii}, Lemma \ref{lemma: mmdiij}, and Lemma \ref{lemma: mmdijr} are used to derive the final expression for the skewness of the scan $B$-statistic:

\begin{lemma} \label{lemma: mmdiii}
Under null hypothesis, 
\begin{equation*}
\begin{split}
&\mathbb{E}\left[ \left(\mbox{MMD}^2 (X_i,Y) \right)^3 \right]\\
=& \frac{8(B-2)}{B^2(B-1)^2} \mathbb{E} \left[ h(x,x', y,y') h(x',x'', y', y'') 
h(x'',x, y'', y) \right]
+\frac{4}{B^2(B-1)^2}\mathbb{E} \left[ h(x,x',y,y')^3\right].
\end{split}
\end{equation*}
\end{lemma}
\begin{proof}
Note that
\begin{align*}
\mathbb{E}\left[ \left(\mbox{MMD}^2 (X_i,Y) \right)^3 \right]
=& {B \choose 2}^{-3}  \mathbb{E}\left[ \left(    \sum_{a < b} h(X_{i,a},X_{i,b}, Y_a, Y_b)  \right)^3 \right]\\
=&{B \choose 2}^{-3}  \sum_k C_k \mathbb{E}\left[ h_{ab}h_{cd} h_{ef} \right],
\end{align*}
where for simplicity we write $h_{ab}=h(X_{i,a},X_{i,b}, Y_a, Y_b)$ and define $C_k$ the corresponding number of combination under specific structure. Most of the terms in $\mathbb{E}\left[ h_{ab}h_{cd} h_{ef} \right]$ vanish under the null. By enumerating all the combinations, only two terms are nonzero:
$ \mathbb{E}\left[ h_{ab}h_{bc} h_{ca} \right]$ and $\mathbb{E}\left[ h_{ab}h_{ab} h_{ab} \right]$. Then, 
\begin{align*}
& \mathbb{E}\left[ \left(\mbox{MMD}^2 (X_i,Y) \right)^3 \right]
= {B \choose 2}^{-3} {B \choose 2} 2 (B-2)\mathbb{E} \left[ h_{ab}h_{bc} h_{ca} \right] +{B \choose 2}^{-3} {B \choose 2} \mathbb{E}\left[ h_{ab}h_{ab} h_{ab} \right] \\
= & \frac{8(B-2)}{B^2(B-1)^2} \mathbb{E} \left[ h(X_{i,a},X_{i,b}, Y_a, Y_b) h(X_{i,b},X_{i,c}, Y_b, Y_c) 
h(X_{i,c},X_{i,a}, Y_c, Y_a) \right] \\
& +\frac{4}{B^2(B-1)^2}\mathbb{E} \left[ h(X_{i,a},X_{i,b}, Y_a, Y_b)^3\right] \\
= & \frac{8(B-2)}{B^2(B-1)^2} \mathbb{E} \left[ h(x,x', y,y') h(x',x'', y', y'') 
h(x'',x, y'', y) \right]
+\frac{4}{B^2(B-1)^2}\mathbb{E} \left[ h(x,x',y,y')^3\right].
\end{align*}
\end{proof}

\begin{lemma} \label{lemma: mmdiij}
Under null hypothesis,
\begin{align*}
&\mathbb{E}\left[ \left(\mbox{MMD}^2 (X_i,Y) \right)^2 \mbox{MMD}^2(X_j, Y) \right]_{i \neq j} \\
=&  \frac{8(B-2)}{B^2(B-1)^2} \mathbb{E} \left[ h(x,x', y,y') h(x',x'', y', y'') 
h(x''',x'''', y'', y) \right]
\nonumber \\
 & \quad \quad +\frac{4}{B^2(B-1)^2}\mathbb{E} \left[ h(x,x',y,y')^2h(x'',x''',y,y')\right].
\end{align*}
\end{lemma}
\begin{proof}
Note that
\begin{align*}
& \mathbb{E}\left[ \left(\mbox{MMD}^2 (X_i,Y) \right)^2 \mbox{MMD}^2(X_j, Y) \right]_{i \neq j}\\
= &{B \choose 2}^{-3}  \mathbb{E}\left[ \left(    \sum_{a < b} h(X_{i,a},X_{i,b}, Y_a, Y_b)  \right)^2 \left(    \sum_{a < b} h(X_{j,a},X_{j,b}, Y_a, Y_b)  \right) \right] \\
=& {B \choose 2}^{-3}  \sum_k C_k \mathbb{E}\left[ h_{i,ab}h_{i,cd} h_{j,ef} \right],
\end{align*}
where for simplicity we write $h_{i,ab}=h(X_{i,a},X_{i,b}, Y_a, Y_b)$ and define $C_k$ the corresponding number of combination under specific structure. Similarly, most of the terms in $\mathbb{E}\left[ h_{i,ab}h_{i,cd} h_{j,ef} \right]$ vanish under the null. By enumerating all the combinations, only two terms are nonzero:
$ \mathbb{E}\left[ h_{i,ab}h_{i,bc} h_{j,ca} \right]$ and $\mathbb{E}\left[ h_{i,ab}h_{i,ab} h_{j,ab} \right]$. Then, 
\begin{align*}
& \mathbb{E}\left[ \left(\mbox{MMD}^2 (X_i,Y) \right)^2 \mbox{MMD}^2(X_j, Y) \right]_{i \neq j} \\
= & {B \choose 2}^{-3} {B \choose 2} 2 (B-2)\mathbb{E} \left[ h_{i,ab}h_{i,bc} h_{j,ca} \right] +{B \choose 2}^{-3} {B \choose 2} \mathbb{E}\left[ h_{i,ab}h_{i,ab} h_{j,ab} \right] \\
= & \frac{8(B-2)}{B^2(B-1)^2} \mathbb{E} \left[ h(X_{i,a},X_{i,b}, Y_a, Y_b) h(X_{i,b},X_{i,c}, Y_b, Y_c) 
h(X_{j,c},X_{j,a}, Y_c, Y_a) \right] \\
&+\frac{4}{B^2(B-1)^2}\mathbb{E} \left[ h(X_{i,a},X_{i,b}, Y_a, Y_b)^2h(X_{j,a},X_{j,b}, Y_a, Y_b)\right] \\
= & \frac{8(B-2)}{B^2(B-1)^2} \mathbb{E} \left[ h(x,x', y,y') h(x',x'', y', y'') 
h(x''',x'''', y'', y) \right] \nonumber \\
& \quad +\frac{4}{B^2(B-1)^2}\mathbb{E} \left[ h(x,x',y,y')^2h(x'',x''',y,y')\right].
\end{align*}
\end{proof}

\begin{lemma} \label{lemma: mmdijr}
Under null hypothesis,
\begin{align*}
& \mathbb{E} \left[ \mbox{MMD}^2(X_i, Y) \mbox{MMD}^2(X_j, Y) \mbox{MMD}^2(X_r, Y) \right]_{i \neq j \neq r} \\
=&\frac{8(B-2)}{B^2(B-1)^2} \mathbb{E} \left[ h(x,x', y,y') h(x'',x''', y', y'') 
h(x'''',x''''', y'', y) \right] \\
&+\frac{4}{B^2(B-1)^2}\mathbb{E} \left[ h(x,x',y,y')h(x'',x''',y,y')h(x'''',x''''',y,y')\right].
\end{align*}
\end{lemma}
\begin{proof}
Note that
\begin{align*}
& \mathbb{E} \left[ \mbox{MMD}^2(X_i, Y) \mbox{MMD}^2(X_j, Y) \mbox{MMD}^2(X_r, Y) \right]_{i \neq j \neq r} \\
=& {B \choose 2}^{-3}  \mathbb{E}\left[ \left(    \sum_{a < b} h(X_{i,a},X_{i,b}, Y_a, Y_b)  \right) \left(    \sum_{c < d} h(X_{j,c},X_{j,d}, Y_c, Y_d)  \right) \left(    \sum_{e < f} h(X_{r,e},X_{r,f}, Y_e, Y_f)  \right)\right] \\
=& {B \choose 2}^{-3}  \sum_k C_k \mathbb{E}\left[ h_{i,ab}h_{j,cd} h_{r,ef} \right].
\end{align*}
Similarly, most of the terms in $\mathbb{E}\left[ h_{i,ab}h_{j,cd} h_{r,ef} \right]$ vanish under the null. By enumerating all the combinations, only two terms are nonzero:
$ \mathbb{E}\left[ h_{i,ab}h_{j,bc} h_{r,ca} \right]$ and $\mathbb{E}\left[ h_{i,ab}h_{j,ab} h_{r,ab} \right]$. Then, 
\begin{align*}
& \mathbb{E} \left[ \mbox{MMD}^2(X_i, Y) \mbox{MMD}^2(X_j, Y) \mbox{MMD}^2(X_r, Y) \right]_{i \neq j \neq r}  \\
= & {B \choose 2}^{-3} {B \choose 2} 2 (B-2)\mathbb{E} \left[ h_{i,ab}h_{j,bc} h_{r,ca} \right] +{B \choose 2}^{-3} {B \choose 2} \mathbb{E}\left[ h_{i,ab}h_{j,ab} h_{r,ab} \right] \\
= & \frac{8(B-2)}{B^2(B-1)^2} \mathbb{E} \left[ h(X_{i,a},X_{i,b}, Y_a, Y_b) h(X_{j,b},X_{j,c}, Y_b, Y_c) 
h(X_{r,c},X_{r,a}, Y_c, Y_a) \right] \\
&+\frac{4}{B^2(B-1)^2}\mathbb{E} \left[ h(X_{i,a},X_{i,b}, Y_a, Y_b) h(X_{j,a},X_{j,b}, Y_a, Y_b)h(X_{r,a},X_{r,b}, Y_a, Y_b) \right] \\
= & \frac{8(B-2)}{B^2(B-1)^2} \mathbb{E} \left[ h(x,x', y,y') h(x'',x''', y', y'') 
h(x'''',x''''', y'', y) \right] \\
&+\frac{4}{B^2(B-1)^2}\mathbb{E} \left[ h(x,x',y,y')h(x'',x''',y,y')h(x'''',x''''',y,y')\right].
\end{align*}
\end{proof}
Using results from  Lemma \ref{lemma: mmdiii}, Lemma \ref{lemma: mmdiij}, and Lemma \ref{lemma: mmdijr}, and we can derive the final expression for the skewness of the scan $B$-statistic, as summarized in Lemma \ref{thm: skewness}.
\begin{proof}
We can write the raw third-order moment as
\begin{align*}
\mathbb{E}[Z_B^3]
=&\mathbb{E}\left[ \left(  \frac{1}{N} \sum_{i=1}^N \mbox{MMD}^2(X_i,Y)  \right)^3 \right] \\
=&\frac{1}{N^3} \mathbb{E} \left[  \left(\sum_{i=1}^N   \mbox{MMD}^2(X_i,Y)    \right)      \left(\sum_{j=1}^N   \mbox{MMD}^2(X_j,Y)    \right)     \left(\sum_{r=1}^N   \mbox{MMD}^2(X_r,Y)    \right)     \right]\\
=&\frac{1}{N^3} N   \mathbb{E}\left[   \left( \mbox{MMD}^2 (X_i,Y) \right)^3   \right]      
+  \frac{1}{N^3}  {3 \choose 2} {N \choose 1} {N-1 \choose 1}\mathbb{E} \left[  \left(  \mbox{MMD}^2 (X_i,Y)\right)^2  \mbox{MMD}^2(X_j,Y) \right]_{i\neq j}\\
&+\frac{1}{N^3}  {N \choose 1} {N-1  \choose 1} {N-2 \choose 1} \mathbb{E} \left[ \mbox{MMD}^2(X_i, Y) \mbox{MMD}^2(X_j, Y) \mbox{MMD}^2(X_r, Y) \right]_{i \neq j \neq r} \\
=& \frac{1}{N^2} \left\{ \frac{8(B-2)}{B^2(B-1)^2} \mathbb{E} \left[ h(x,x', y,y') h(x',x'', y', y'') 
h(x'',x, y'', y) \right]
+\frac{4}{B^2(B-1)^2}\mathbb{E} \left[ h(x,x',y,y')^3\right] \right\} \\
&+\frac{3(N-1)}{N^2}  \left\{ \frac{8(B-2)}{B^2(B-1)^2} \mathbb{E} \left[ h(x,x', y,y') h(x',x'', y', y'') 
h(x''',x'''', y'', y) \right] \right.\\
&\qquad \left.+\frac{4}{B^2(B-1)^2}\mathbb{E} \left[ h(x,x',y,y')^2h(x'',x''',y,y')\right] \right\} \\
&+\frac{(N-1)(N-2)}{N^2} \left\{ \frac{8(B-2)}{B^2(B-1)^2} \mathbb{E} \left[ h(x,x', y,y') h(x'',x''', y', y'') 
h(x'''',x''''', y'', y) \right]\right. \\
&\qquad \left.+\frac{4}{B^2(B-1)^2}\mathbb{E} \left[ h(x,x',y,y')h(x'',x''',y,y')h(x'''',x''''',y,y')\right]   \right\} 
\end{align*}
\end{proof}

\section{$Z_B$ DOES NOT CONVERGE TO GAUSSIAN}\label{app:scewness}

Note that the third-order moment of $Z_B$ scales as $\mathcal{O}(B^{-3})$ (due to (\ref{skewness_eqn})), but when dividing by its variance which scales as $\mathcal{O}(B^{-2})$, the skewness becomes a constant with respect to $B$. Furthermore, examining the Taylor expansion of moment generating function at $\theta = 0$, we have
\[
\mathbb{E}[e^{\theta Z_B'}] = 1 + \underbrace{\mathbb{E}[Z_B']}_0\theta + \frac{\theta^2}{2} \underbrace{\mathbb{E}[(Z_B')^2]}_{1}
+ \frac{\theta^3}{6} \mathbb{E}[(Z_B')^3 e^{\theta Z_B'}]
+ o(\theta^3).
\]
Recall that the moment generating function of a standard normal $Z$ is given by
$
\mathbb{E}[e^{\theta Z}] = 1 + \theta^2/2 + o(\theta^3). 
$
The difference between the two moment generating functions is given by
\begin{equation}
\left|\mathbb{E}[e^{\theta Z_B'}] - \mathbb{E}[e^{\theta Z}]\right|
= \frac{|\theta|^3}{6} |\mathbb{E}[(Z_B')^3 e^{\theta' Z_B'}]| + o(\theta^3)
> \frac{|\theta|^3}{6} c |\mathbb{E}[(Z_B')^3]| + o(\theta^3),
\label{diff2}
\end{equation}
where the inequality is due to the fact that $e^{\theta' Z_B'} > 0$ and we may assume it is larger than an absolute constant $c$. Note that the first term on the right hand side of (\ref{diff2}) is given by
$
(c\theta^3/6)\rm{Var}[Z_B]^{-3/2}|\mathbb{E}[{Z_B}^3]|, 
$
which is clearly bounded away from zero. Hence, 
\[
\left|\mathbb{E}[e^{\theta Z_B'}] - (1  + \frac{\theta^2}{2})\right| > \frac{|\theta|^3}{6} \gamma + o(\theta^3)
\]
for some constant $\gamma > 0$. This shows that the difference between the moment generating functions of $Z_B'$ and a standard normal is always non-zero and, hence, $Z_B'$ does not converge to a standard normal in any sense. This explains why incorporating the skewness of $Z_B$ can improve the accuracy of the approximations for SL in Theorem \ref{thm:tail_fixedsample} and for ARL in Theorem \ref{thm:ARL}.

\section{MORE DETAILS FOR REAL-DATA EXPERIMENTS}\label{app:real_data}

\subsection{CENSREC-1-C Speech Dataset} CENSREC-1-C is a real-world speech dataset in the Speech Resource Consortium (SRC) corpora provided by National Institute of Informatics (NII)\footnote{Available from http://research.nii.ac.jp/src/en/CENSREC-1-C.html}. 
	This dataset contains two categories of data: (1) Simulated data. The simulated speech data are constructed by concatenating several utterances spoken by one speaker. Each concatenated sequence is then added with 7 different levels of noise from 8 different environments. So there are totally 56 different types of noise. Each noise setting contains 104 sequences from 52 males and 52 females speakers. (2) Recording data. The recording data is from two real-noisy environments (in university restaurant and in the vicinity of highway), and with two Signal Noise Ratio (SNR) settings (lower and higher). Ten subjects were employed for recording, and each one has four speech sequence data.  
	
	{\it Experiment Settings.}
	We will compare our algorithm with the baseline algorithm from \citep{density-ratio2013}. \citep{density-ratio2013} only utilized 10 sequences from ``STREET\_SNR\_HIGH'' setting in recording data. Here we will use all the settings in recording data, the SNR level 20 dB and clean signals from simulated data. See Figure~\ref{speech_example} for some examples of the testing data, as well as the statistics computed by our algorithm. For each sequence, we decompose it into several segments. Each segment consists of two types of signals (noise vs speech). Given the reference data from noise, we want to detect the point where the signal changes from noise to speech. 
	
	{\it Evaluation Metrics.}
	We use Area Under Curve (AUC) to evaluate the computed statistics, like in \citep{density-ratio2013}. Specifically, for each test sequence that consists of two signal distributions, we will mark the points as change-points whose statistics exceed the given threshold. If the distance between the detected point and true change-point is within the size of detection window, then we consider it as True Alarm (True Positive). Otherwise it is a False Alarm (False Positive). 
	
	We use 10\% of the sequences to tune the parameters of both algorithms, and use the rest 90\% for reporting AUC. The kernel bandwidth is tuned in $\{0.1d_{\rm med}, 0.5d_{\rm med}, d_{\rm med}, 2d_{\rm med}, 5d_{\rm med}\}$, where $d_{\rm med}$ is the median of pairwise distances of reference data. Block size is fixed to be 50, and the number of blocks is simply tuned in $\{10, 20, 30\}$. 
	
	{\it Results.}
	Table~\ref{tabauc} shows the AUC of two algorithms on different background settings. Our algorithm outperforms the baseline on most cases. Both algorithms are performing quite well on the simulated clean data, since the difference between speech signals and background is more significant than the noisy ones. The averaged AUC of our algorithm on all these settings is \textbf{.8014}, compared to \textbf{.7578} achieved by the baseline algorithm. See the ROC curves in Figure~\ref{auc_speech_fig} for a complete comparison. 

\begin{figure}[h!]
	\centering
	\begin{tabular}{ccc}
		\includegraphics[scale=0.30, trim=0 0 0 0]{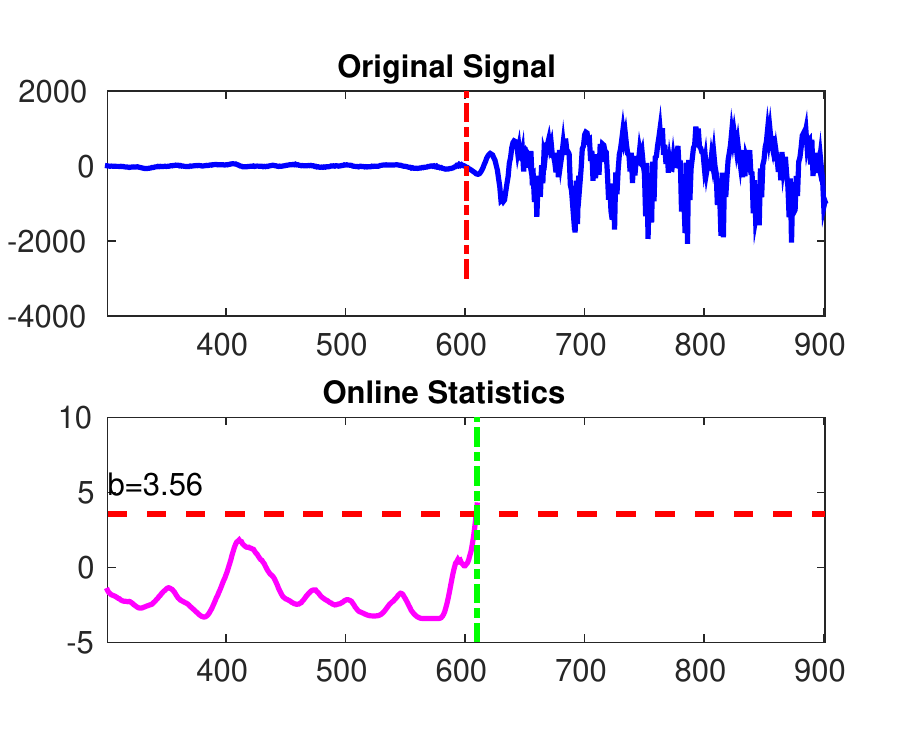} &
		\includegraphics[scale=0.30, trim=0 0 0 0]{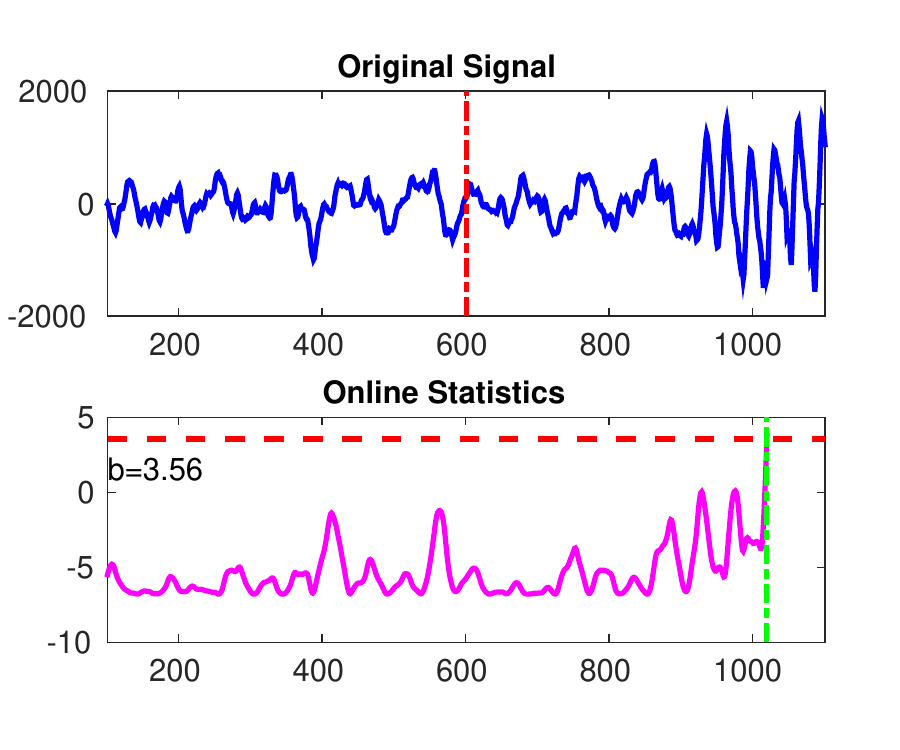} &
		\includegraphics[scale=0.30, trim=0 0 0 0]{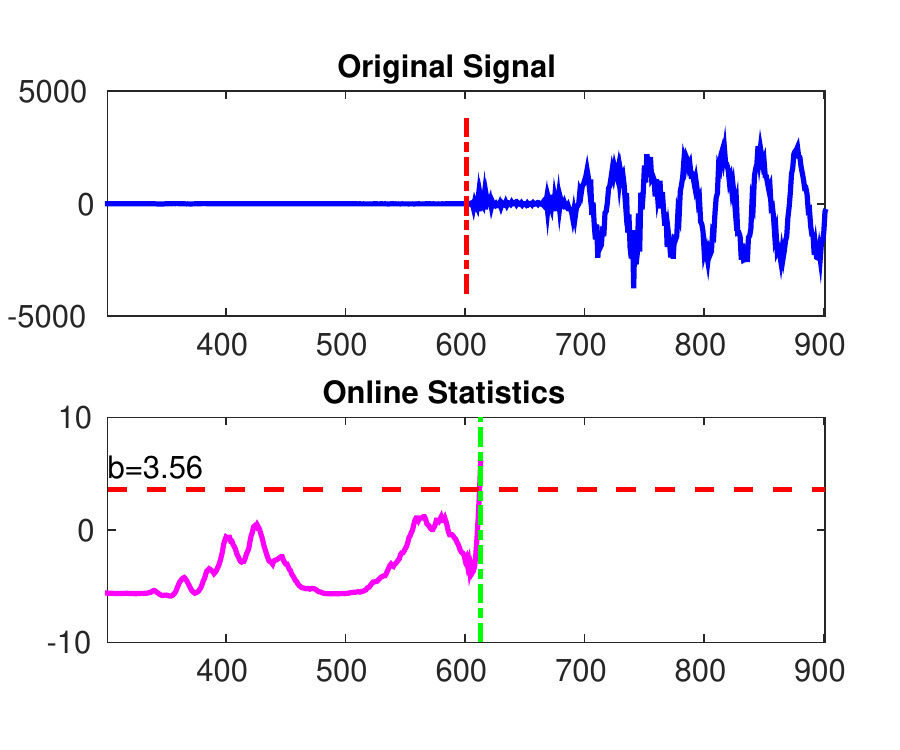} \\
	{RESTAURANT SNR HIGH}
	&
	{STREET SNR HIGH}
	&
	{Clean-1}
	\\
		\includegraphics[scale=0.30, trim=0 0 0 0]{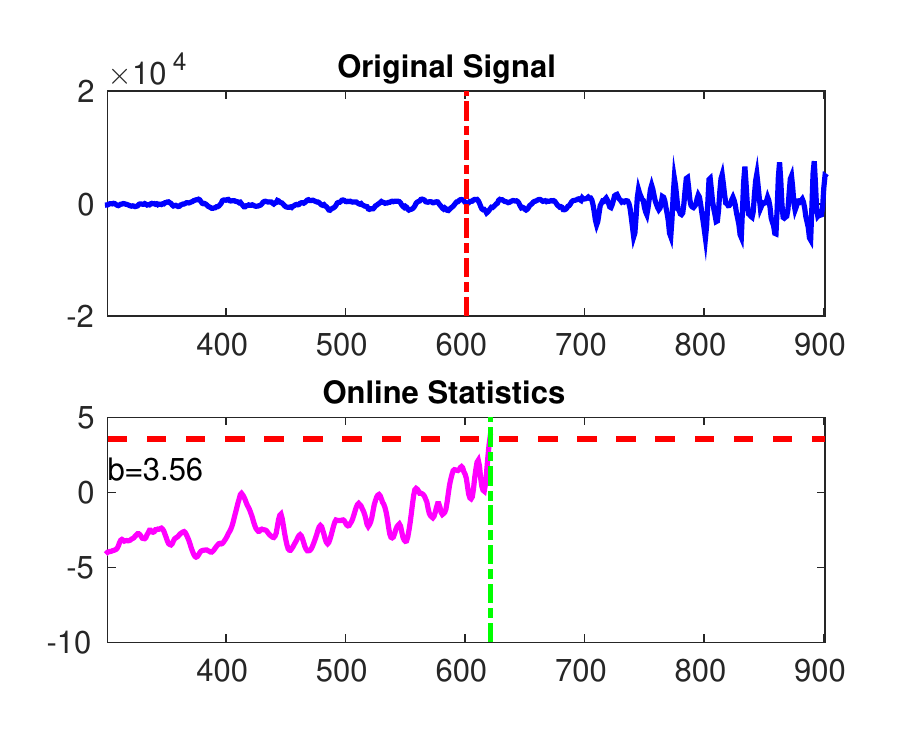}
		&\includegraphics[scale=0.30, trim=0 0 0 0]{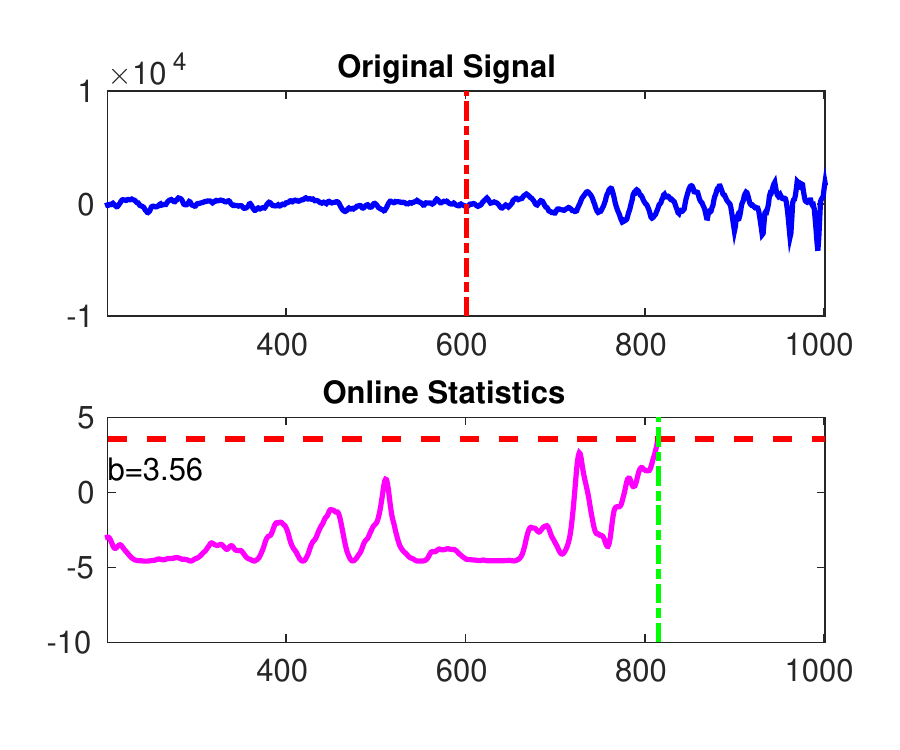}
		&
		\includegraphics[scale=0.30, trim=0 0 0 0]{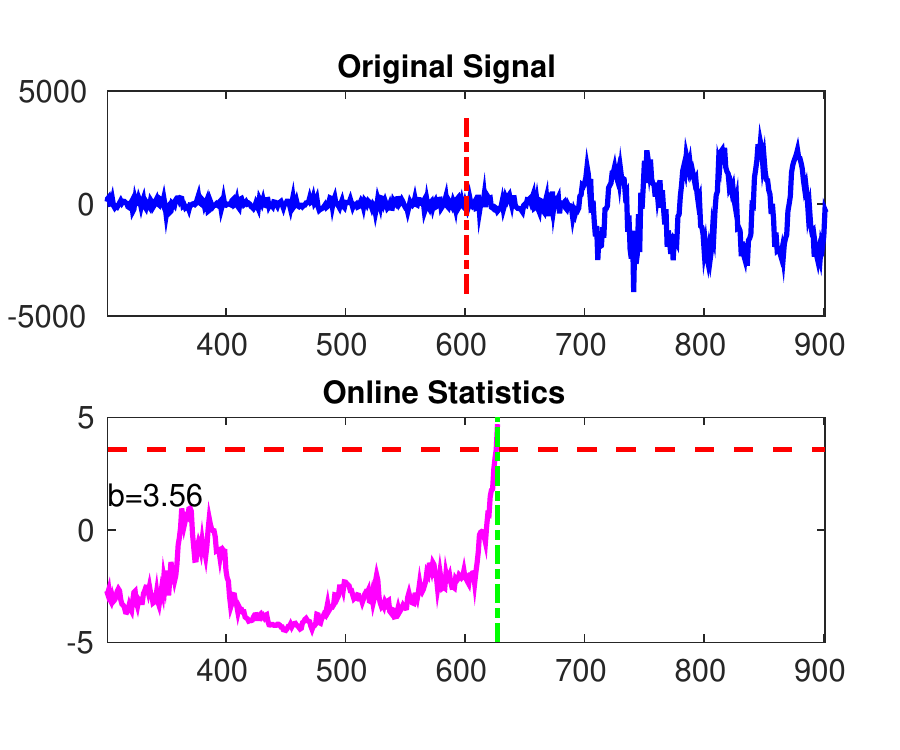} \\
		{RESTAURANT SNR LOW}
		& {STREET SNR LOW}
		& {SNR 20dB}
\end{tabular}
	\caption{Examples of speech dataset. The red vertical bar shown in the upper part of each figure is the ground truth of change-point; The green vertical bar shown in the lower part is the change-point detected by our algorithm (the point where the statistic exceeds the threshold). We also plot the threshold as a red dashed horizontal line in each figure. Once the statistics touch the threshold, we will stop the detection. }
	\label{speech_example}
\end{figure}

\begin{table}[h!] \label{tabauc}
	\begin{center}
					\caption{AUC results in CENSREC-1-C speech dataset. 
				Recording data are from RESTAURANT\_SNR\_HIGH (RH), RESTAURANT\_SNR\_LOW (RL), STREET\_SNR\_HIGH (SH) and STREET\_SNR\_LOW (SL).}
			\begin{tabular}{|c|c|c|c|c|}
			\hline
			& RH & RL & SH & SL \\
			\hline
			Ours & \textbf{0.7800} & \textbf{0.7282} & \textbf{0.6507} & \textbf{0.6865} \\
			\hline
			Baseline & 0.7503 & 0.6835 & 0.4329 & 0.6432 \\
			\hline
			\end{tabular}				\end{center}
		\end{table}

		\begin{table}[h!]
					\caption{Simulate data with low SNR, with noise from different environment.}
		\begin{center}
			\begin{tabular}{|c|c|c|c|c|c|c|c|c|}
			\hline
			& C1 & C2 & C3 & C4 & C5 & C6 & C7 & C8 \\
			\hline
			Ours & \textbf{0.9413} & \textbf{0.9446} & \textbf{0.9236} & \textbf{0.9251} & \textbf{0.9413} & \textbf{0.9446} & \textbf{0.9236} & \textbf{0.9251} \\
			\hline
			Baseline & 0.9138 & 0.9262 & 0.8691 & 0.9128 & 0.9138 & 0.9216 & 0.8691 & 0.9128 \\
			\hline
			\end{tabular}
			\end{center}
		\end{table}
\begin{table}[h!]	
					\caption{Simulated data with SNR = 20dB, with noise from different environment.}
		\begin{center}
			\begin{tabular}{|c|c|c|c|c|c|c|c|c|}
			\hline
			& S1 & S2 & S3 & S4 & S5 & S6 & S7 & S8 \\
			\hline
			Ours & 0.7048 & \textbf{0.7160} & \textbf{0.7126} & \textbf{0.7129} & \textbf{0.7094} & \textbf{0.7633} & \textbf{0.6796} & \textbf{0.7145}\\
			\hline
			Baseline & \textbf{0.7083} & 0.6681 & 0.6490 & 0.7119 & 0.6994 & 0.6815 & 0.6487 &  0.6541\\
			\hline
			\end{tabular}
	\end{center}
\end{table}
\begin{figure}
\begin{center}
\begin{tabular}{cccc}
		\includegraphics[scale=0.27, trim=30 50 0 0]{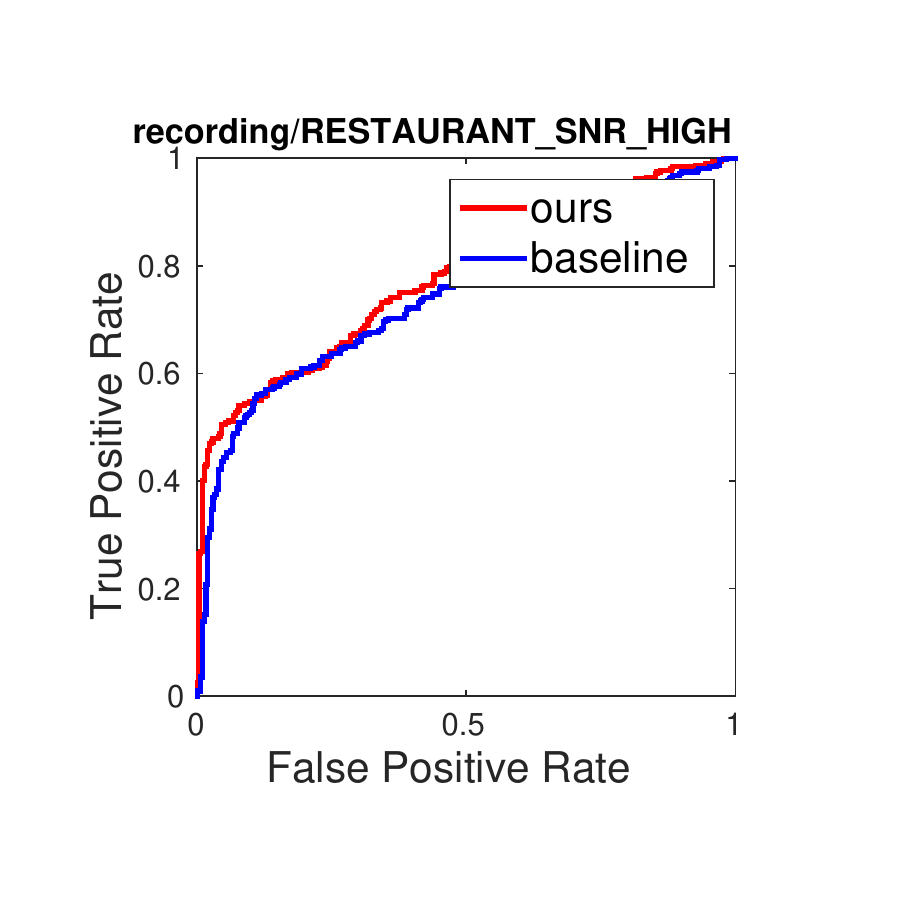}
&
		\includegraphics[scale=0.27, trim=30 50 0 0]{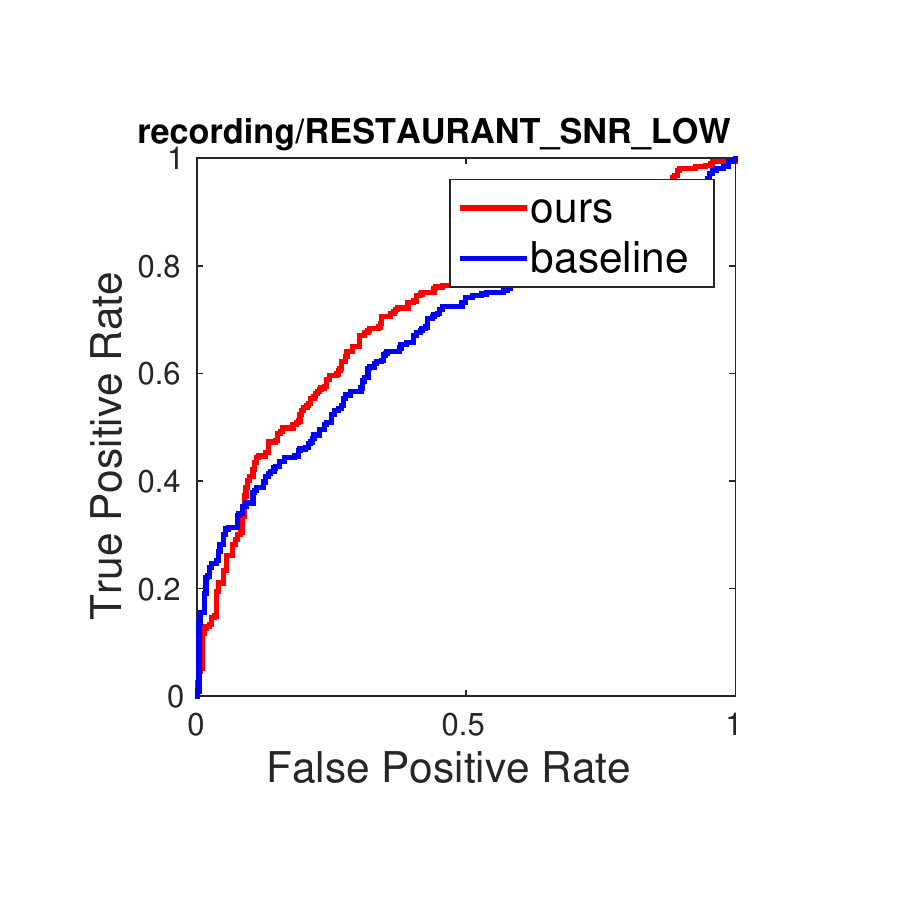}
		&
		\includegraphics[scale=0.27, trim=30 50 0 0]{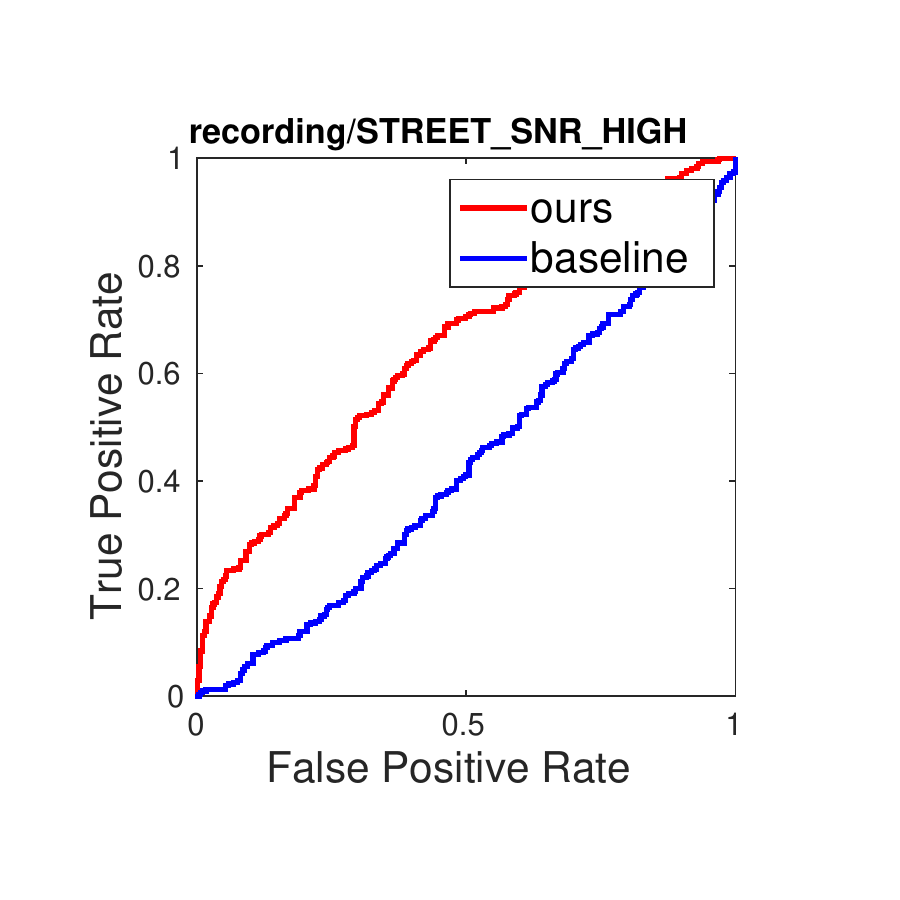}
		&
		\includegraphics[scale=0.27, trim=30 50 0 0]{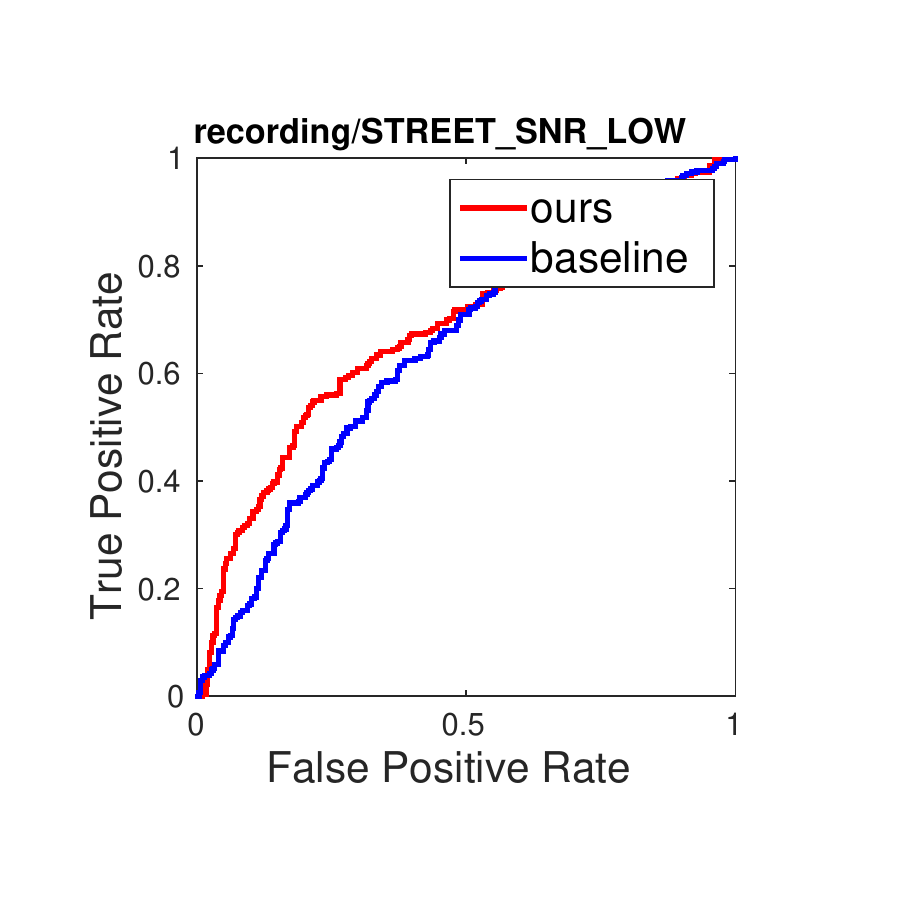}\\
		{RH} & {RL} & {SH} & {SL}\\
		\includegraphics[scale=0.27, trim=30 50 0 0]{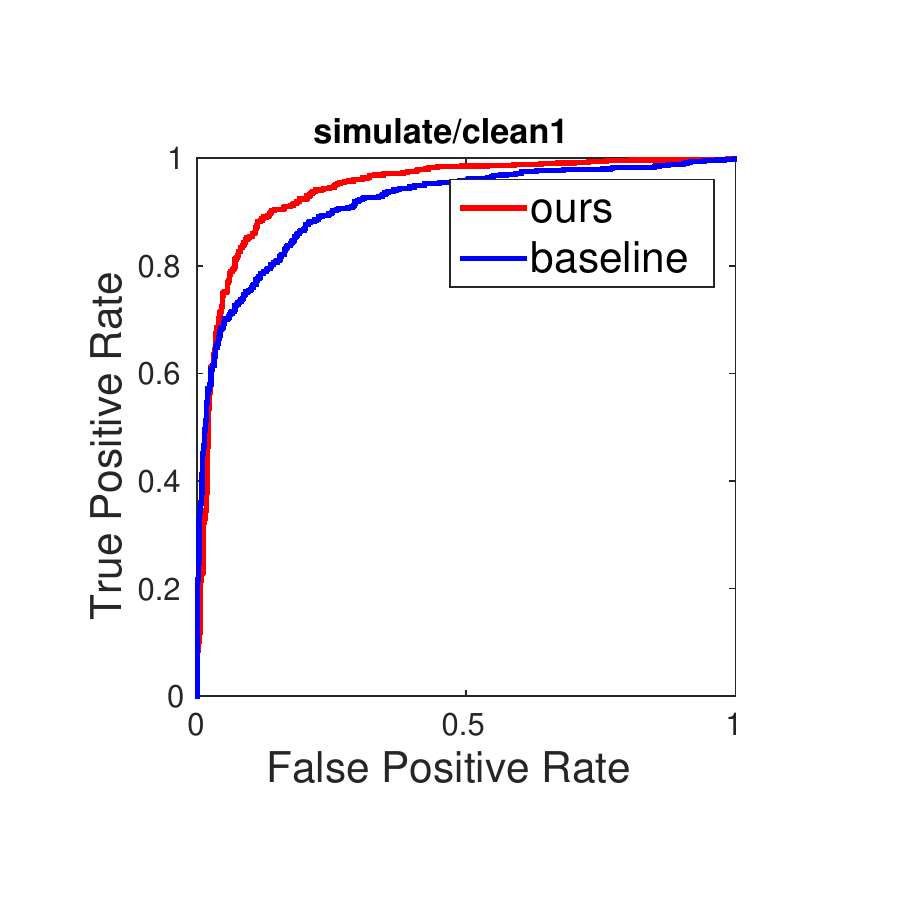}
		&
		\includegraphics[scale=0.27, trim=30 50 0 0]{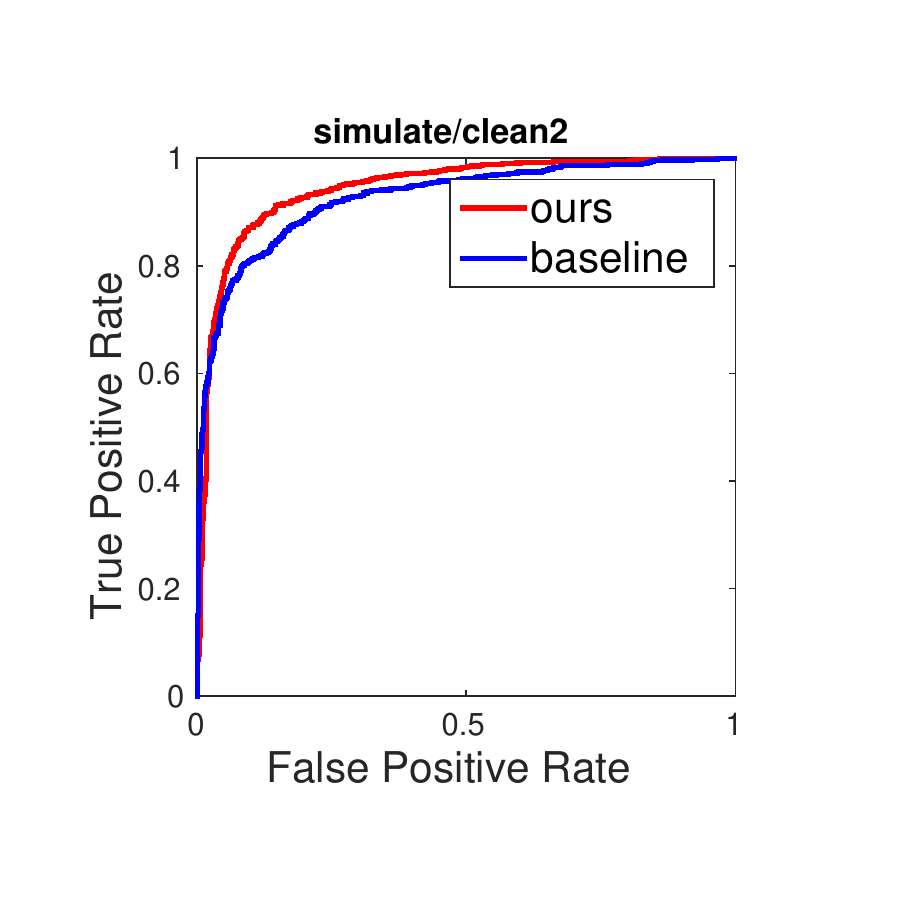}
		&
		\includegraphics[scale=0.27, trim=30 50 0 0]{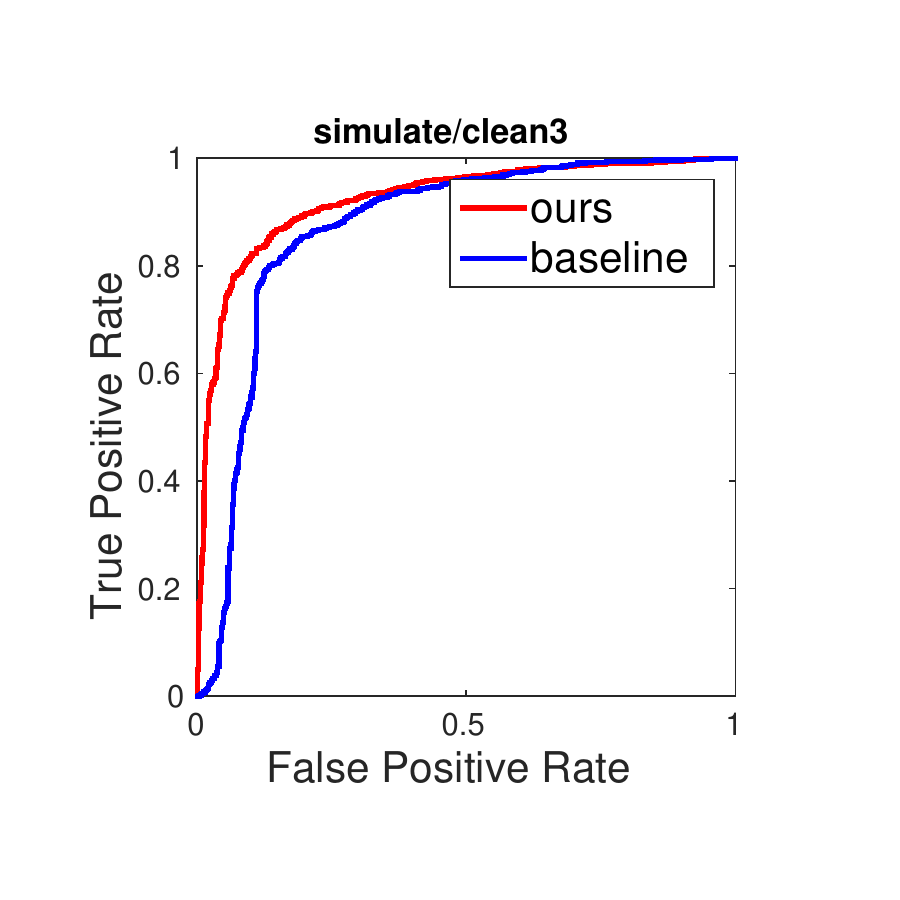}
		&
		\includegraphics[scale=0.27, trim=30 50 0 0]{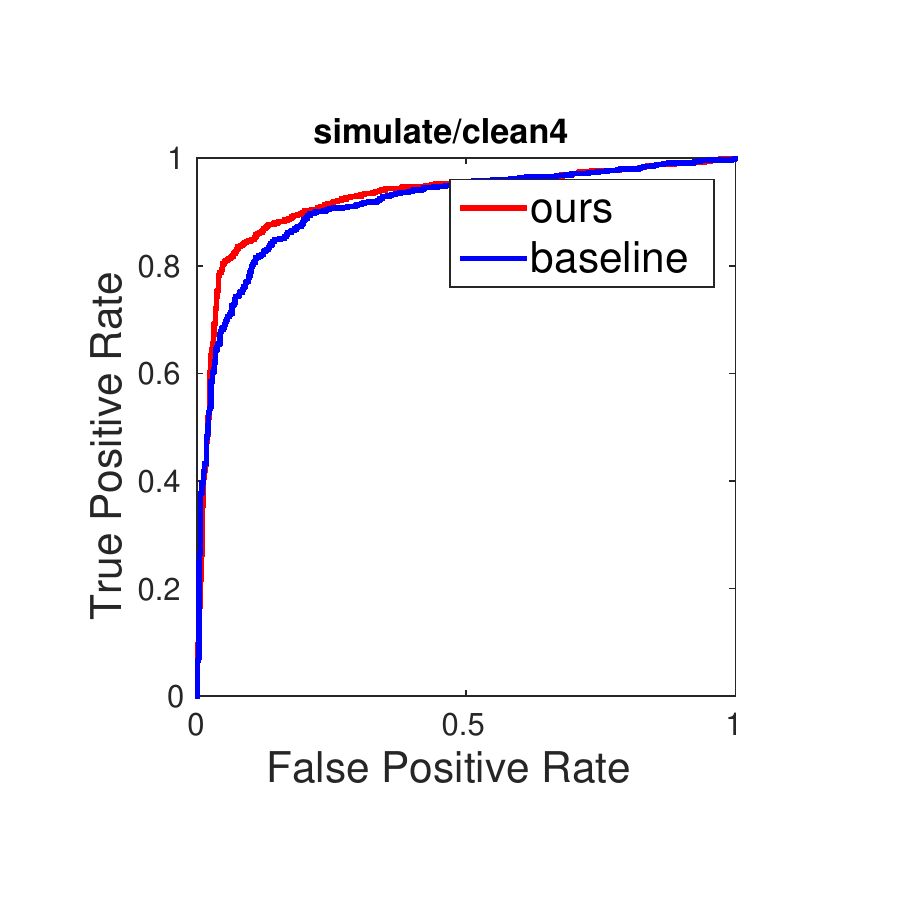} \\
		{C1} & {C2} & {C3} & {C4} \\
		\includegraphics[scale=0.27, trim=30 50 0 0]{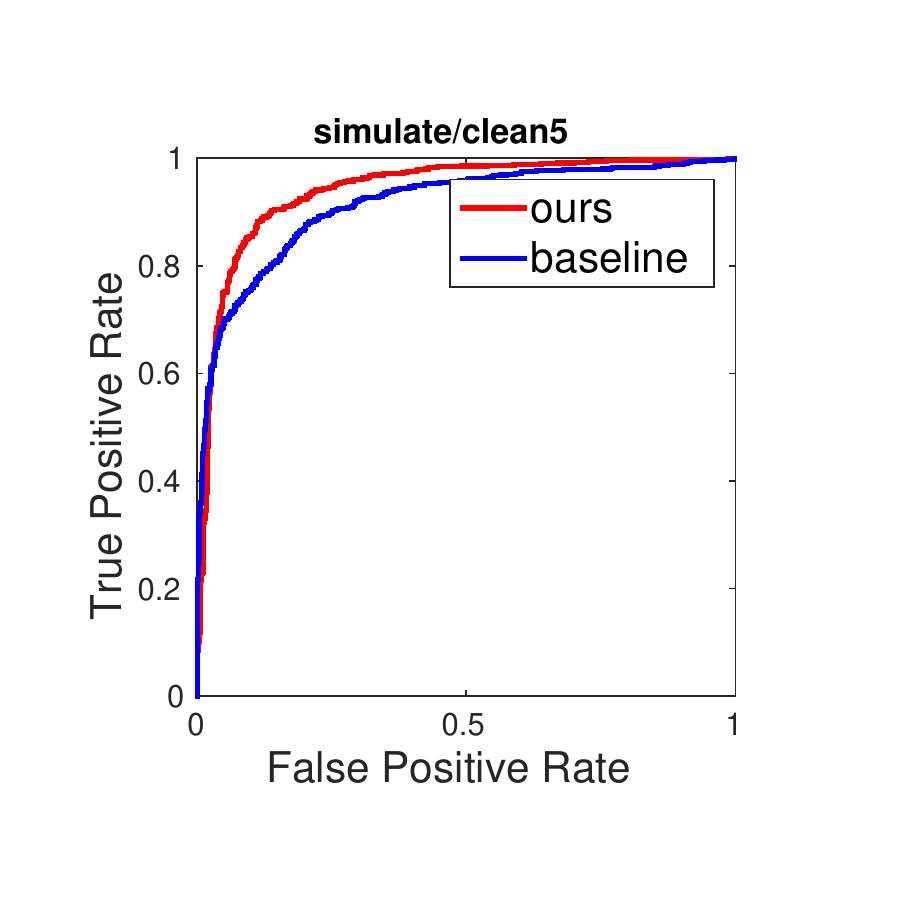}
		&
		\includegraphics[scale=0.27, trim=30 50 0 0]{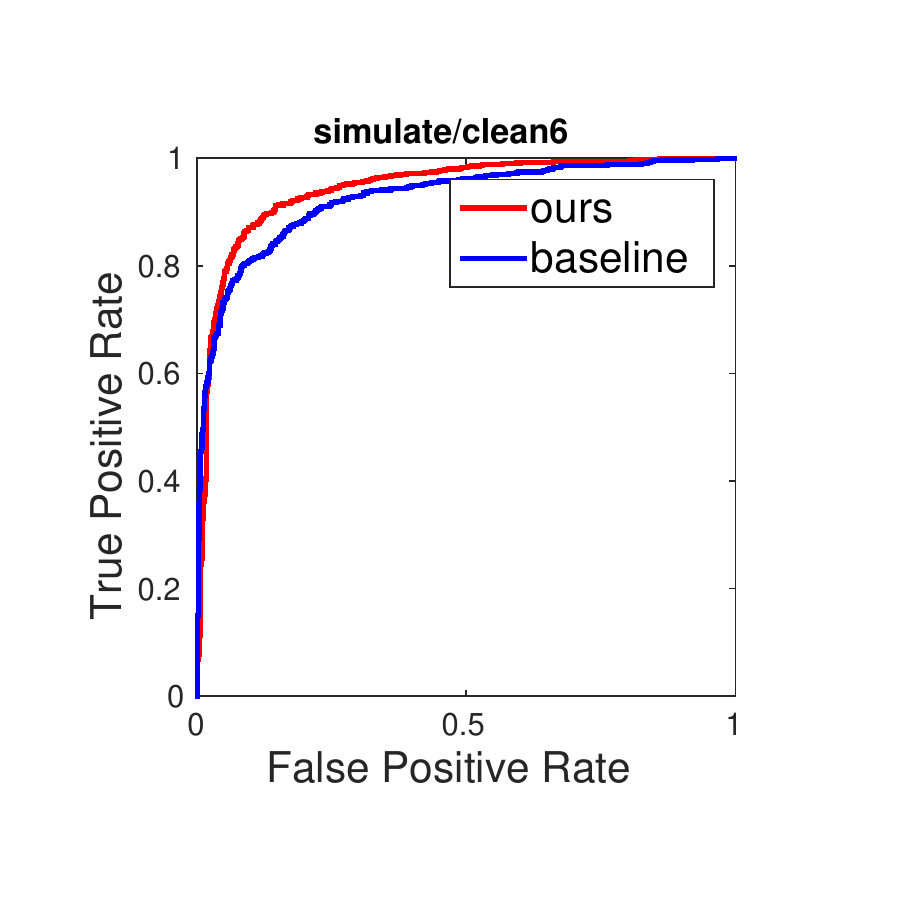}
		&
		\includegraphics[scale=0.27, trim=30 50 0 0]{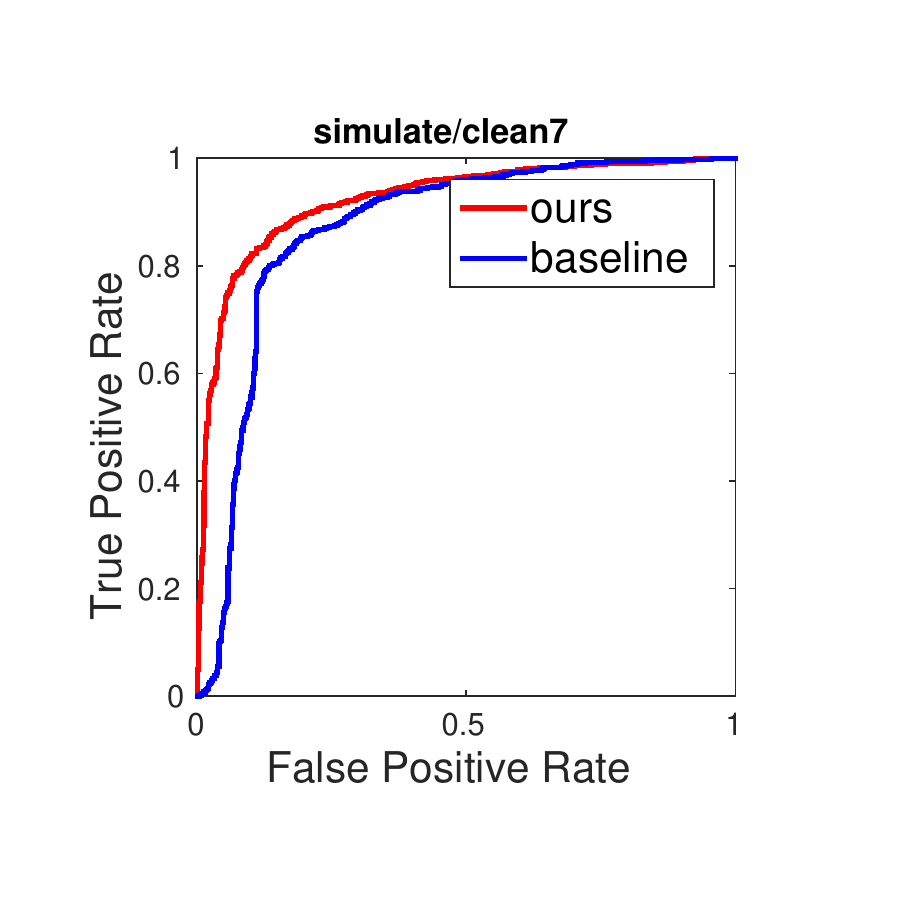}
		&		\includegraphics[scale=0.27, trim=30 50 0 0]{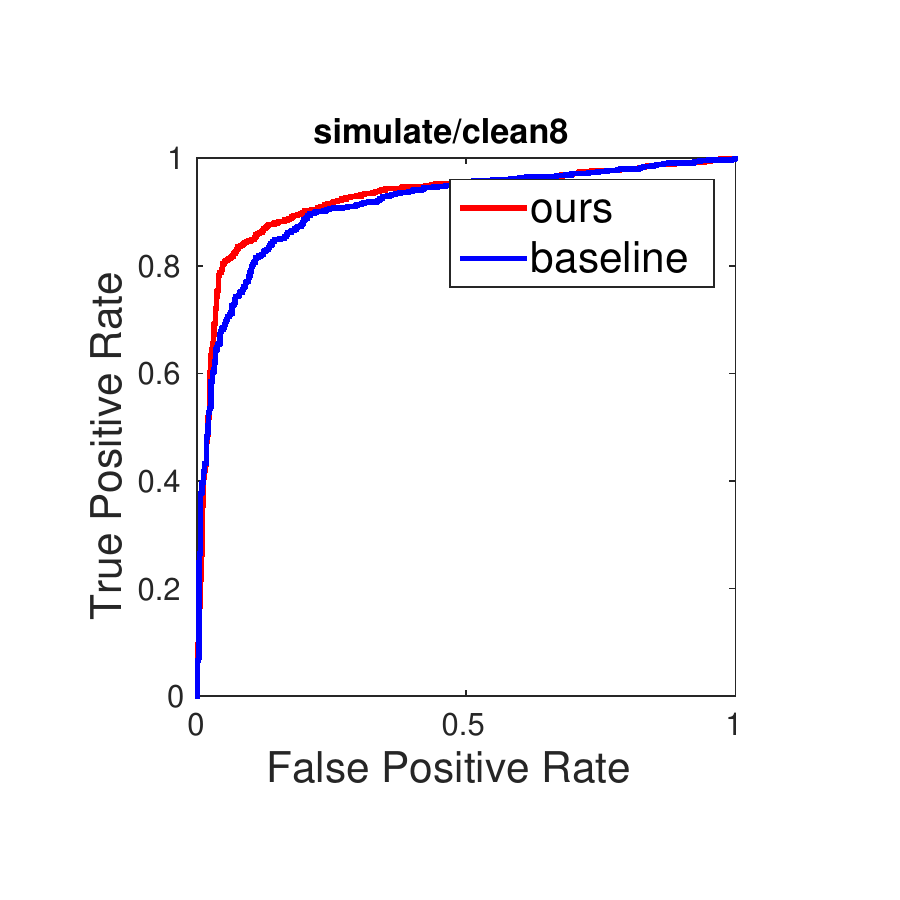}\\
		{C5} & {C6} & {C7} & {C8}\\
		\includegraphics[scale=0.27, trim=30 50 0 0]{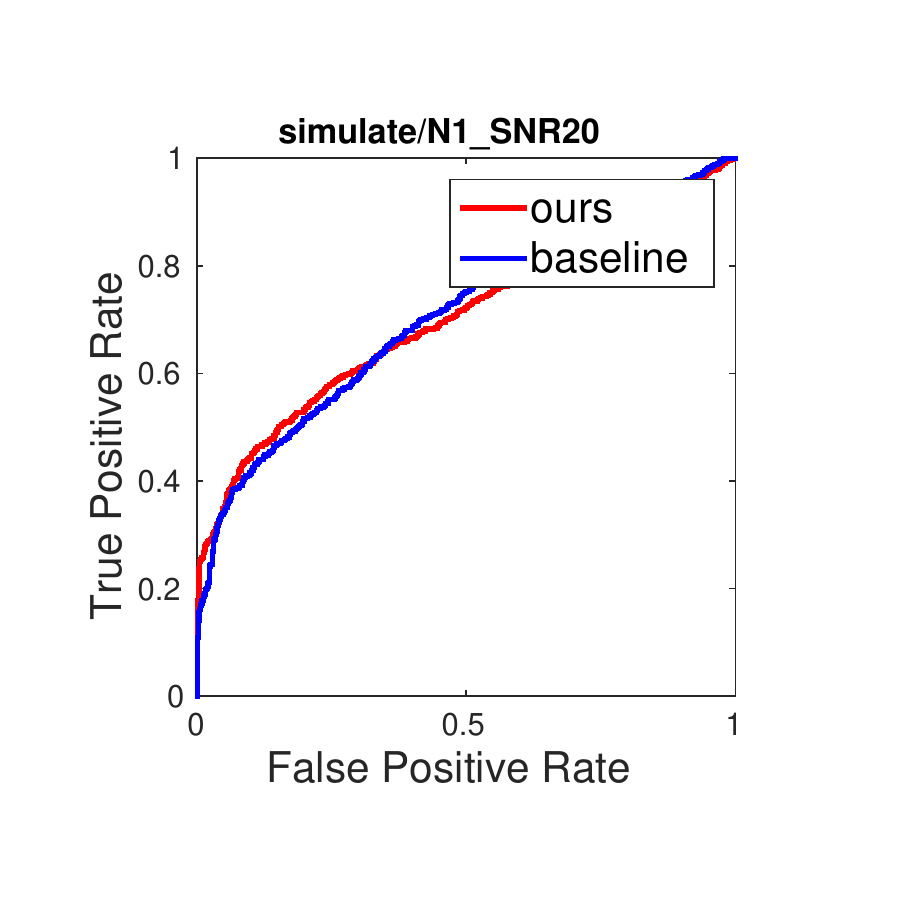}
		&
		\includegraphics[scale=0.27, trim=30 50 0 0]{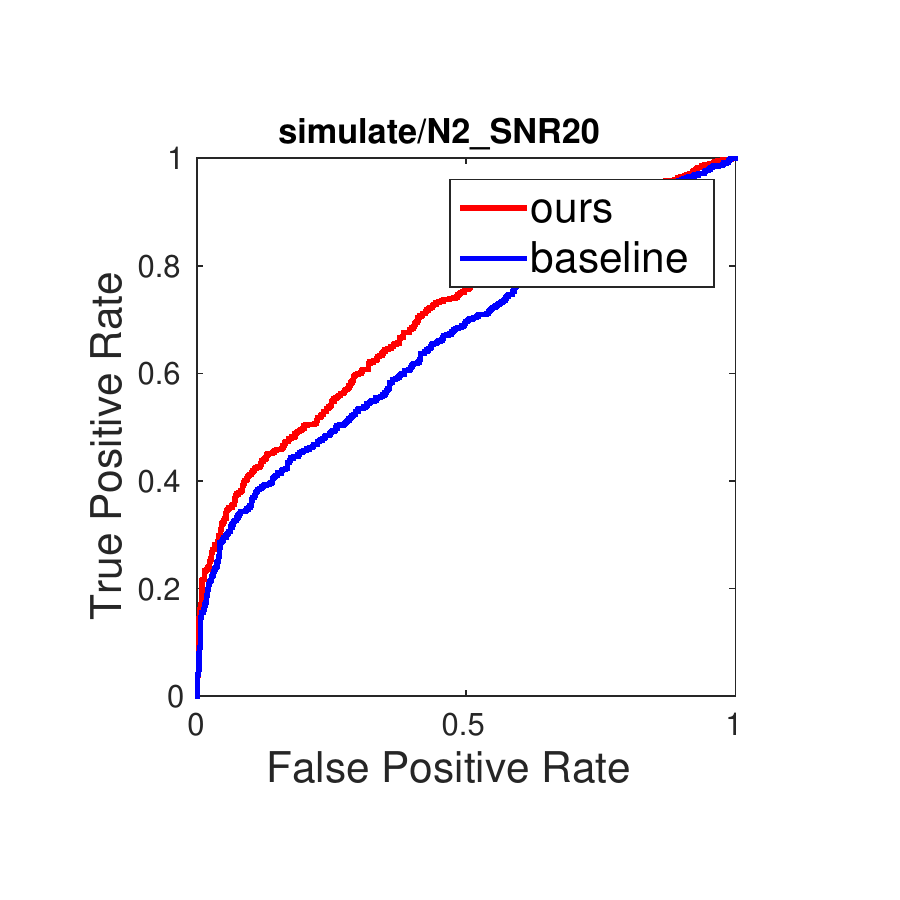}
		&
		\includegraphics[scale=0.27, trim=30 50 0 0]{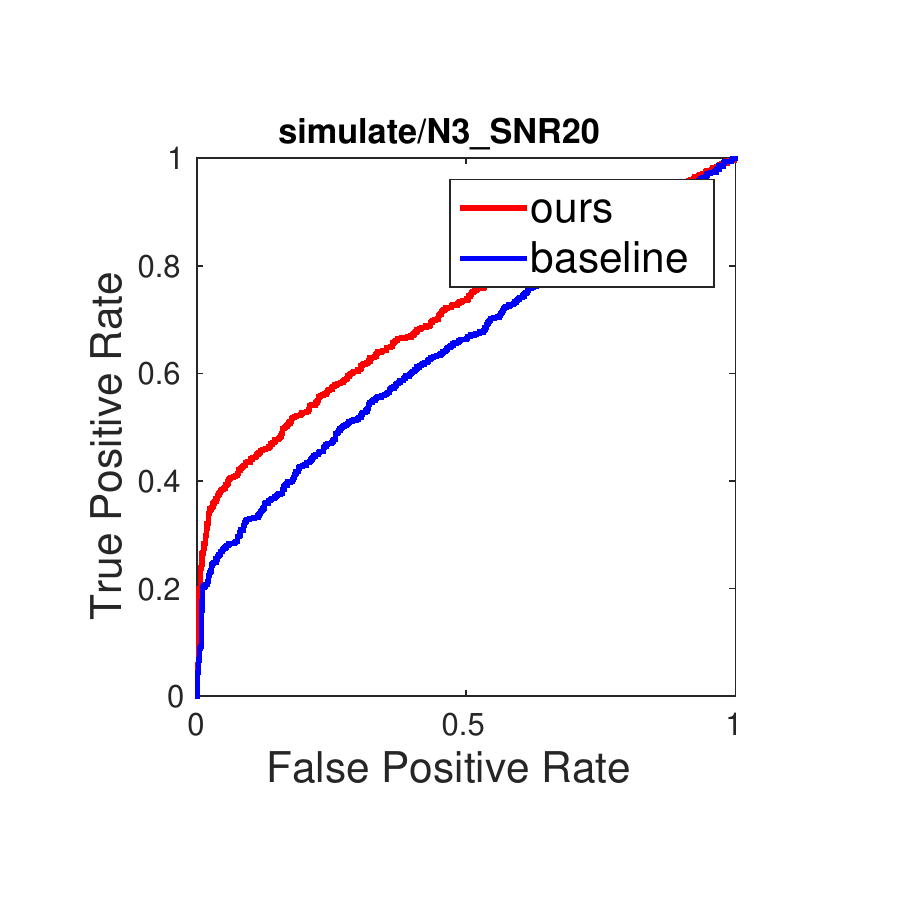}
		&		\includegraphics[scale=0.27, trim=30 50 0 0]{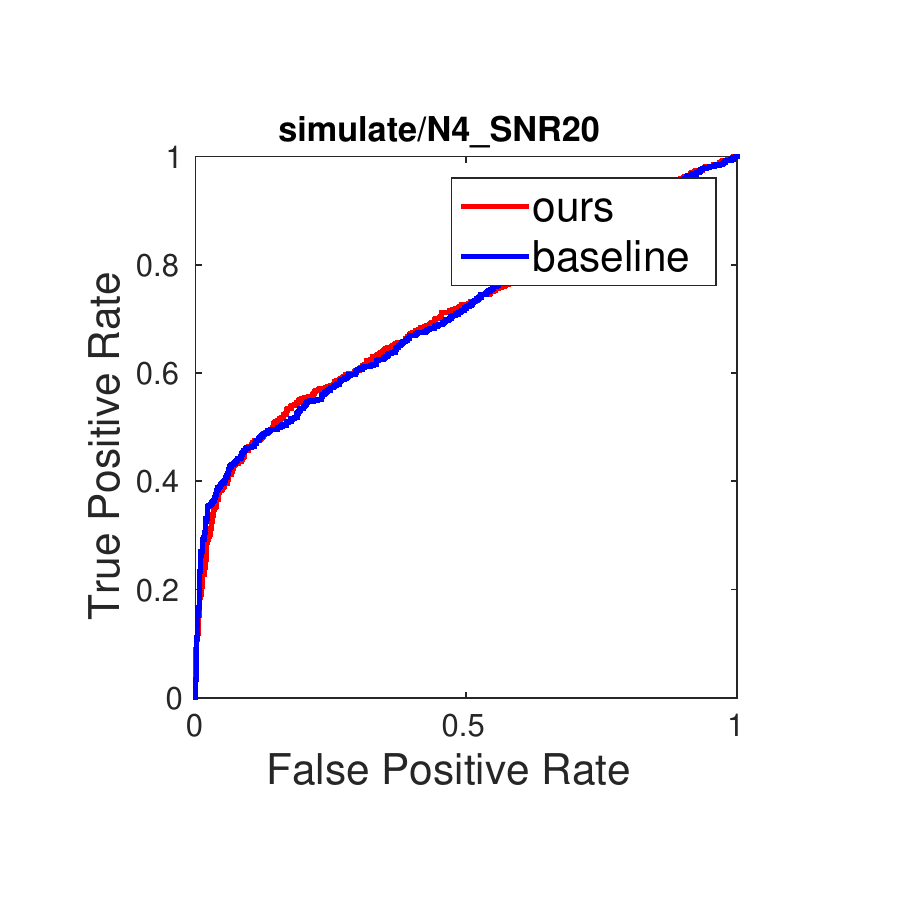} \\
		{S1} & {S2} & {S3} &{S4} \\
		\includegraphics[scale=0.27, trim=30 50 0 0]{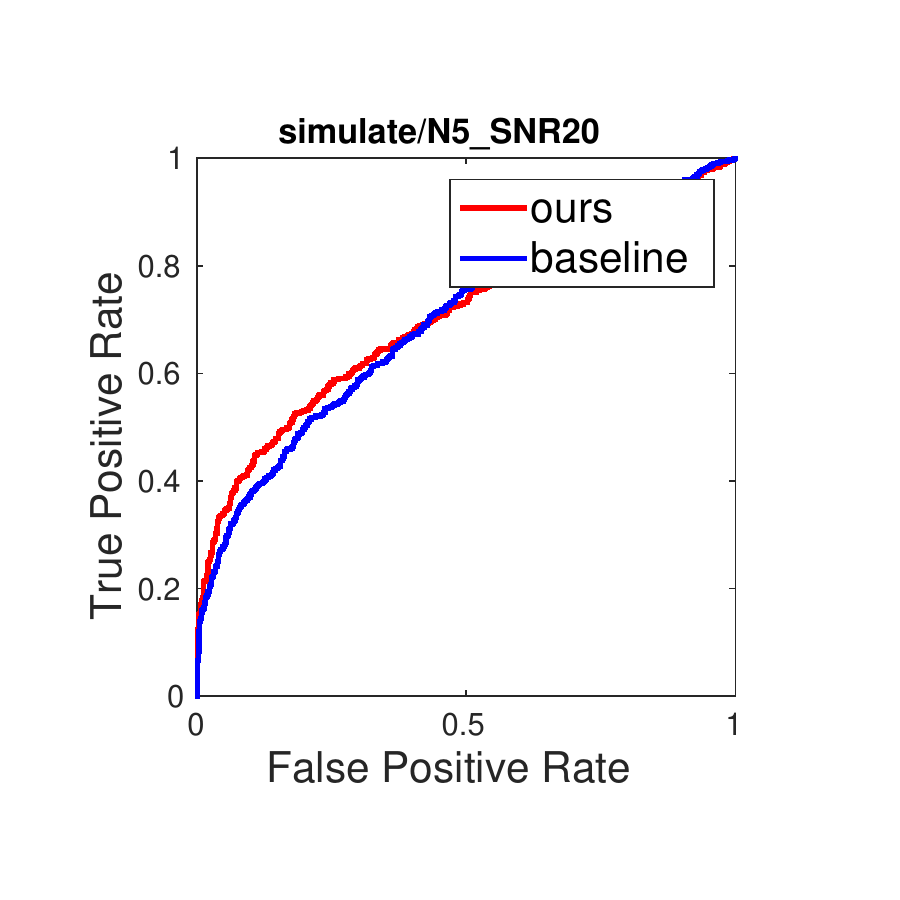}
		&
		\includegraphics[scale=0.27, trim=30 50 0 0]{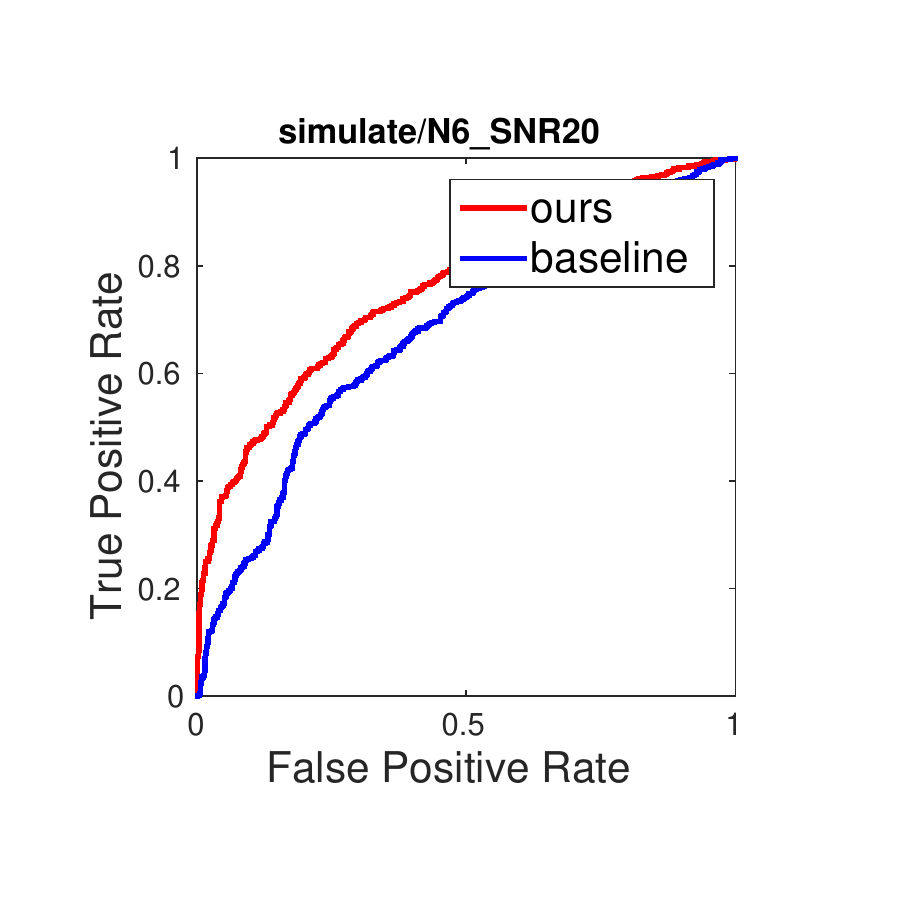}
		&
		\includegraphics[scale=0.27, trim=30 50 0 0]{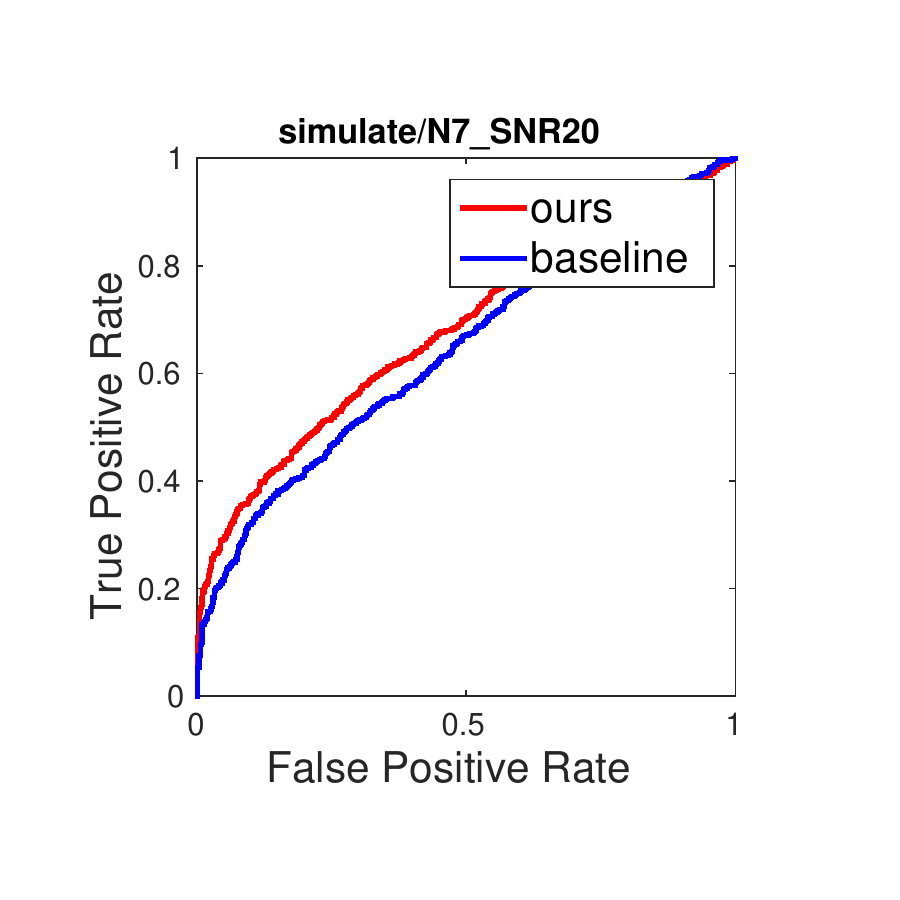}
		&
		\includegraphics[scale=0.27, trim=30 50 0 0]{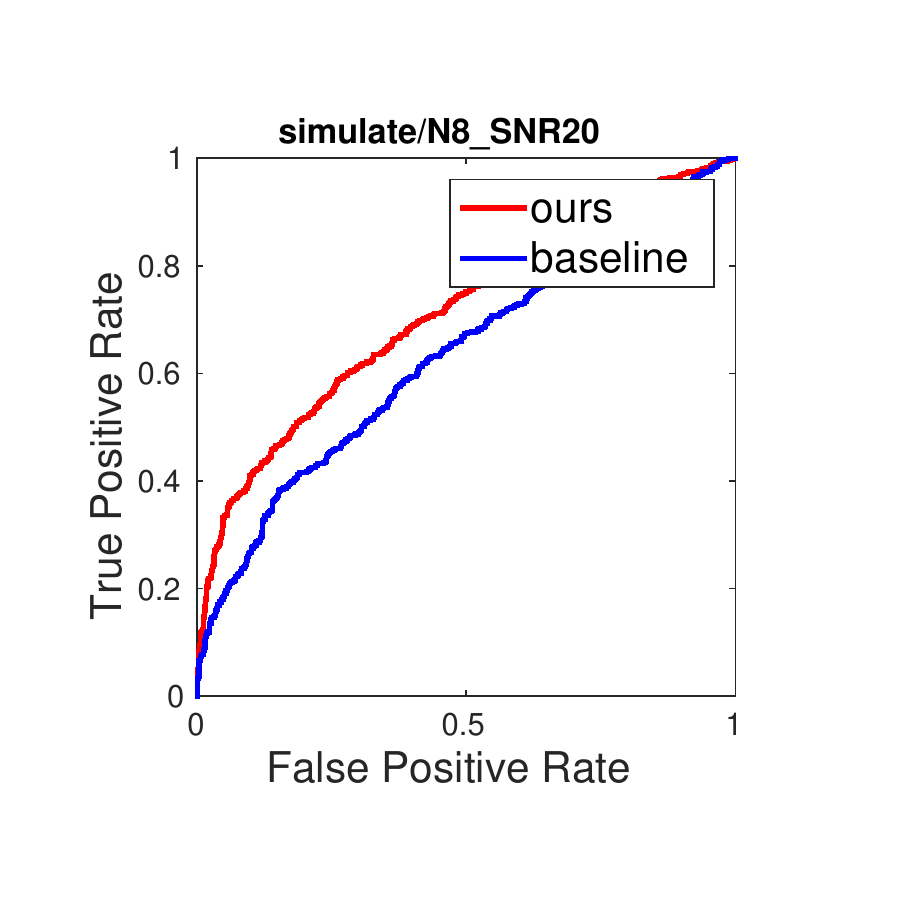}\\
		{S5} & {S6} & {S7} &{S8}
		\end{tabular}
	\caption{ROC curves comparison for speech dataset.}	
	\label{auc_speech_fig}
	\end{center}
\end{figure}

\subsection{HASC Human Activity Dataset} This data set is from \textit{Human Activity Sensing Consortium (HASC) challenge 2011\footnote{http://hasc.jp/hc2011}}. Each data consists of human activity information collected by portable three-axis accelerometers. Following the setting in \citep{density-ratio2013}, we use the $\ell_2$-norm of 3-dimensional data (i.e., the magnitude of acceleration) as the signals. 
	
	We use the `RealWorldData' from HASC Challenge 2011, which consists of 6 kinds of human activities:
	\begin{center}
	walk/jog, stairUp/stairDown, elevatorUp/elevatorDown, \\escalatorUp/escalatorDown, movingWalkway, stay. 
	\end{center}
	We make pairs of signal sequences from different activity categories, and remove the sequences which are too short. We finally get 381 sequences. We tune the parameters using the same way as in CENSREC-1-C experiment. The AUC of our algorithm is \textbf{.8871}, compared to \textbf{.7161} achieved by baseline algorithm, which greatly improved the performance. 
	
	Examples of the signals are shown in Figure~\ref{activity_example}. Some sequences are easy to find the change-point, like Figure~\ref{activity_example}(a), and \ref{activity_example}(d). Some pairs of the signals are hard to distinguish visually, like Figure~\ref{activity_example}(b) and \ref{activity_example}(c). The examples show that our algorithm can tell the change-point from walk to stairUp/stairDown, or from stairUp/stairDown to escalatorUp/escalatorDown. There are some cases when our algorithm raises false alarm. See Figure~\ref{activity_example}(h). It finds a change-point during the activity `elevatorUp/elevatorDown'. It is reasonable, since this type of action contains the phase from acceleration to uniform motion, and the phase from uniform motion to acceleration. 
	
\begin{figure}[h!]
	\begin{center}
	\begin{tabular}{ccc}
		\includegraphics[scale=0.30, trim=0 0 0 0]{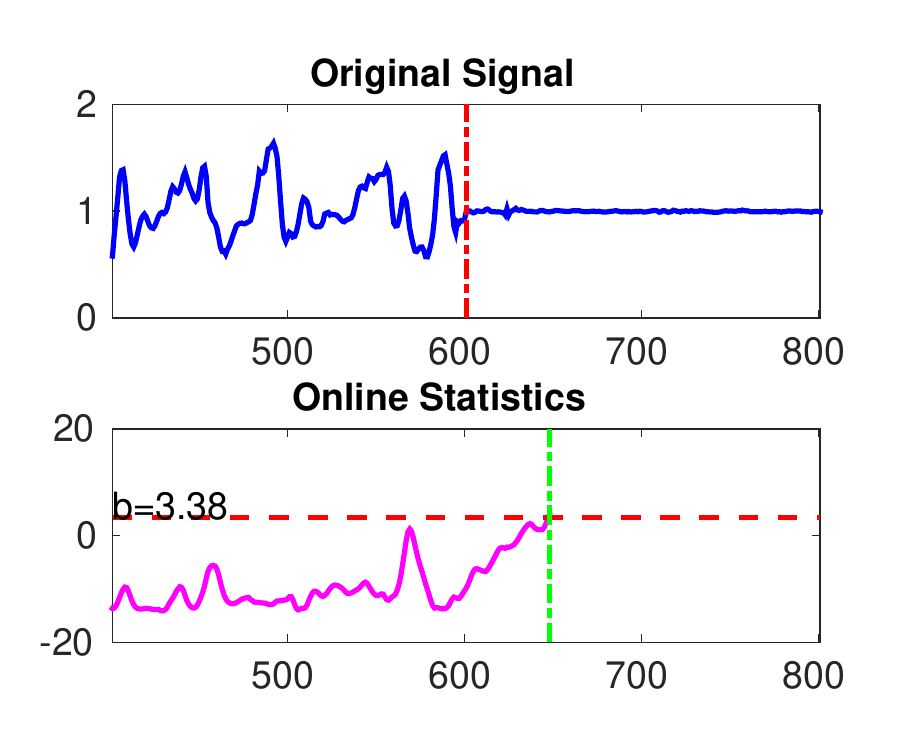}
		&
		\includegraphics[scale=0.30, trim=0 0 0 0]{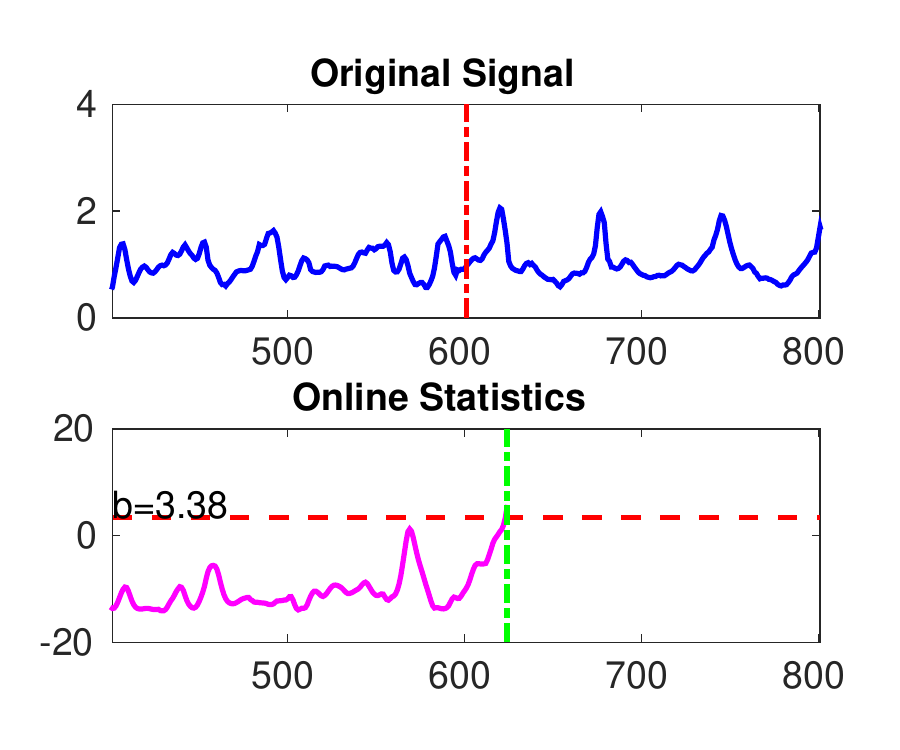}
		&
		\includegraphics[scale=0.30, trim=0 0 0 0]{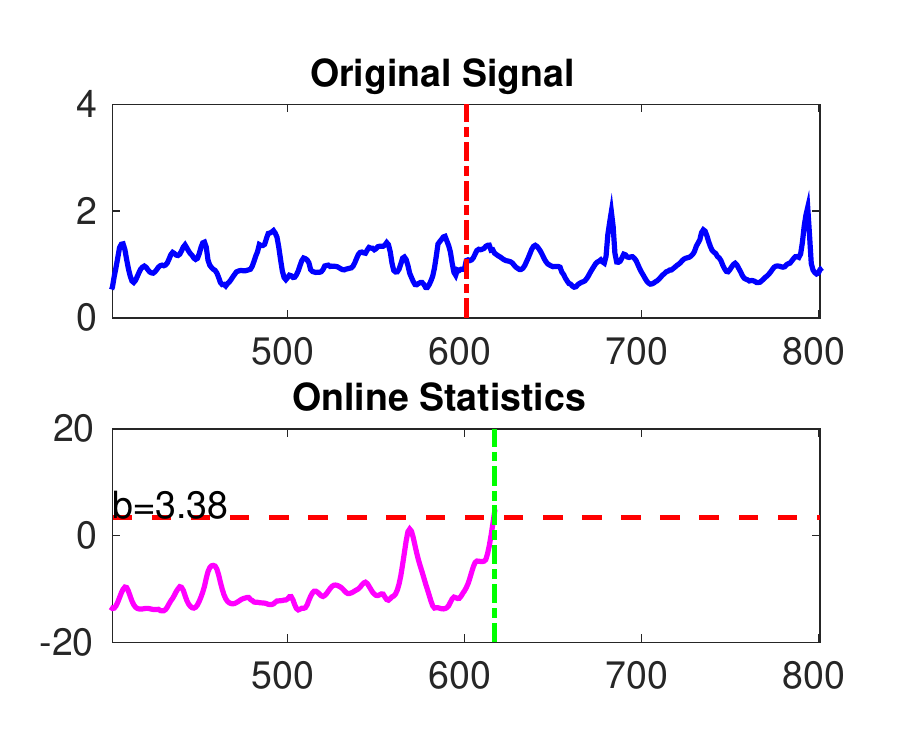} \\
		{(a) A1 vs A6} & {(b) A1 vs A4} & {(c) A1 vs A2} \\
		\includegraphics[scale=0.30, trim=0 0 0 0]{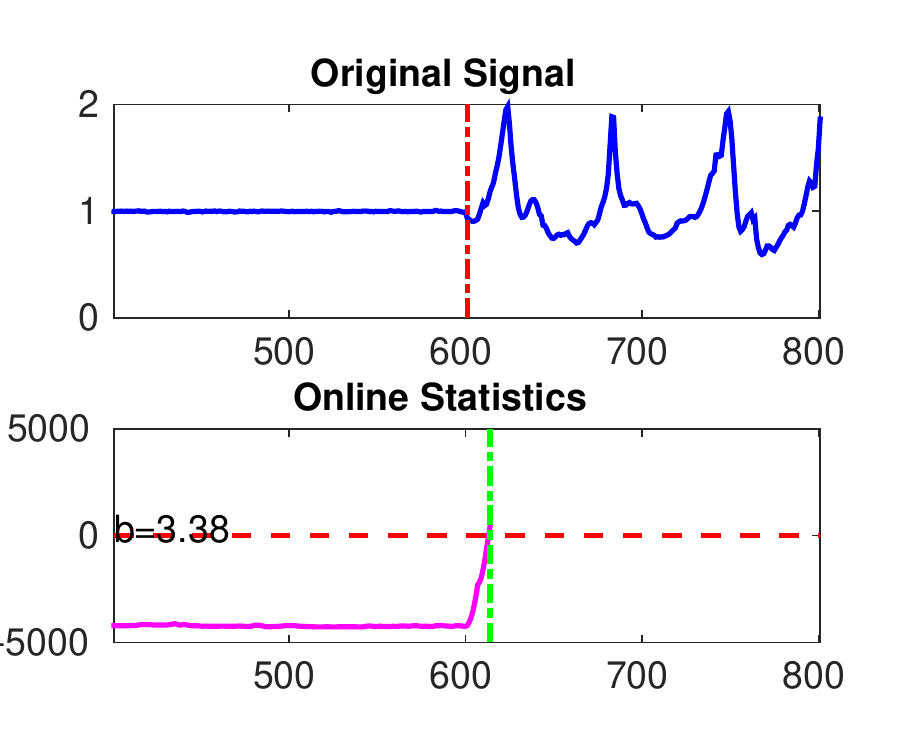}
		&
		\includegraphics[scale=0.30, trim=0 0 0 0]{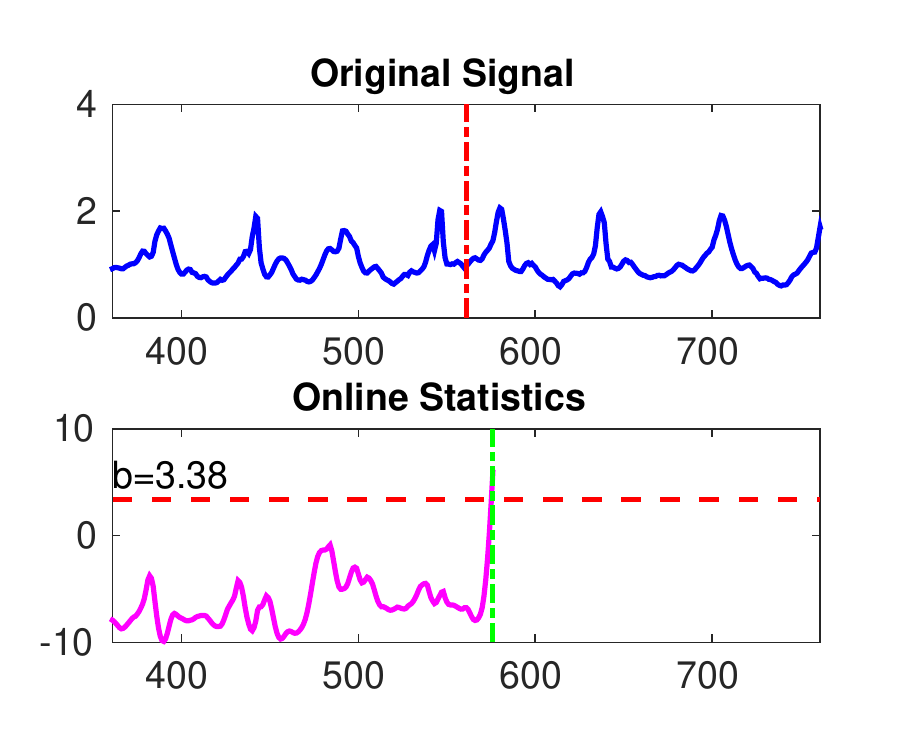}
		&
		\includegraphics[scale=0.30, trim=0 0 0 0]{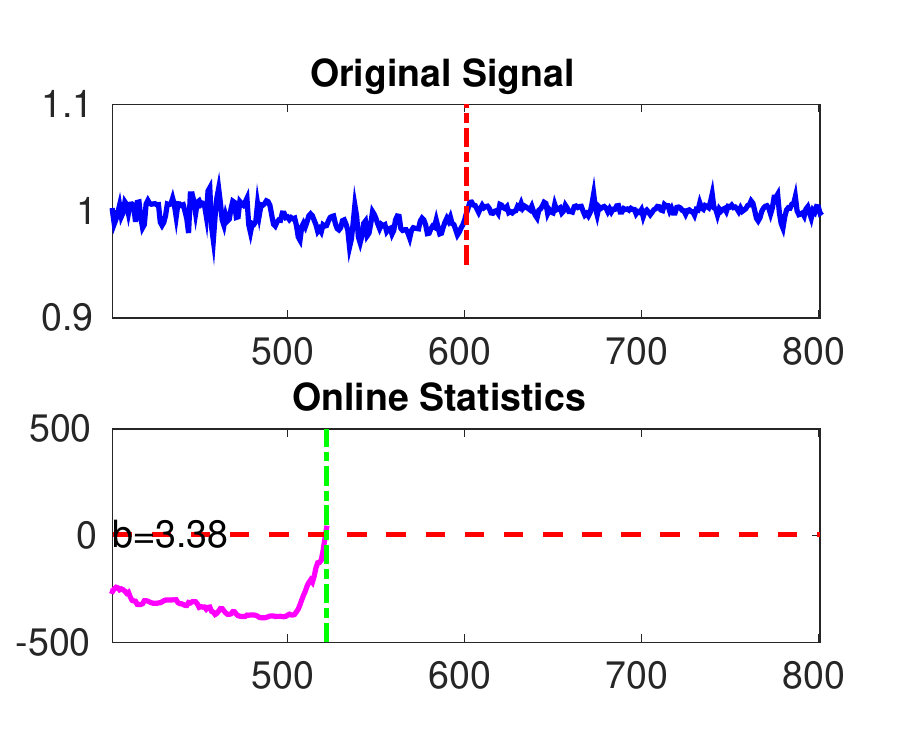} \\
		{(d) A6 vs A2} & {(e) A2 vs A4} & {(f) A4 vs A3}\\
		\includegraphics[scale=0.30, trim=0 0 0 0]{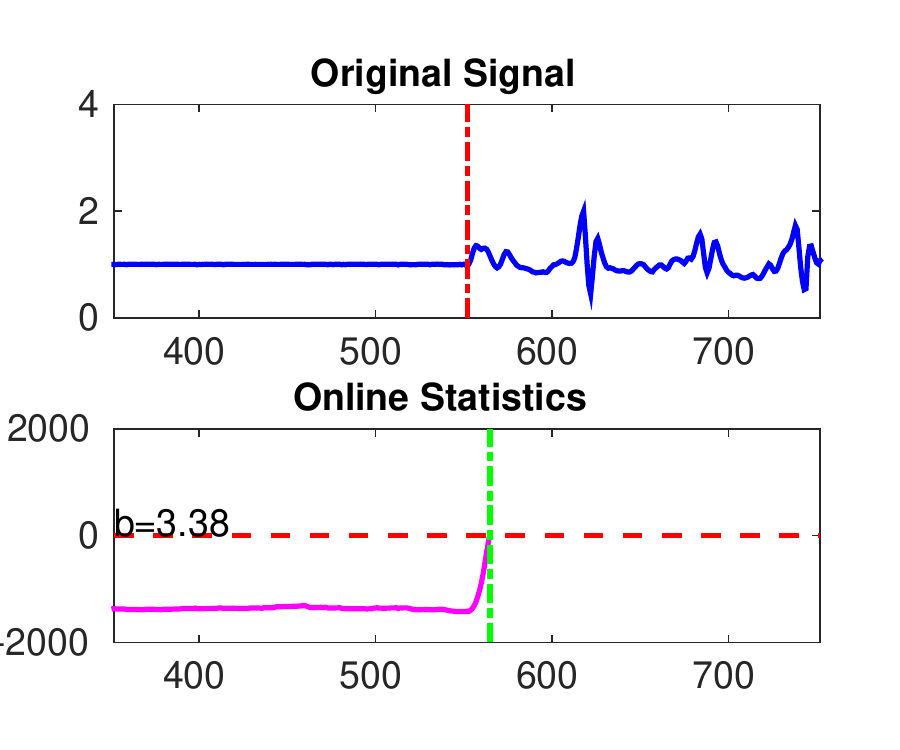}
		&
		\includegraphics[scale=0.30, trim=0 0 0 0]{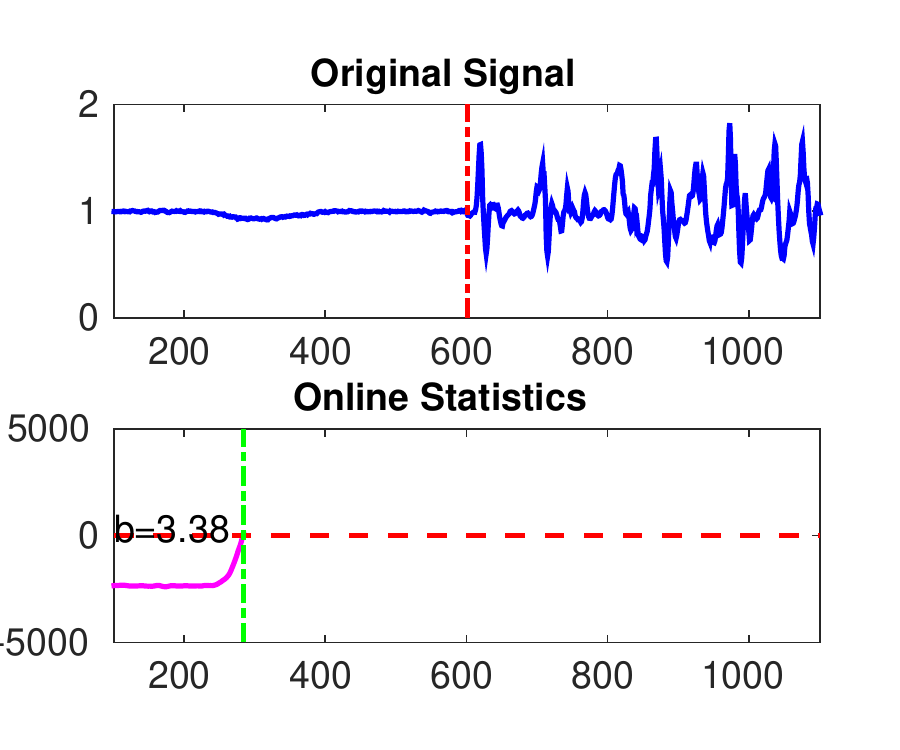}&
		\includegraphics[scale=0.30, trim=0 0 0 0]{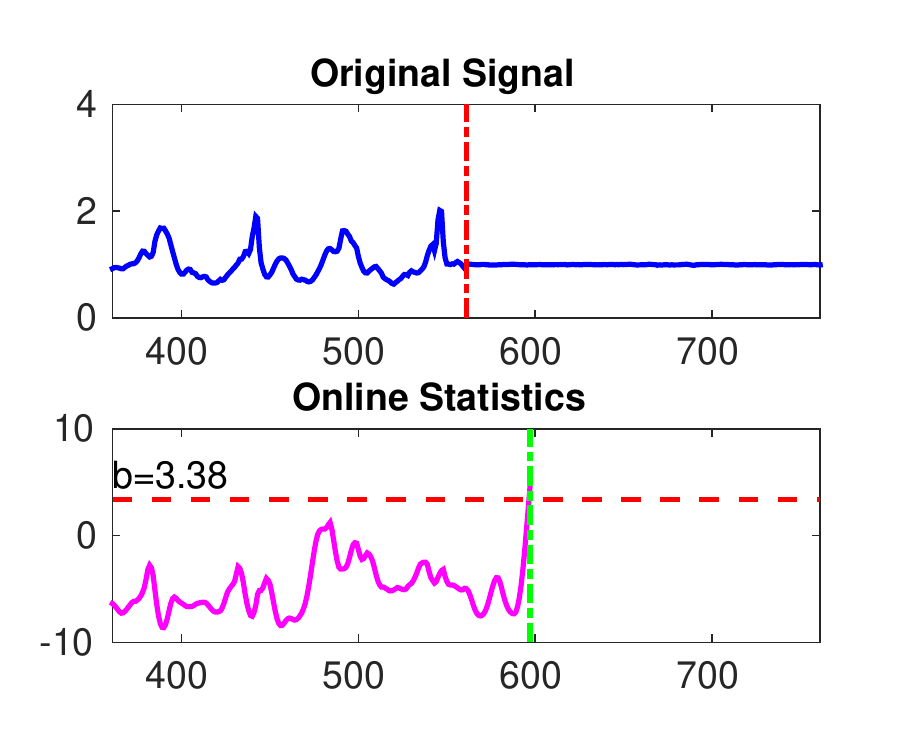}
\\
		{(g) A6 vs A4} & {(h) A3 vs A1} &				{(i) A2 vs A6}
		\end{tabular}
		\end{center}
			\caption{Examples of HASC dataset. The meaning of the markers in this figure are the same as those in Figure~\ref{speech_example}.}
	\label{activity_example}
\end{figure}

\end{document}